\RequirePackage{fix-cm}
\documentclass[final,12pt,numbers]{elsarticle}
\usepackage{hyperref}
\usepackage[centertags]{amsmath}
\usepackage{amsfonts}
\usepackage{amssymb}
\usepackage{adjustbox}
\usepackage{amsmath, bm}
\usepackage{changes}

\newtheorem{proof}{Proof}

\usepackage[ruled,vlined]{algorithm2e}
\usepackage{enumitem}
\usepackage{multicol}
\usepackage{multirow}
\usepackage{makecell}
\usepackage{ragged2e}
\usepackage{color}

\usepackage{titlesec}
\titleclass{\subsubsubsection}{straight}[\subsubsection]
\newcounter{subsubsubsection}[subsubsection]
\renewcommand{\thesubsubsubsection}
  {\thesubsubsection.\arabic{subsubsubsection}}
\titleformat{\subsubsubsection}[hang]
  {\normalfont\normalsize\itshape\mdseries}
  {\thesubsubsubsection.}
  {1em}
  {}
\titlespacing{\subsubsubsection}
  {0pt}
  {12pt plus 6pt minus 3pt}
  {1pt}

\newtheorem{lemma}{Lemma}

\newtheorem{remark}{Remark}
\journal{Journal}
\numberwithin{table}{section}
\usepackage[ruled]{algorithm2e}
\SetKwBlock{GrO}{Grouping optimization:}{end}
\SetKwBlock{PeO}{Perform optimization steps:}{end}
\DeclareMathOperator{\relu}{ReLU}
\numberwithin{equation}{section}
\numberwithin{figure}{section}
\begin{document}
\begin{frontmatter}
\title{RepNN: Tackling spectral bias in deep neural networks for regression and PDE problems via parameter reparameterization}
\author[hust]{Yong Wang} 
\author[cas]{Tao Zhou}
\author[hust]{Xuhui Meng\corref{cor1}}
\ead{xuhui\_meng@hust.edu.cn}

\cortext[cor1]{Corresponding author}
\address[hust]{Institute of Interdisciplinary Research for Mathematics and Applied Science, School of Mathematics and Statistics, Huazhong University of Science and Technology, Wuhan 430074, China}
\address[cas]{Institute of Computational Mathematics, Academy of Mathematics and Systems Science, Chinese
Academy of Sciences, Beijing 100190, China}

\begin{abstract}
Deep neural networks (DNNs) have achieved remarkable success in scientific computing, yet they often suffer from spectral bias in capturing oscillatory and multiscale behaviors. In this study, we investigate this limitation by examining the failure of shallow ReLU neural networks in fitting high-frequency functions. This observation identifies two important factors in resolving rapid oscillations: the initial slope scale and the distribution of partition points induced by the networks. Motivated by this analysis, we propose RepNN, a reparameterized neural network model with ReLU or tanh activations designed for high-frequency and multiscale problems. The key idea is to reparameterize the weights and biases in the first hidden layer, which enables effective control of the initial slope scale and provides an appropriate distribution of the initial partition points. Furthermore, treating the reparameterized weights and biases as trainable parameters allows the DNN to achieve adaptive frequency scaling during training. In addition, we derive quantitative estimates for the output and slope magnitudes of the reparameterized DNN to guide the initialization of the proposed method. Numerical experiments, including multiscale one-, two-, and four-dimensional function approximations, forward and inverse PDE problems in combination with physics-informed neural networks (PINNs), and operator learning for an earthquake problem using real data, demonstrate that RepNN improves the predicted accuracy of vanilla DNNs in capturing highly oscillatory features.
These results indicate that RepNN provides an effective and flexible approach for overcoming spectral bias and applying DNNs to multiscale problems. 

\end{abstract}
\begin{keyword}
Deep learning
\sep Spectral bias
\sep Multiscale problems
\sep PDE problems
\sep Reparameterized DNNs
\end{keyword}
\end{frontmatter}


\section{Introduction}
Deep neural networks (DNNs) have achieved remarkable success in solving problems involving partial differential equations (PDEs), particularly high-dimensional PDEs \cite{yu2018deep,HU2024106369}, and have enabled the development of fast PDE solvers through neural operators \cite{li2021ftric,YOU2026114530,lu2021learning,AMIN2026118645}. These deep learning–based approaches have been successfully applied to a wide range of real-world applications, including seismic imaging \cite{101093gjiggac399,GENG2025105967}, combustion modeling \cite{zanardi2023adaptive,ZHANG2024113647}, and fluid dynamics \cite{ZHENG2025110329,Zhang_Shukla_Wang_s}. Despite these advances, a fundamental and well-recognized limitation of vanilla DNNs is spectral bias \cite{pmlrv97rahaman19a,XU2025113905}, also known as the frequency principle \cite{luo2019thinciplegeneral,xu2025overview}. During training, DNNs tend to learn low-frequency components of a target function significantly faster and more accurately than high-frequency components. This inherent preference for low-frequency information limits their ability to represent oscillatory and multiscale features, posing a significant challenge for PDE problems characterized by complex multiscale phenomena \cite{khodakarami2026spectral}.

Spectral bias has been extensively investigated from various perspectives in recent years. Khodakarami {\sl et al.}~\cite{khodakarami2026spectral} demonstrated that this bias \textcolor{black}{arises from several factors,}
including the representation capacity of neural networks, optimization, as well as loss formulations. In this study, we focus specifically on mitigating this limitation by architecture design to enhance the network expressivity. Along this direction, numerous strategies have been proposed.
In one of the most recent review papers on spectral bias \cite{XU2025113905}, Xu {\sl et al.} categorize these approaches into frequency shifting \cite{doi10113719M1310050,10555536258343625935,hu2024neutron}, frequency scaling \cite{tancik2020fourier,CiCP281970,WANG2021113938,101093gjiggac399,WANG2024113112,HUANG2025117751,feng202ention}, and hybrid methods \cite{CHEN2022110996,zhang2024blending,ALDIRANY2024116666}.
Frequency shifting methods aim to transform challenging high-frequency learning tasks into more tractable low-frequency ones, thereby accelerating the training of DNNs on oscillatory components. A representative example is the phase shift DNN (PhaseDNN) \cite{doi10113719M1310050}, which has demonstrated strong performance for high-frequency wave equations in inhomogeneous media. However, PhaseDNN often struggles with highly complex multiscale problems and suffers from the curse of dimensionality \cite{XU2025113905}.
Another important class of methods is frequency scaling, which maps high-frequency components to lower-frequency regimes before training, enabling DNNs to better capture multiscale features. Representative architectures include the multiscale DNN (MscaleDNN) \cite{CiCP281970,LI2023112242,YOU2026114530} and multiscale Fourier feature networks (MFF) \cite{WANG2021113938,LI2023114963,wang2safsded}. In the MscaleDNN developed in ~\cite{CiCP281970}, Liu {\sl et al.} employed the radial down-scaling mappings in the frequency domain, which has been successfully applied to solving the Boltzmann equation \cite{doi10113723M1558227} and three-dimensional turbulent wind field reconstruction \cite{CHEN2025124577}. The MFF ~\cite{WANG2021113938} is developed from the perspective of the neural tangent kernel, with the core idea to add the random Fourier feature embeddings for the input of DNNs, and it has been applied to ice melting simulations \cite{chsicsinformed} and seismic denoising \cite{xie2025simultaneous}. Other notable architectures include SIREN \cite{sitzmann2020implicit}, DNNs with HAT activations \cite{hong2022activation}, high-frequency-enhanced DNNs \cite{HUANG2025113676}, cross-attention-based DNNs \cite{feng202ention}, and data-integrated DNNs \cite{zhelving}. Despite these advances, most of the aforementioned approaches rely on pre-defined multiscale transformations or frequency bases \cite{WANG2021113938,huasorneuralnetworks,feng202ention}, which can restrict their flexibility and robustness when dealing with problems involving unknown, spatially varying, or highly heterogeneous frequency characteristics \cite{wangves,HUANG2025117751}.

More recently, several adaptive approaches \cite{10106350286561,huasorneuralnetworks,tang2based} have also been proposed to reduce the reliance on pre-defined hyperparameters. For example, Huang {\sl et al.} developed frequency-adaptive MscaleDNNs \cite{HUANG2025117751} and frequency-adaptive tensor neural networks \cite{huasorneuralnetworks} to improve the robustness of the models developed in ~\cite{CiCP281970,WANG2021113938}. However, Huang {\sl et al.} noted that these adaptive approaches incur substantially higher computational costs \cite{HUANG2025117751}. In ~\cite{wangves}, Wang {\sl et al.} proposed a random weight factorization strategy, in which an additional trainable scaling parameter is assigned to each neuron in a vanilla DNN to facilitate the learning of high-frequency components. Nevertheless, the performance of this approach depends on the careful initialization of the associated hyperparameters. A similar limitation also exists in the DNN framework with trainable sinusoidal activation functions proposed in \cite{KHADEMI2025109672}. In \cite{LIU2025106886}, Liu {\sl et al.} employed the training loss to guide spatially adaptive Fourier feature encoding. However, its effectiveness may be compromised when the training loss does not accurately reflect the frequency-wise approximation quality. More recently, Hou {\sl et al.}~\cite{HOU2026108247} proposed a framework that dynamically generates Fourier features from layer-wise residual representations. While effective, this approach can be computationally expensive, particularly for deep neural network architectures.

In this work, we attempt to address the aforementioned issues by developing a principled yet lightweight approach for overcoming the spectral bias. Motivated by the partition of the input space induced by ReLU as well as tanh DNNs~\cite{he2020reluReLU}, we propose a reparameterized deep neural network model, which is referred to as RepNN, for high-frequency and multiscale problems. The key idea is to reparameterize the weights and biases in the first hidden layer, enabling effective control of the slope scale and an appropriate initial decomposition of the computational domain. Our contributions are summarized as follows:
\begin{itemize}

\item We propose a reparameterized neural network (RepNN) model that effectively controls the initial slope scale and provides a well-distributed initialization for the partition points, thereby enhancing the capability of DNNs to capture highly oscillatory behavior. Further, by treating the reparameterized weights and biases as trainable parameters, RepNN provides an {\emph{adaptive frequency scaling mechanism}}, thereby reducing its reliance on pre-defined multiscale features.

\item We provide quantitative estimates for the output and slope magnitudes of the proposed RepNN, which can guide the design of suitable parameter initialization. We also discuss the gradient dynamics at initialization, offering a heuristic perspective that tanh saturation may help constrain gradient magnitudes and mitigate gradient explosion, while sparse active pathways may contribute to training stability.

\item We demonstrate the {\emph{effectiveness and robustness}} of RepNN through extensive numerical experiments on function approximation, forward PDE problems, inverse parameter identification, and operator learning.
\end{itemize}

The rest of this paper is organized as follows: In Section~\ref{Section: section Methodology}, we present the RepNN together with quantitative estimates of the output and slope magnitudes to guide the parameter initialization. In Section~\ref{section: Numerical examples}, we first conduct ablation studies on the key hyperparameters and then evaluate RepNN on a suite of benchmarks including function approximation, forward and inverse PDE problems, and operator learning. In the last section, we give a summary of this work.

\section{Methodology}\label{Section: section Methodology}
In this section, we first briefly review the vanilla deep neural networks using function approximation as an example, and then we introduce the proposed reparameterization as well as initialization of the parameters in RepNN to mitigate the spectral bias.

\subsection{Vanilla deep neural networks}\label{section: VDNN}
Let us consider a vanilla fully-connected neural network (FNN) with $K$ layers, consisting of $K-1$ hidden layers and one output layer, preceded by an input layer. The output of each layer can be formulated as follows:
\begin{equation}\label{DNNs}
\begin{cases}
Z_1 = \sigma(W_1\bm{x} +b_1),&\\
Z_l = \sigma(W_lZ_{l-1}+b_l),& l=2, \cdots, K-1,\\
u_{\bm{\theta}}(\bm{x}) = W_KZ_{K-1}+b_K,\\
\end{cases}
\end{equation}
where $\bm{\theta} =\{\bm{W}_{l}=(W_{ij}^{(l)}) \in \mathbb{R}^{m_{l} \times m_{l-1}}, \bm{b}_{l} = (b_{i}^{(l)}) \in \mathbb{R}^{m_{l}}\}_ {l=1}^K$ denotes all the parameters in the DNN, including the weight matrices and bias vectors, $m_{l}$ represents the number of neurons in the $l$-th layer. Further, $\bm{x} \in \mathbb{R}^{m_0}$ is the input vector which denotes the spatial and/or temporal coordinates, and $u_{\bm{\theta}}(\bm{x})$ is the output of the DNN with respect to input $\bm{x}$ parameterized by $\bm{\theta}$, and $\sigma$ denotes the element-wise activation function, two widely used examples are ReLU,
\begin{equation}
\relu(x) = \max\{0,x\}, \quad x \in \mathbb{R},  
\end{equation}
and the hyperbolic tangent activation function,
\begin{equation}
\tanh(x) = \frac{e^x - e^{-x}}{e^x + e^{-x}}, \quad x \in \mathbb{R}.   
\end{equation}
For the stability of training, when $\sigma = \relu$, the parameter set $\bm{\theta}$ is generally initialized using He initialization \cite{he2015delving}:
\begin{equation}\label{DNN:he_initialition}
 W_{ij}^{(l)} \sim \mathcal{N}(0, \frac{2}{m_{l-1}}), \quad b_{i}^{(l)}=0,
\end{equation}
with $\mathcal{N}$ denoting the Gaussian distribution. 
When $\sigma = \tanh$, the set $\bm{\theta}$ is initialized using Xavier initialization \cite{glorot}:
\begin{equation}\label{DNN:Xavier_initialition}
 W_{ij}^{(l)} \sim \mathcal{N}(0, \frac{2}{m_{l-1}+m_l}), \quad b_{i}^{(l)}=0.
\end{equation}
In this study, the above initialization strategies are adopted for the vanilla DNNs defined in Eq.~\eqref{DNNs}.

\begin{remark}
To reduce the influence of different coordinate scales, the input vector $\bm{x}$ denotes the normalized coordinate throughout this work. Given an original coordinate vector $\bm{x}^{\rm ori}$, we define
\begin{equation}
\bm{x}
=
2\frac{\bm{x}^{\rm ori}-\bm{x}_{\min}}
{\bm{x}_{\max}-\bm{x}_{\min}}
-\mathbf{1},
\end{equation}
where $\bm{x}_{\min}, \bm{x}_{\max} \in \mathbb{R}^{m_0}$ are the component-wise lower and upper bounds of the computational domain, respectively. Thus, the original coordinates are normalized to the range \([-1,1]^{m_0}\) before being fed into the neural network.
\end{remark}

The training of a DNN is typically formulated as an optimization problem, whose objective is to identify a suitable parameter set $\bm{\theta}$. For instance, consider a regression problem with a training dataset $\mathcal{D} = \{(\bm{x}_i, {u}_i)\}_{i=1}^N$. The goal is to obtain a DNN $u_{\bm{\theta}}(\cdot)$ that best approximates the underlying mapping from inputs $\bm{x}_i$ to targets ${u}_i$. This is achieved by minimizing the mean squared error (MSE) loss function with respect to the parameter set $\bm{\theta}$:
\begin{equation}\label{eq:dnn-mse-loss}
\mathcal{L}(\bm{\theta}; \mathcal{D}) = \frac{1}{N} \sum_{i=1}^N | u_{\bm{\theta}}(\bm{x}_i) - {u}_i |^2.
\end{equation}
The loss Eq.~\eqref{eq:dnn-mse-loss} is then minimized via gradient-based optimization algorithms, such as Adam \cite{kingma2} and L-BFGS \cite{LLLBFGS}.
Suppose
\begin{equation}\label{eq:theta-opt}
{\bm{\theta}^*} = \arg\min_{\bm{\theta}} \mathcal{L}(\bm{\theta}; \mathcal{D}),
\end{equation}
and the resulting $u_{\bm{\theta}^*}(\cdot)$ serves as the learned approximation to the target function given the training dataset.

\subsection{RepNN: Reparameterized neural networks}\label{Section: Methodology}
In this section, we first discuss the limitations of ReLU neural networks in approximating high-frequency functions using an example of a one-dimensional case, as described in Section~\ref{Methodology: ReLU neural networks}. Motivated by this analysis, we then propose a reparameterized deep neural network with both ReLU and hyperbolic tangent activations ($\mathrm{tanh}$) for high-frequency and multiscale problems in Section~\ref{section: reparameterization of DNN framework}. Note that RepNN does not completely change the learning pattern of DNNs: low-frequency components still converge faster than high-frequency components. However, RepNN learns high-frequency components substantially more efficiently than vanilla DNNs, as we demonstrate in \ref{FWTBDNN}.


\subsubsection{ReLU neural networks for regression with high-frequency function}\label{Methodology: ReLU neural networks}

We consider a single-hidden-layer fully connected ReLU neural network (ReLU NN) with $m$ hidden neurons, denoted by $u_{\bm{\theta}}(x): \mathbb{R} \to \mathbb{R}$. The detailed forward pass of the network is given by
\begin{equation}\label{eq:ssingle-hidden-layer relu}
\begin{cases}
Z_1 = \relu (W_1 x + b_1), \\
u_{\bm{\theta}}(x) = W_2 Z_1 + b_2,
\end{cases}
\end{equation}
where $W_1, W_2, b_1, b_2$, defined consistently with vanilla DNN Eq.~\eqref{DNNs}, are the parameters of the network to be determined. Following the analysis in ~\cite{he2020reluReLU}, the neurons in the hidden layer correspond to a set of partition points $\bm{x}_p \in \mathbb{R}^m$, defined by
\begin{equation}\label{partition_hyperplane}
W_1 \odot \bm{x}_p + b_1 = 0,
\end{equation}
where $\odot$ denotes the Hadamard product.
Since $\relu$ is a continuous piecewise linear function, each associated partition point divides the input space $\mathbb{R}$ into activated and non-activated regions. Furthermore, for a ReLU NN, there exists a corresponding interval decomposition of input space $\mathbb{R}$ such that the network provides a linear approximation on each subinterval in such a decomposition. Note that, through the lens of the finite element method (FEM)~\cite{doi1011370718033,CHEN2025113939}, such an interval decomposition can be regarded as a ``grid'' that discretizes the computational domain. 

Next, we employ the ReLU NN Eq.~\eqref{eq:ssingle-hidden-layer relu} to approximate the following one-dimensional high-frequency function:
\begin{equation}\label{Methodology: high-frequency}
u(x) = \sin (15 \pi x), \quad x\in[-1,1].
\end{equation}
We set the number of hidden neurons to $m=400$, and initialize all network parameters using the He initialization in Eq.~\eqref{DNN:he_initialition}:
\begin{equation}
W_{ij}^{(1)} \sim \mathcal{N}(0, 2), \quad W_{ij}^{(2)} \sim \mathcal{N}(0,0.005), \quad b_{i}^{(1)} =b_{i}^{(2)}= 0.
\end{equation}
The training dataset $\{x_i, u(x_i)\}_{i=1}^{1280}$ is generated by uniformly sampling $x_i$ from the interval $[-1,1]$, and is used to evaluate the loss function in Eq.~\eqref{eq:dnn-mse-loss}. Other experimental settings, \textcolor{black}{including the optimizer}, follow the descriptions in Section~\ref{section: Numerical examples}. Fig.~\ref{fig Methodology: fit high-frequency} (left) presents a detailed comparison between the exact solution and the prediction obtained by the ReLU network. It can be observed that the network accurately approximates the target function only in a narrow region around \(x=0\), while exhibiting poor accuracy over the remainder of the computational domain. This behavior is somewhat surprising because the ReLU network used in this example possesses approximately 400 partition points, which should in principle be sufficient to resolve the oscillatory target function.
To further investigate this phenomenon, we plot the histograms of the initial and optimized partition points in Fig.~\ref{fig Methodology: fit high-frequency} (right). It can be observed that most of the optimized partition points remain concentrated near \(x=0\).
Consequently, only a small portion of the computational domain is adequately partitioned, which explains why the ReLU network achieves good approximation accuracy only near \(x=0\).
These results reveal a significant gap between the established theoretical approximation capability of ReLU networks \cite{he2020reluReLU} and their practical performance when trained using gradient-based optimization. In particular, although the network has sufficient representational capacity, the learned distribution of partition points may be highly non-uniform, limiting its ability to capture highly oscillatory features across the entire domain.

\begin{figure}[!ht]
 \centering
 \includegraphics[width=1.0\textwidth]{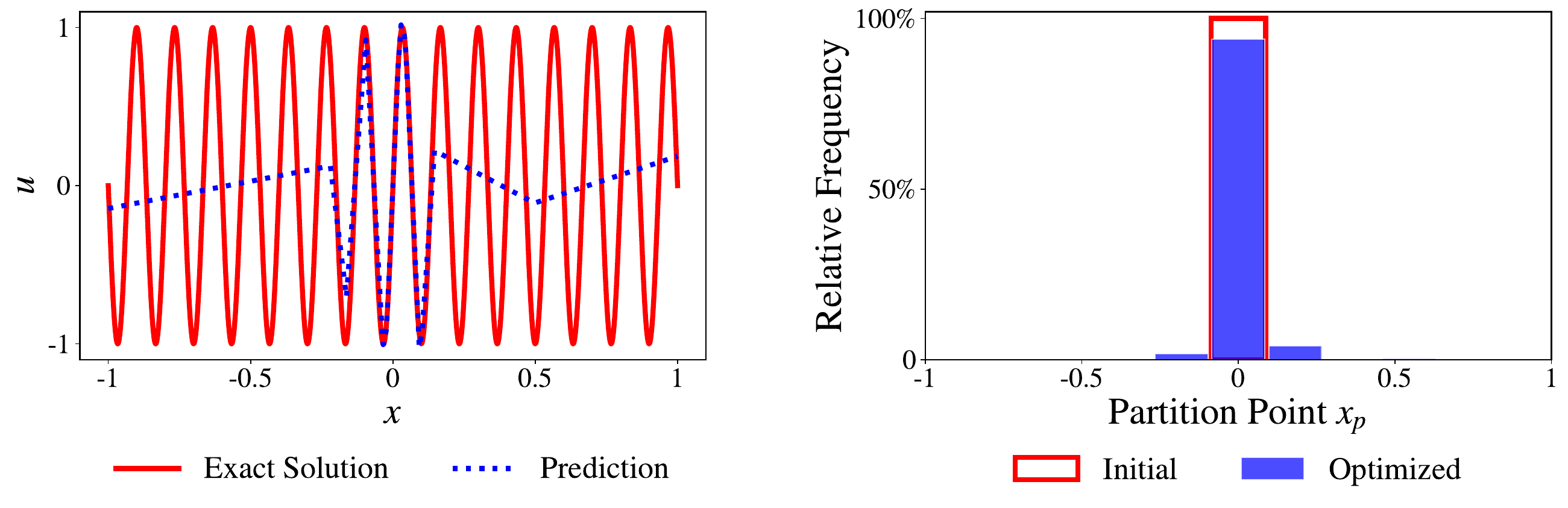} 
 \caption{1D high-frequency function approximation in Eq.~\eqref{Methodology: high-frequency}: The exact solution and predicted solution obtained by the ReLU NN Eq.~\eqref{eq:ssingle-hidden-layer relu} (left). (Normalized) Histograms of the initial and optimized partition points in the  ReLU NN Eq.~\eqref{eq:ssingle-hidden-layer relu} (right),  \textcolor{black}{where 2 optimized partition points  (0.5\%) lie outside $[-1,1]$ and are not shown.}}
 \label{fig Methodology: fit high-frequency}
\end{figure}

\textcolor{black}{Based on the pedagogical example, we identify two possible factors that may hinder the learning of high-frequency components and thus contribute to the spectral bias of ReLU NN:}
\begin{itemize}
 \item {\bf{Mismatch between \textcolor{black}{the} initial slope and that of target function}}.

 High-frequency target functions usually exhibit large local slopes. As discussed in \cite{he2020reluReLU}, a ReLU NN provides a linear approximation on each partitioned subdomain, where each activated hidden neuron contributes a local slope associated with $W_{ij}^{(1)}$. Hence, from an optimization perspective, small initialization of $W_{ij}^{(1)}$ leads to a small initial slope scale (e.g., under He initialization), which may slow down convergence, while larger initial $W_{1}$ can better match the target slope scale, thereby facilitating optimization.

 \item {\bf{Poor initialization of partition points}}.

 Even with appropriately scaled weights, a second issue arises from the multiplicative coupling between the weight and bias. From Eq.~\eqref{partition_hyperplane}, the partition point induced by the $i$-th neuron is $x_p^{(i)} = -b_i^{(1)}/W_i^{(1)}$. High-frequency functions typically require multiple spatially distributed partition points to resolve local variations across the input domain. Under the He initialization $b_i^{(1)} = 0$, however, all partition points collapse to the origin regardless of the weight magnitudes. To place the partition points within the computational domain $[-1,1]$, the bias must scale proportionally with the weight, i.e., $|b_i^{(1)}| \sim \mathcal{O}(|W_i^{(1)}|)$. This coupling constraint must hold throughout training: as $W_i^{(1)}$ evolves to capture steep local slopes, $b_i^{(1)}$ must track it proportionally, or the partition points drift toward the origin, impeding the network's ability to resolve high-frequency features.
    
\end{itemize}


\subsubsection{Reparameterized \texorpdfstring{$\operatorname{ReLU}$}{ReLU} and \texorpdfstring{$\tanh$}{tanh} neural networks}\label{section: reparameterization of DNN framework}

Inspired by the possible failure mechanisms of the ReLU NNs discussed in Section~\ref{Methodology: ReLU neural networks}, we develop a reparameterized neural network (RepNN) for high-frequency and multiscale problems. Specifically,  we first present the reparameterized ReLU NN and its deep extension in Section~\ref{Subsection: reRELU-DNN}, then extend it to tanh activation in Section~\ref{Section: redeepNN} and finally propose a reparameterized tensor DNN to address high-dimensional problems in Section~\ref{Section:RetensorNN}.   \textcolor{black}{Further, we note that the proposed reparameterization is fundamentally designed for coordinate-based DNNs.}

\subsubsubsection{Reparameterized ReLU neural networks}\label{Subsection: reRELU-DNN}

Building on the analysis in Section~\ref{Methodology: ReLU neural networks}, we propose a reparameterized strategy that decouples the weight and bias in the first hidden layer, enabling independent control of the local slope scale and the partition points. Specifically, we reparameterize the weight and bias pair $(W_1,b_1)$ by introducing a slope control vector $W_1$ and a shifted bias vector $\tilde b_1$. The resulting feed-forward propagation reads:
\begin{equation}\label{eq:shallow-dnn-reparameterization}
\begin{cases}
Z_1=\relu \!\left(W_{1} \odot (x\mathbf{1}_m+\tilde b_1)\right), \\
u_{\bm{\theta}}(x) = W_2 Z_1 + b_2,
\end{cases}
\end{equation}
where \(\mathbf{1}_m\), \( W_1 =(W_i^{(1)}) \), \( \tilde{b}_1 = (\tilde{b}_i^{(1)}) \in \mathbb{R}^m \). This construction corresponds to the following reparameterization of the vanilla ReLU NN Eq.~\eqref{eq:ssingle-hidden-layer relu}:
\begin{equation}
(W_{1},b_{1}) = (W_{1},W_{1} \odot \tilde b_1).
\end{equation}

The reparameterization addresses the two failure mechanisms identified in Section~\ref{Methodology: ReLU neural networks} through a unified design. First, $W_1$ now appears both as the weight controlling local slopes and as a multiplicative factor in the bias term $b_1 = W_1 \odot \tilde b_1$. This coupling resolves the partition-point degeneracy: since $\bm{x}_p = -\tilde b_1$, the partition points are determined solely by $\tilde b_1$, independently of the slope scale encoded in $W_1$. The two quantities are thus decoupled - $W_1$ sets the local slope magnitudes, while $\tilde b_1$ sets the partition geometry. Second, the reparameterized weights and biases remain trainable. During training, $W_1$ and $\tilde b_1$ adapt simultaneously: $W_1$ learns the appropriate slope scale for each neuron, while the bias $b_1 = W_1 \odot \tilde b_1$ tracks $W_1$ proportionally, 
\textcolor{black}{thereby encouraging the partition points to remain within the computational domain.}
From this perspective, the proposed RepNN functions as a frequency-adaptive method, learning the frequency-related components from data rather than relying on pre-defined multiscale features.

For initialization, (1) we sample $\tilde b_i^{(1)} \sim \mathcal{U}(-1,1)$ ($\mathcal{U}$ denotes the uniform distribution) to distribute the initial partition points approximately uniformly over the computational domain, avoiding the degeneracy of zero-bias initialization, and (2) the weight $W_i^{(1)}$ is drawn from a Gaussian distribution with the variance $\nu_s^2$:
\begin{equation}
 W_{i}^{(1)} \sim \mathcal{N}(0,\nu_s^2).
\end{equation}
Intuitively, choosing a large $\nu_s$ enlarges the range of initial slope magnitudes, enabling the network to capture steep local variations from the start of training. 


\begin{figure}[!ht]
 \centering
 \includegraphics[width=1.0\textwidth]{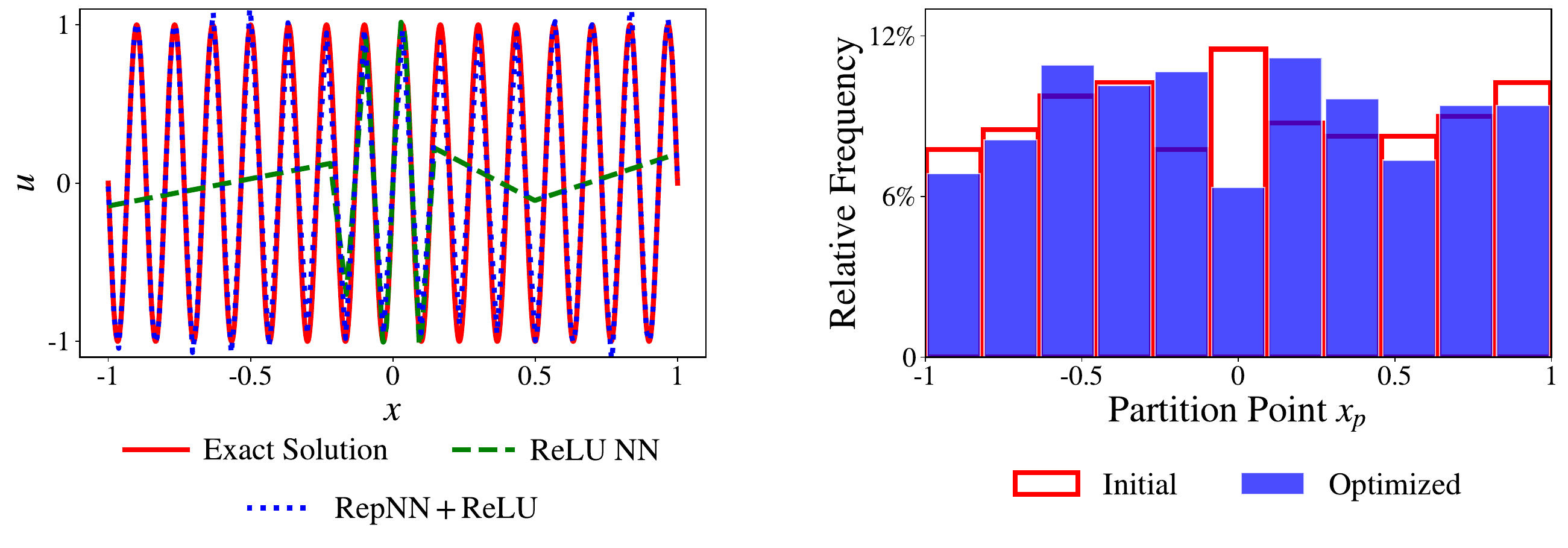}
 \caption{1D high-frequency function Eq.~\eqref{Methodology: high-frequency}: The exact solution and predictions obtained by the ReLU NN Eq.~\eqref{eq:ssingle-hidden-layer relu} and the reparameterized ReLU NN Eq.~\eqref{eq:shallow-dnn-reparameterization} (left). (Normalized) Histograms of the initial and optimized partition points in the reparameterized ReLU NN Eq.~\eqref{eq:shallow-dnn-reparameterization} (right), 
 \textcolor{black}{where 6 optimized partition points   (1.5\%)  lie outside $[-1,1]$ and are not shown.}}
 
 \label{fig Methodology: re_parameterization}
\end{figure}

We now apply the reparameterized ReLU NN Eq.~\eqref{eq:shallow-dnn-reparameterization} to the high-frequency function Eq.~\eqref{Methodology: high-frequency}, with the initialization
\begin{equation}\label{Methodology: wswmb1}
 W_{i}^{(1)} \sim \mathcal{N}(0,10^2), \quad
\tilde b_{i}^{(1)} \sim \mathcal{U}(-1,1).
\end{equation}
All other experimental settings are kept the same as those in Section~\ref{Methodology: ReLU neural networks}. 
As shown in Fig.~\ref{fig Methodology: re_parameterization} (left), the reparameterized ReLU NN successfully captures the high-frequency components of the target function. 
In addition, Fig.~\ref{fig Methodology: re_parameterization} (right) shows that both the initial and optimized partition points are well distributed within the computational domain, whereas the vanilla ReLU NN fails to achieve this behavior (see Fig.~\ref{fig Methodology: fit high-frequency} (right)).

When applying the reparameterized ReLU NN Eq.~\eqref{eq:shallow-dnn-reparameterization} to highly oscillatory and multiscale problems, a natural approach to improve accuracy is to increase the number of neurons in the hidden layer. However, this strategy is computationally expensive \cite{cheng2026generalized}. A more efficient alternative is to deepen the network. We empirically find that reparameterizing only the weights and biases of the first hidden layer is able to achieve good accuracy, which can be understood as an adaptive feature expansion and also simplifies the implementation. We note that a similar strategy has also been adopted in ~\cite{cheng2026generalized,Structured_First_Layer}. The forward pass of the resulting reparameterized ReLU DNN with $K$ layers is then defined as follows:
\begin{equation}\label{eq:deep-dnn-reparameterization}
\begin{cases}
Z_1=\relu \!\left(W_{1} \odot (x\mathbf{1}_m+\tilde b_1)\right),\\
Z_l = \relu(W_l Z_{l-1}+b_l),& l=2, \cdots, K-1,\\
u_{\bm{\theta}}(x) = W_K Z_{K-1}+b_K.\\
\end{cases}
\end{equation}
The weights and biases are defined consistently with
Eqs.~\eqref{eq:shallow-dnn-reparameterization} and~\eqref{DNNs}.

We now quantify the output and slope magnitudes of the reparameterized ReLU DNN, thereby providing guidance for designing an appropriate initialization strategy.
Consider a reparameterized ReLU DNN in which each hidden layer has a width \(H\), and the output layer consists of a single neuron. Following commonly used initialization strategies, we sample each remaining weight independently from a Gaussian distribution and initialize the biases to \(0\):
\begin{equation}
 W_{ij}^{(l)} \sim \mathcal{N}(0,\nu_l^2),
 \quad
 b_i^{(l)} = 0,
 \quad l=2,\cdots,K.
\end{equation}
We focus on a local region of the input domain where all ReLU neurons are activated (i.e., their pre-activations are positive), so that the network reduces to a linear function of the input. To simplify the analysis, we further assume that the input lies in a neighborhood of size \(\mathcal{O}(H^{-1})\) and ignore bias terms. Under these conditions, the network output and its derivative with respect to the input \(x\) can be written as
\begin{equation}\label{magnitude relu:u}
\begin{cases}
u_{\bm{\theta}}(x) = W_K W_{K-1}\cdots W_2 W_1 x,\\[2mm]
\displaystyle
\frac{\mathrm{d} u_{\bm{\theta}}}{\mathrm{d} x}(x) = W_K W_{K-1}\cdots W_2 W_1,
\end{cases}
\end{equation}
where \(W_1 \in \mathbb{R}^{H}\), \(W_l \in \mathbb{R}^{H \times H}\) for \(l = 2,\cdots,K-1\), and \(W_K \in \mathbb{R}^{1 \times H}\).

The scalar quantity \(S \triangleq W_K W_{K-1} \cdots W_2 W_1\) can be expanded as a sum over all paths through the network. Writing the matrix products explicitly,
\begin{equation}\label{eq:path-expansion}
S = \sum_{i_1=1}^{H} \sum_{i_2=1}^{H} \cdots \sum_{i_{K-1}=1}^{H}
W^{(K)}_{1,i_{K-1}} \; W^{(K-1)}_{i_{K-1},i_{K-2}} \;\cdots\; W^{(2)}_{i_2,i_1} \; W^{(1)}_{i_1},
\end{equation}
where the \(K-1\) summation indices correspond to the \(H\) neurons in each of the first \(K-1\) hidden layers. This yields \(H^{K-1}\) paths, each contributing a product of \(K\) independent zero-mean Gaussian weights. By the symmetry of the Gaussian initialization, \(\mathbb{E}[S] = 0\). Under the independent zero-mean Gaussian initialization, the second-order cross-moment between any two distinct path products is zero, while  the variance of each path product is \(\nu_K^2 \nu_{K-1}^2 \cdots \nu_2^2 \nu_s^2\).

\begin{lemma}\label{Lemma1}
Let \(Y_1, Y_2, \dots, Y_m\) be random variables with \(\mathbb{E}[Y_i] = 0\) and \(\mathrm{Var}(Y_i) = \sigma^2\) for each \(i=1,\cdots,m\), where \(\sigma>0\). Suppose further that for any \(i\neq j\),
$
\mathbb{E}[Y_iY_j]=0.
$
Define
\(
S_m = \sum_{i=1}^{m} Y_i.
\)
Then
\(
S_m = \mathcal{O}_p(\sqrt{m}\,\sigma).
\)
\end{lemma}

The proof of Lemma~\ref{Lemma1} can be found in ~\ref{Appendix: Proof}.

Applying Lemma~\ref{Lemma1} to the $H^{K-1}$  path contributions, we obtain
\begin{equation}\label{eq:scalar-estimate}
S = \mathcal{O}_p\!\left(\nu_K\nu_{K-1}\cdots\nu_2\nu_s H^{\frac{K-1}{2}}\right).
\end{equation}
Substituting this into the expressions in Eq. \eqref{magnitude relu:u} and noting that \(x = \mathcal{O}(H^{-1})\) in the local region under consideration, we arrive at
\begin{equation}\label{magnitude2 relu:u}
\begin{cases}
 u_{\bm{\theta}}(x) = S \cdot x = \mathcal{O}_p\!\left(\nu_K\nu_{K-1}\cdots\nu_2\nu_s H^{\frac{K-3}{2}}\right),\\[2mm]
 \displaystyle
 \frac{\mathrm{d} u_{\bm{\theta}}}{\mathrm{d} x}(x) = S = \mathcal{O}_p\!\left(\nu_K\nu_{K-1}\cdots\nu_2\nu_s H^{\frac{K-1}{2}}\right).
\end{cases}
\end{equation}
The factor of \(H\) difference between the two estimates arises because the output carries an additional factor of \(x = \mathcal{O}(H^{-1})\) relative to the derivative.

\begin{remark}\label{remark:he_relu}
According to Eq.~\eqref{magnitude2 relu:u}, the He initialization Eq.~\eqref{DNN:he_initialition} yields
\(u_{\bm{\theta}}(x) = \mathcal{O}_p\!\left(H^{-1}\right)\) and
\(\frac{\mathrm{d} u_{\bm{\theta}}}{\mathrm{d} x} (x)= \mathcal{O}_p(1)\) on a local region.
The relatively small initial slope may cause the vanilla DNN to struggle in fitting high-frequency functions.
\end{remark}

\begin{remark}
According to Eq.~\eqref{magnitude2 relu:u}, increasing the magnitude of
\(\frac{\mathrm{d} u_{\bm{\theta}}}{\mathrm{d} x}(x)\) generally also increases the scale of \(u_{\bm{\theta}}(x)\).
For multiscale functions with small amplitudes but steep local slopes, how to balance the scales of \(u_{\bm{\theta}}(x)\) and \(\frac{\mathrm{d} u_{\bm{\theta}}}{\mathrm{d} x}(x)\) in the RepNN remains an important open issue.
\end{remark}

\subsubsubsection{Reparameterized NN with \texorpdfstring{$\tanh$}{tanh} activation}\label{Section: redeepNN}

The effectiveness of the reparameterized ReLU NN has been demonstrated in Section~\ref{Subsection: reRELU-DNN}. However, the ReLU activation function exhibits limited nonlinearity and, more importantly, satisfies \(\mathrm{ReLU} \in C^0(\mathbb{R}) \setminus C^1(\mathbb{R})\). Consequently, ReLU-based NNs are less commonly used in PDE-related problems.
In this subsection, we extend the reparameterized ReLU NN Eq.~\eqref{eq:deep-dnn-reparameterization} to a broader class of reparameterized NN, more specifically, the DNNs with $\mathrm{tanh}$ activation that is widely used in scientific machine learning models, e.g., physics-informed neural networks (PINNs) \cite{RAISSI1} and deep operator networks (DeepONet) \cite{lu2021learning}, due to its smoothness and stronger nonlinearity.

We first rewrite the expression for the output of the reparameterized DNN with $K$ layers in a more general way:
\begin{equation}\label{eq:general-deep-dnn-reparameterization}
\begin{cases}
Z_1=\sigma \!\left(W_{1} \odot(x\mathbf{1}_m+\tilde b_1)\right),\\
Z_l = \sigma(W_l Z_{l-1}+b_l),& l=2, \cdots, K-1,\\
u_{\bm{\theta}}(x) = W_K Z_{K-1}+b_K,
\end{cases}
\end{equation}
where $\sigma$ denotes a nonlinear activation function, and  $\sigma  = \tanh$ here. All trainable parameters are defined in the same way as in the reparameterized ReLU DNN Eq.~\eqref{eq:deep-dnn-reparameterization}.  
The transition from ReLU to tanh, however, calls for a careful discussion of the partition-point analogy. For ReLU, the partition point $x_p = -\tilde{b}_1$ marks a hard kink in the piecewise-linear output, giving rise to a genuine mesh-like decomposition of the input domain. For tanh, the same condition $W_{1} \odot(x\mathbf{1}_m+\tilde b_1)=0$  defines the sign-change boundary separating positive and negative saturation regions. Near this boundary, tanh behaves approximately linearly ($\tanh(z) \approx z$ for $|z| \ll 1$), and the local slope is proportional to $W_1$. Thus, while tanh does not produce a FEM-style grid of kinks, the reparameterization still achieves the two essential objectives identified in Section~\ref{Methodology: ReLU neural networks}: (i) large $W_1$ values create steep local transitions that can match high-frequency target features, and (ii) uniformly distributed $\tilde{b}_1$ ensure that these transitions are spread across the domain rather than concentrated at the origin. The key difference is that tanh provides \textit{soft} rather than \textit{hard} partitioning - the transition from negative to positive saturation is smooth, controlled by the local slope $W_1$. This soft partitioning, combined with tanh's built-in saturation (which will be discussed in detail in \ref{Section: gradient dynamics}), makes tanh the preferred activation for the proposed framework.

Similarly, we present a quantitative analysis of the output and slope magnitudes of the RepNN with $\tanh$ activation to guide the initialization. The initialization parameters are set in the same way as in Section~\ref{Subsection: reRELU-DNN}. By the symmetry of the Gaussian initialization, the RepNN satisfies $\mathbb{E}[u_{\bm{\theta}}(x)] = 0$.
We focus on a local region satisfying $x=\mathcal{O}(H^{-1})$.

%
%

For the output amplitude, since $|\tanh(x)|<1$ for all $x\in\mathbb{R}$, each neuron in the final hidden layer produces a bounded output of variance $\mathcal{O}(1)$. The output layer computes a weighted sum of these $H$ bounded terms with weights $W_{i}^{(K)} \sim \mathcal{N}(0, \nu_K^2)$. Applying Lemma~\ref{Lemma1} to this sum yields
\begin{equation}
 u_{\bm{\theta}}(x) = \mathcal{O}_p(\nu_K H^{\frac{1}{2}} ).
\end{equation}

For the slope magnitude, we provide two estimates corresponding to different regimes. Near the origin, the Taylor expansion for $\tanh(\cdot)$ reads as
\begin{equation}
\tanh(x)=x+\mathcal{O}(x^3), \qquad \tanh'(x)=1+\mathcal{O}(x^2)
\end{equation}
implies $\tanh'(x) \approx 1$ for $|x| \ll 1$. Under this local linear approximation, each activation Jacobian reduces to the identity, and the network behaves like a deep linear network. The path-expansion argument of Section~\ref{Subsection: reRELU-DNN} then applies directly, giving
\begin{equation}\label{magnitude tanh:u}
 \frac{\mathrm{d} u_{\bm{\theta}}}{\mathrm{d} x}(x)|_{\mathrm{lin}}
 = \mathcal{O}_p\!\left(\nu_K\nu_{K-1}\cdots\nu_2\nu_s H^{\frac{K-1}{2}}\right).
\end{equation}
In contrast, when $\nu_s \gg 1$, the large first-layer weights drive most neurons deep into saturation, where $\tanh'(x) \approx 0$. In this regime, only a small fraction of the $H^{K-1}$ paths through the network remain active-those for which every neuron along the path falls within the linear region of $\tanh$. The number of such fully active paths is $\mathcal{O}(1)$, so the sum involves only a constant number of terms. Applying Lemma~\ref{Lemma1} with $m = \mathcal{O}(1)$ yields the saturation estimate
\begin{equation}
 \frac{\mathrm{d} u_{\bm{\theta}}}{\mathrm{d} x}(x)|_{\mathrm{sat}}
 = \mathcal{O}_p\!\left(
 \nu_K\nu_{K-1}\cdots\nu_2\nu_s \right).
\end{equation}
Together, the two estimates bracket the possible slope magnitudes: the linear estimate describes the expressive regime where many pathways are active, while the saturation estimate describes the initialization regime where $\tanh$ naturally constrains the gradient scale.
Similar to Remark~\ref{remark:he_relu}, the Xavier initialization leads to $u_{\bm{\theta}} = \mathcal{O}_p(1)$ and $\frac{\mathrm{d} u_{\bm{\theta}}}{\mathrm{d} x} = \mathcal{O}_p(1)$, which may also cause the vanilla DNN to struggle in fitting high-frequency functions as in ReLU NNs.

We note that a systematic comparison on the effects of the architectures, e.g. width/depth, and activation, on the predicted accuracy of RepNN based on a one-dimensional regression problem is conducted in \ref{sec:nn_architec} to justify the above discussion. 

\subsubsubsection{Reparameterized tensor DNN for high-dimensional problems}\label{Section:RetensorNN}
The RepNN resolves high-frequency features by inducing a fine partition of the input domain through the first hidden layer. However, as the input dimension \(d\) grows, the number of neurons required to maintain a given partition density grows exponentially, leading to the curse of dimensionality. To address this, we adopt a tensor product structure~\cite{novikov2015tensorizing,newman2018stabletensorneuralnetworks,wang202ing} that decomposes the high-dimensional approximation into \(d\) one-dimensional sub-problems, each handled by an independent reparameterized sub-network, and combines them through a tensor product.

\begin{figure}[!ht]
 \centering
 \includegraphics[width=1.0\textwidth]{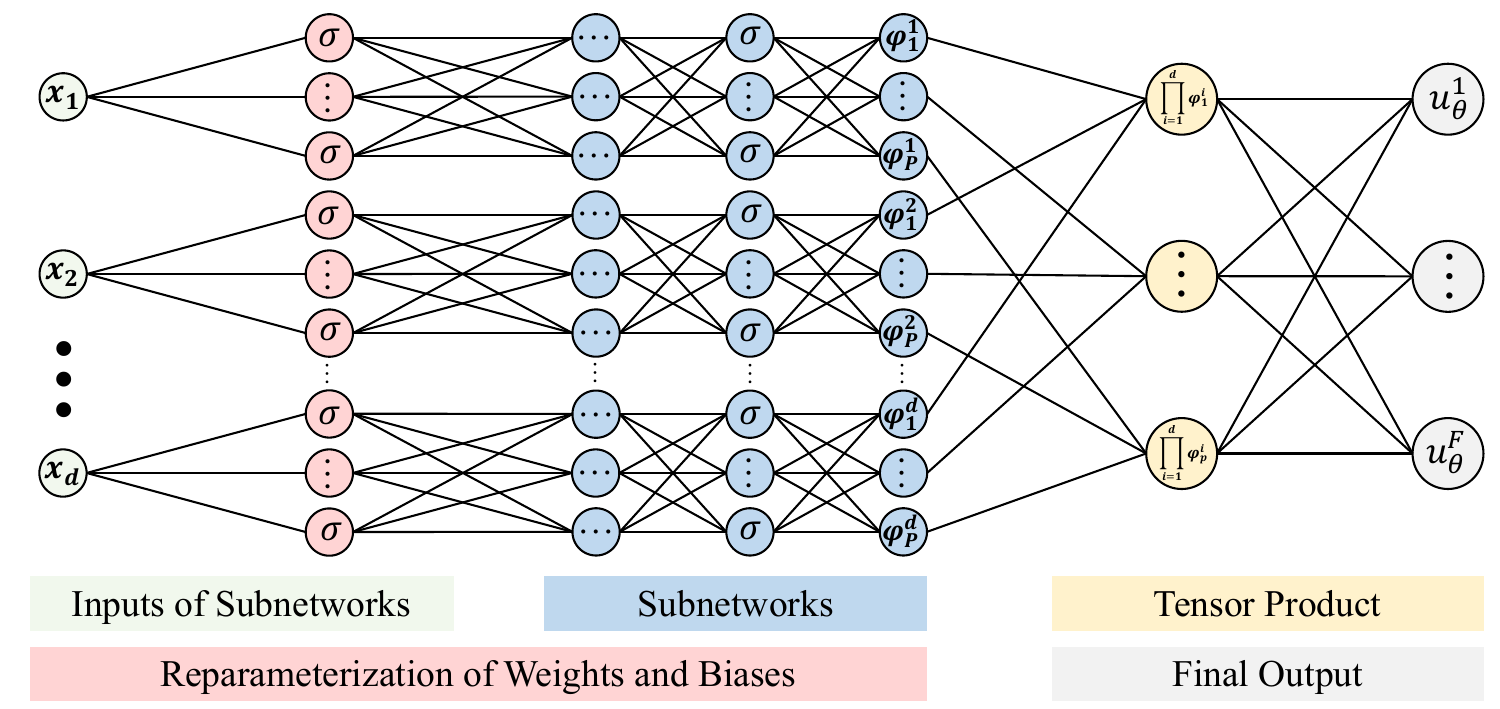}
 \caption{The architecture of reparameterized tensor deep neural networks.}
 \label{fig Methodology: RES-DNN.}
\end{figure}

Fig.~\ref{fig Methodology: RES-DNN.} illustrates the overall architecture. For a \(d\)-dimensional input \(\bm{x}=(x_1,\cdots,x_d)\) with \(d \geq 2\), the network consists of \(d\) sub-networks, each defined as in Eq.~\eqref{eq:general-deep-dnn-reparameterization} and having \(p\) neurons in its output layer. We denote these sub-networks by
\begin{equation}
u_{\bm{\theta}}^j(x_j) = \left(\varphi_1^j(x_j), \cdots, \varphi_p^j(x_j)\right) \in \mathbb{R}^p,
\quad j = 1,\cdots,d.
\end{equation}
Each sub-network inherits the reparameterized first layer, so the partition points and slope scales are controlled independently along each coordinate direction. The final \(F\)-dimensional output is obtained via the tensor product
\begin{equation}\label{tensor re dnn}
u_{\bm{\theta}}(\bm{x}) = W_f \cdot \left(u_{\bm{\theta}}^1(x_1) \odot u_{\bm{\theta}}^2(x_2) \odot \cdots \odot u_{\bm{\theta}}^d(x_d)\right) + b_f,
\end{equation}
where \(W_f = (W_{ij}^{(f)}) \in \mathbb{R}^{F \times p}\) and \(b_f = (b_i^{(f)}) \in \mathbb{R}^{F}\) are trainable parameters. 
The approximation capacity of tensor neural network (TNN) is established in \cite{TensorNN}. In  RepNN, we retain only the tensor product structure for function representation and the loss function for training the TNNs is similar to that in Eq.~\eqref{eq:dnn-mse-loss}.


We conclude this subsection with a quantitative analysis of the output and slope magnitudes. The weight matrix \(W_f\) and bias \(b_f\) are initialized as
\begin{equation}
 W_{ij}^{(f)} \sim \mathcal{N}(0,\nu_f^2),
 \quad
 b_i^{(f)} = 0.
\end{equation}
For simplicity, set \(F=1\). By the symmetry of the Gaussian initialization, \(\mathbb{E}[u_{\bm{\theta}}(\bm{x})]=0\).

For the output magnitude, the tensor product in Eq. \eqref{tensor re dnn} yields a vector in \(\mathbb{R}^p\) whose \(i\)-th component is \(\prod_{j=1}^{d} \varphi_i^j(x_j)\). From the one-dimensional analysis of Section~\ref{Section: redeepNN}, each \(\varphi_i^j(x_j)\) is the output of a reparameterized DNN with a single output neuron, which scales as \(\mathcal{O}_p(\nu_K H^{\frac{1}{2}})\). Taking the product over \(d\) independent sub-networks, each component of the tensor product has magnitude \(\mathcal{O}_p(\nu_K^d H^{\frac{d}{2}})\). The final output is obtained by multiplying this \(p\)-dimensional vector by \(W_f \in \mathbb{R}^{1 \times p}\) with entries \(W_{ij}^{(f)} \sim \mathcal{N}(0, \nu_f^2)\). Applying Lemma~\ref{Lemma1} to this sum of \(p\) terms yields
\begin{equation}
 u_{\bm{\theta}}(\bm{x})
 = \mathcal{O}_p\!\left(
 \nu_f \sqrt{p}\,\nu_K^d H^{\frac{d}{2}}
 \right).
\end{equation}

For the partial derivatives, we note that only the \(j\)-th sub-network contributes to \( \frac{\partial u_{\bm{\theta}}}{ \partial x_j}\), while the remaining \(d-1\) sub-networks enter through their output values. Specifically, for each component \(i = 1,\cdots,p\),
\[
\frac{\partial}{\partial x_j} \prod_{k=1}^{d} \varphi_i^k(x_k)
= \frac{\mathrm{d} \varphi_i^j}{\mathrm{d} x_j}(x_j) \cdot \prod_{k \neq j} \varphi_i^k(x_k).
\]
Using the slope estimates from Section~\ref{Section: redeepNN} for the derivative term and the output estimate \(\mathcal{O}_p(\nu_K H^{\frac{1}{2}})\) for each of the \(d-1\) remaining factors, then applying Lemma~\ref{Lemma1} to the sum over the \(p\) components combined by \(W_f\), we obtain
\begin{equation}
\begin{cases}
 \displaystyle
\frac{\partial u_{\bm{\theta}}}{\partial x_j}(\bm{x})|_{\mathrm{lin}}
= \mathcal{O}_p\!\left(
\nu_{K-1}\cdots\nu_2\nu_s\nu_f\sqrt{p}\,\nu_K^{d}H^{\frac{d+K-2}{2}}
\right),\\[4mm]
 \displaystyle
\frac{\partial u_{\bm{\theta}}}{\partial x_j}(\bm{x}) |_{\mathrm{sat}}
= \mathcal{O}_p\!\left(
\nu_{K-1}\cdots\nu_2\nu_s\nu_f\sqrt{p}\,\nu_K^{d}H^{\frac{d-1}{2}}
\right).
\end{cases}
\end{equation}
In the linear estimate, the derivative factor contributes \(\nu_{K-1}\cdots\nu_2\nu_s H^{\frac{K-1}{2}}\) and the \(d-1\) output factors each contribute \(\nu_K H^{\frac{1}{2}}\), combining to give the exponent \(\frac{d+K-2}{2}\). In the saturation estimate, the derivative factor loses the \(H^{\frac{K-1}{2}}\) dependence (retaining only \(\nu_{K-1}\cdots\nu_2\nu_s\)), reducing the exponent to \(\frac{d-1}{2}\). These estimates extend the one-dimensional analysis of Section~\ref{Section: redeepNN} to the tensor setting and guide the choice of initialization hyperparameters for high-dimensional problems.

\textcolor{black}{Notably, the proposed reparameterization strategy is not limited to the TNN. It can also be combined with a vanilla FNN, with only the first hidden layer reparameterized, as we will show in Section~\ref{section: 2D nonseparable function}. However, the reparameterized tensor architecture provides a more intuitive representation by explicitly decomposing the input space into coordinate-wise directions and allowing independent control over the frequency scales and partition-point distributions. Therefore, in Section~\ref{section: Numerical examples}, we adopt the reparameterized TNN to  demonstrate the effectiveness of the proposed method for high-dimensional problems.}

\section{Numerical examples}\label{section: Numerical examples}

In this section, we present extensive numerical experiments to verify the effectiveness of the proposed method, including approximating one- and four-dimensional multiscale functions, solving forward and inverse PDE problems in combination with physics-informed neural networks (PINNs) as well as  deep operator network (DeepONet). Further, the relative \(L_2\) error is employed to evaluate the accuracy of the proposed method in all test cases.
It is defined as follows:
\begin{equation}\label{Numerical_examples_L2norm}
\left\|\epsilon\right\|_2=\frac{\sqrt{\sum_{i=1}^N\left|u_{\bm{\theta}}(\bm{x}_i)-u(\bm{x}_i)\right|^2}}{\sqrt{\sum_{i=1}^N\left|u(\bm{x}_i)\right|^2}},
\end{equation}
where \(N\) represents the number of test points and is typically set to \(10000\), \(u(\bm{x}_i)\) denotes the exact or reference solution at the test point \(\bm{x}_i\), and \(u_{\bm{\theta}}(\bm{x}_i)\) represents the corresponding neural network prediction. The test points are randomly distributed over the computational domain.

In terms of the implementations for all the models, we use  PyTorch to perform all the experiments, with the data type set as \texttt{float32}. 
In addition, to show the robustness of the present model, \textcolor{black}{the RepNN utilized in this section is trained using the same experimental settings, such as the architectures, and initialization.}  In particular, (1)
the weights \(W_1\) and the auxiliary bias variables \(\tilde b_1\) in the first hidden layer are initialized as follows:
\begin{equation}\label{numerical wb}
 W_{i}^{(1)} \sim \mathcal{N}(0,10^2), \quad
\tilde b_{i}^{(1)} \sim \mathcal{U}(-1,1),
\end{equation}
and (2) the remaining weights are sampled from zero-mean Gaussian distributions, with the corresponding biases initialized to zeros. Unless otherwise stated, we set
\begin{equation}\label{eq:init_computation}
 K=5,\ p=H=300,\ \nu_{K}=\nu_{f}=H^{-\frac{1}{2}},\ \nu_{l}=1\ \text{for}\ l = 2, \cdots, K-1,
\end{equation}
throughout this study. The loss function used to train all the NNs in regression problems is the same as in Eq. \eqref{eq:dnn-mse-loss}. In addition, the loss functions for solving PDE problems in the context of PINNs and DeepONet are provided in  \ref{Appendix: pinn} and \ref{Appendix: deepOnet}, respectively. Moreover, the Adam optimizer is employed and run for 50,000 iterations to optimize the loss function, with a cosine annealing schedule for the learning rate. All training and testing are conducted on a single NVIDIA GeForce RTX 3090 GPU. Note that with the aforementioned architectures as well as initialization, we have
\begin{equation}
u_{\bm{\theta}}(\bm{x}) = \mathcal{O}_p(1),
\end{equation}
while for the partial derivative \(\frac{\partial u_{\bm{\theta}}}{\partial x_i}(\bm{x})\), \(i=1,\cdots,d\), we obtain
\begin{equation}
 \frac{\partial u_{\bm{\theta}}}{\partial x_i}(\bm{x}) |_{\mathrm{lin}}
 = \mathcal{O}_p\!\left(\nu_s H^{\frac{3}{2}}\right) \quad \text{and} \quad \frac{\partial u_{\bm{\theta}}}{\partial x_i}(\bm{x}) |_{\mathrm{sat}}
 = \mathcal{O}_p\!\left(\nu_s H^{-\frac{1}{2}} \right).
\end{equation}
We provide a discussion in \ref{Section: gradient dynamics} to address the concern that initialization with large weights of DNNs may cause gradient explosion during training based on the RepNN with $\tanh$ activation.
A sample of the data and code is available at \href{https://github.com/UQ4MI/RepNN}{https://github.com/UQ4MI/RepNN}.

{\color{black}
For fairness of comparison, the results reported for PINN-based PDE problems are obtained either from open-source implementations or directly from published studies. We do not implement every DNN model tested in the regression problem in Section~\ref{section: 1D multiscale function} for each problem, because implementation choices—such as the weighting of individual loss terms—substantially affect PINN training, especially for problems with highly oscillatory solutions.
}

\subsection{Regression problem}\label{section: Fitting problem}

\subsubsection{1D high-frequency function}\label{section: 1D high-frequency function}

We first apply the RepNN to approximating the following high-frequency target function:
\begin{equation}\label{sec100pi High-frequency}
u(x)=\sin(100\pi x), \qquad x\in[-1,1].
\end{equation}
To eliminate the effect of training data on the performance, we assume that we have sufficient measurements on the target function, and we employ the minibatch strategy in the training with a batch size \(N=512\) in Eq.~\eqref{eq:dnn-mse-loss}.

As shown in Fig.~\ref{fig Methodology: loss_dnnre}, the vanilla DNN achieves training losses on the order of $10^{-1}$, accompanied by relatively large approximation errors. In contrast, the proposed RepNN  significantly improves both optimization and approximation performance. For the RepNN, the training loss decreases to $6.96\times10^{-7}$, which is approximately 5 orders smaller compared to the corresponding vanilla DNN. The associated testing error is reduced to $1.24\times10^{-3}$,  demonstrating the effectiveness of the proposed reparameterization in capturing highly oscillatory features.

\begin{figure}[!ht]
    \centering
    \includegraphics[width=1.0\textwidth]{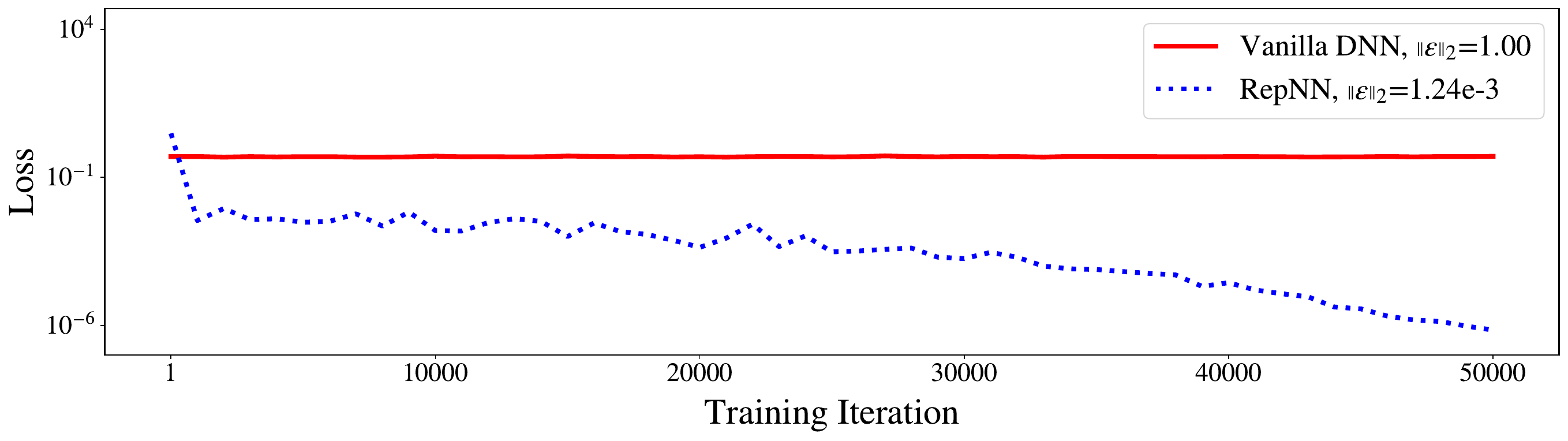} 
    \caption{1D high-frequency function Eq.~\eqref{sec100pi High-frequency}: Training loss curves and relative $L_2$  errors $\|\epsilon\|_2$ of the vanilla DNN Eq.~\eqref{DNNs} and the RepNN Eq.~\eqref{eq:general-deep-dnn-reparameterization}. The definition of $\|\epsilon\|_2$ is given in Eq.~\eqref{Numerical_examples_L2norm}.}
    \label{fig Methodology: loss_dnnre}
\end{figure}

We further conduct two ablation studies on  \(W_1\) and  \(b_1\) in the RepNN Eq.~\eqref{eq:general-deep-dnn-reparameterization}. We first investigate the effect of the initialization scale of \(W_1\). Specifically, we set \(\nu_s=1,10,\) and \(20\), which correspond to different initialization scales of \(W_1\).
\begin{figure}[!ht]
    \centering
    \includegraphics[width=1.0\textwidth]{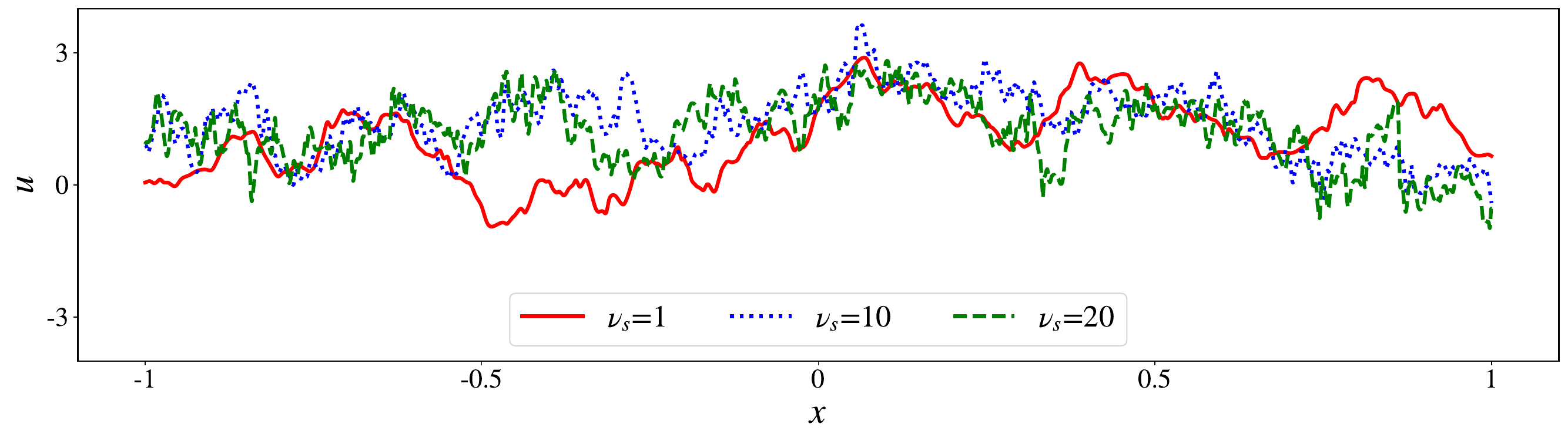} 
\caption{Visualization of the initialized RepNN Eq.~\eqref{eq:general-deep-dnn-reparameterization} with different values of $\nu_s$.}
    \label{fig Methodology: Visualization different variances.}
\end{figure}

Fig.~\ref{fig Methodology: Visualization different variances.} illustrates the initialized outputs of the RepNN with different values of \(\nu_s\). As observed, increasing \(\nu_s\) produces more pronounced oscillatory patterns in the initial output, which is consistent with the analysis in Section~\ref{Section: redeepNN}. We present in Fig.~\ref{fig Methodology: w loss and error.} the training loss curves and relative \(L_2\) errors obtained with different values of \(\nu_s\). As shown, the losses for the RepNNs with \(\nu_s=10\) and \(20\)  decrease faster than that with $\nu_s = 1$ at the first thousand iterations. 
\textcolor{black}{In addition, the RepNNs with \(\nu_s=1,10,\) and \(20\) (1) exhibit similar convergence behavior during training and (2) achieve relative \(L_2\) errors of \(2.71\times10^{-3}\), \(1.24\times10^{-3}\), and \(1.13\times10^{-3}\), respectively.}  
Although different values of \(\nu_s\) produce varying initial outputs, all testing errors remain at the same order of magnitude. Since \(W_1\) is trainable and can be automatically adapted during training, the RepNN is robust and insensitive to the value of \(\nu_s\) within the tested range.

\begin{figure}[!ht]
    \centering
    \includegraphics[width=1.0\textwidth]{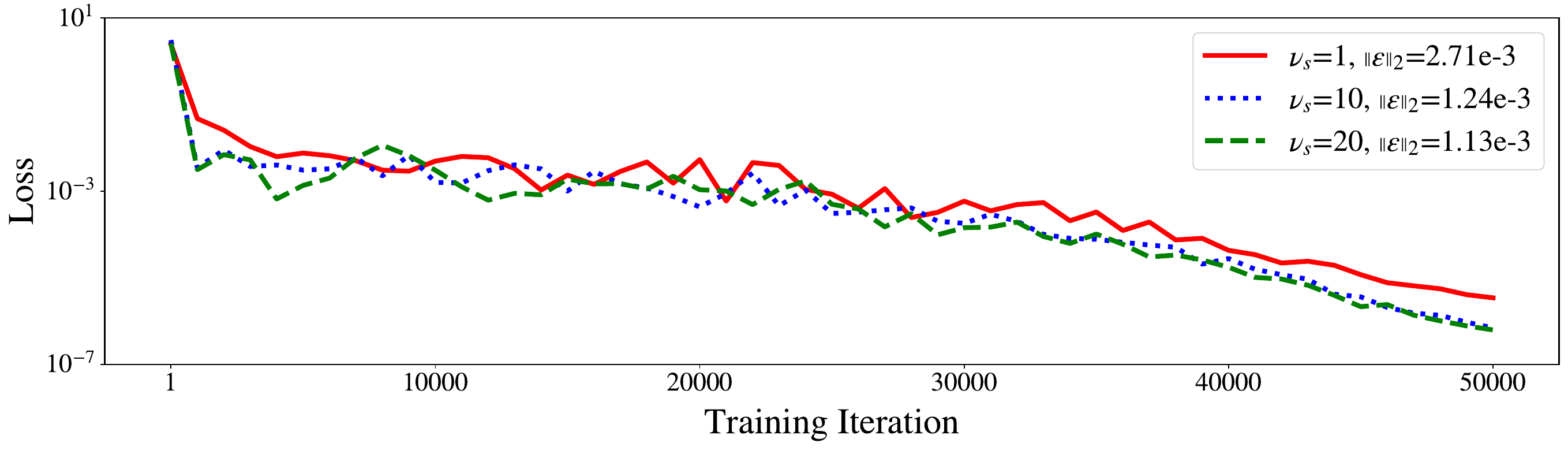} 
    \caption{1D high-frequency function Eq.~\eqref{sec100pi High-frequency}: Training loss curves and relative $L_2$  errors $\|\epsilon\|_2$ of the RepNN Eq.~\eqref{eq:general-deep-dnn-reparameterization} with different values of $\nu_s$.}
    \label{fig Methodology: w loss and error.}
\end{figure}

Next, we examine the impact of the initialization of $b_1$ on the performance of the RepNN.  Recall that \(b_1\) is reparameterized as
\begin{equation}
b_{1} = W_{1} \odot \tilde b_{1}.
\end{equation}
Accordingly,  the  partition points are given by $\bm{x}_p = -\tilde{b}_1$. Therefore, the initialization of \(\tilde b_1\) directly determines the distribution of the initial partition points. Here we consider three different initializations for \(\tilde b_1\): \(\tilde b_1=0\), equally spaced points in \([-1,1]\), denoted by \(\tilde b_1 \sim \mathcal{E}(-1,1)\), and uniformly distributed random points in \([-1,1]\), denoted by \(\tilde b_1 \sim \mathcal{U}(-1,1)\).

\begin{figure}[!ht]
    \centering
    \includegraphics[width=1.0\textwidth]{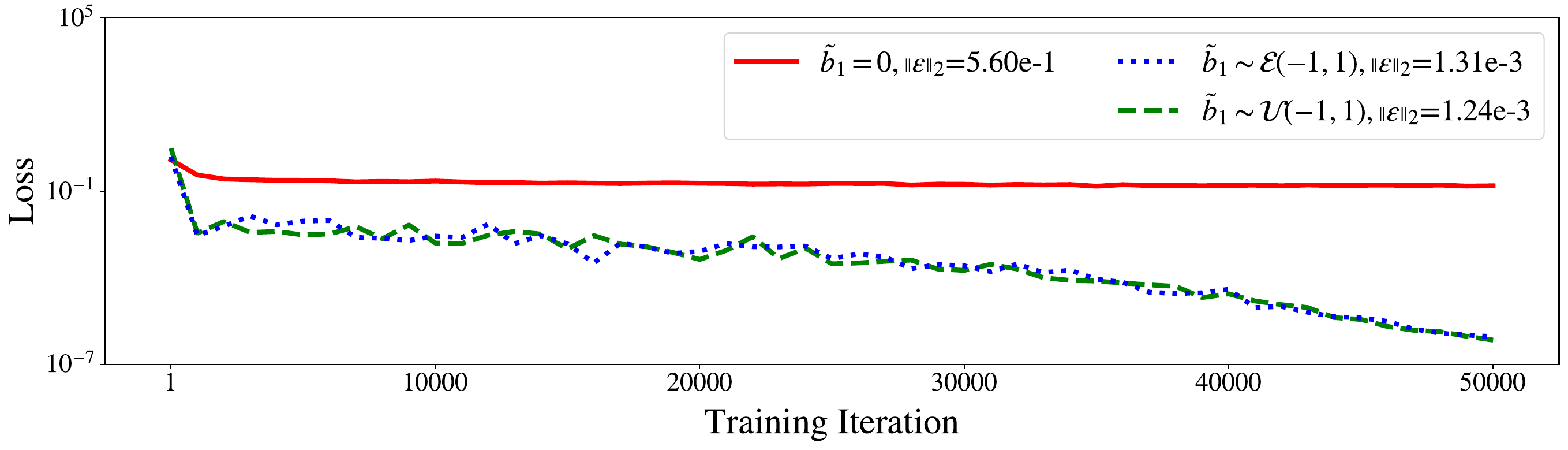} 
    \caption{1D high-frequency function Eq.~\eqref{sec100pi High-frequency}: Training loss curves  and relative $L_2$ errors of the RepNN Eq.~\eqref{eq:general-deep-dnn-reparameterization} with different initializations of $\tilde{b}_1$.}
    
    \label{fig Methodology: b loss and error}
\end{figure}

We present in Fig.~\ref{fig Methodology: b loss and error} the training loss curves and the corresponding relative \(L_2\) errors obtained by different initializations of \(\tilde b_1\). As observed, initializing all entries of \(\tilde b_1\) as zero gives a relative \(L_2\) error of \(5.60\times 10^{-1}\). In contrast, when the initial partition points are spread over \([-1,1]\), the prediction accuracy is significantly improved. Specifically, the relative \(L_2\) errors obtained with \(\tilde b_1 \sim \mathcal{E}(-1,1)\) and \(\tilde b_1 \sim \mathcal{U}(-1,1)\) are \(1.31\times 10^{-3}\) and \(1.24\times 10^{-3}\), respectively. These results indicate that distributing the initial partition points throughout the target interval is beneficial for high-frequency approximation.

Finally, we perform a comparative study to highlight the necessity of reparameterization as well as the initialization for \(W_1\) and \(b_1\) in the proposed RepNN. To this end, we consider two additional scenarios. In scenario (a), we employ the vanilla DNN as shown in 
Eq. \eqref{DNNs}, and use the same initialization in  Eqs.~\eqref{numerical wb} and \eqref{eq:init_computation} except for setting $b_1 = 0$.
In scenario (b), we use the RepNN in Eq. \eqref{eq:general-deep-dnn-reparameterization}, but we only employ the new initialization for \(b_1\) as \(\tilde b_1 \sim \mathcal{U}(-1,1)\) and keep the initializations of the remaining parameters the same as in the vanilla DNNs, i.e., Xavier initialization. For both scenarios, the trainable parameters are defined in the same manner as those in Eq.~\eqref{DNNs} and  Eq.~\eqref{eq:general-deep-dnn-reparameterization}. This comparison allows us to isolate the individual effects of the reparameterization as well as initialization of \(W_1\) and \(b_1\), thereby demonstrating the importance of simultaneous reparameterization and initialization for both quantities.
 

\begin{figure}[!ht]
    \centering
    \includegraphics[width=1.0\textwidth]{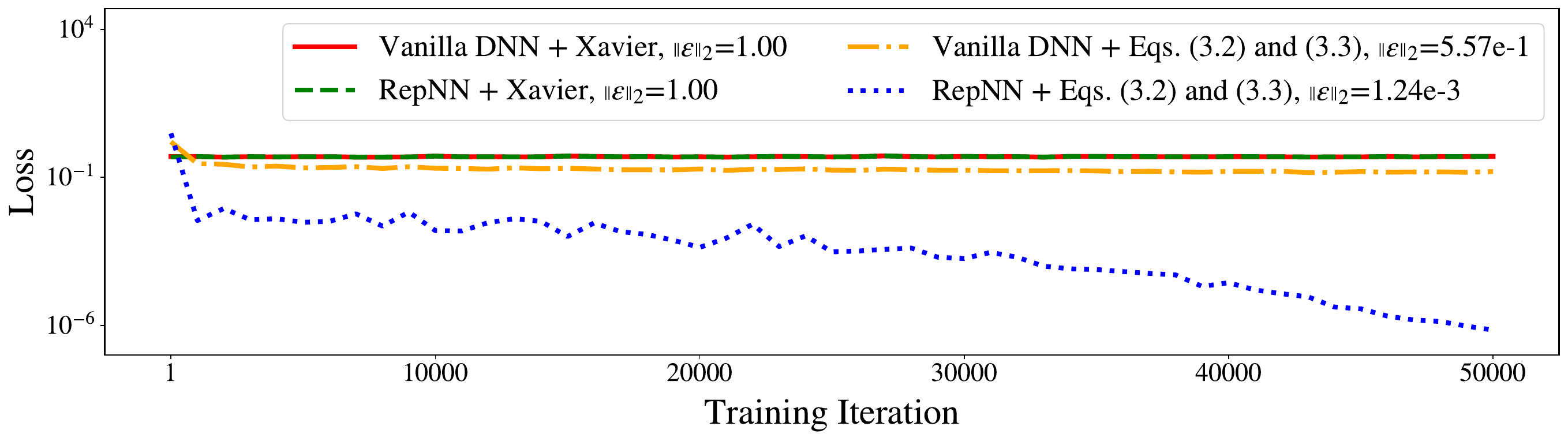} 
    \caption{1D high-frequency function Eq.~\eqref{sec100pi High-frequency}: Training loss curves  and relative $L_2$ errors of the vanilla DNN and RepNN.}
    \label{fig Methodology: loss_REWB}
\end{figure}

We present in Fig.~\ref{fig Methodology: loss_REWB} the training loss curves and the corresponding relative \(L_2\) errors of the vanilla DNN, the RepNN, and its variants. As observed, the RepNN with the proposed initialization achieves the smallest training loss and relative \(L_2\) error among all tested models. The vanilla DNN and  RepNN with the Xavier initialization for the weights  achieve training losses and relative \(L_2\) errors at the same order, which are about 6 and 3 orders greater than the RepNN with the initialization in Eqs.~\eqref{numerical wb} and \eqref{eq:init_computation}. Further, the vanilla DNN with the initialization proposed in this study reduces the training loss by approximately one order of magnitude and achieves a relative \(L_2\) error of \(5.57\times 10^{-1}\), which is about 2 orders larger than the RepNN with the initialization in Eqs.~\eqref{numerical wb} and \eqref{eq:init_computation}.   These results indicate that the  reparameterization and initialization of \(W_1\) and \(b_1\) improve the predicted accuracy in resolving high-frequency oscillatory components.

\subsubsection{1D multiscale function}\label{section: 1D multiscale function}
 We now consider the following 1D multiscale function adopted from Ref.~\cite{WANG2021113938}:
\begin{equation}\label{eq:1d-multiscale}
u(x) = \sin(\pi x) + 0.1 \sin(50\pi x), \quad x \in [-1, 1].
\end{equation}
This function has a sparse dominant spectrum, exhibiting low frequency in the macro-scale and high frequency in the micro-scale. It is well known that vanilla DNNs struggle to capture high-frequency components. In this subsection, we consider the vanilla DNN and three specified DNNs: MFF with the Fourier feature mapping parameter { \(\gamma_x=5\) \(\bigl(x\mapsto[\cos( B_xx),\sin( B_xx)],\ B_x=(B_i^{(x)})\in\mathbb{R}^{\frac{H}{2}},\ B_i^{(x)} \sim\mathcal{N} (0,\gamma_x^2)\bigl)\) }, MscaleDNN with scale coefficients \(\{1,2,4\}\), and the RepNN. In all test cases, we also assume that we have sufficient training data and the minibatch training strategy is employed. The batch size is set to $N_u=512$, and the training points $\{ x_i,u(x_i)\}_{i=1}^{N_u}$ are randomly sampled within the domain $[-1,1]$.


We present in Table~\ref{Table 1D multiscale function:} the results obtained by different models. As shown, the RepNN, MscaleDNN, and MFF all reduce the training loss by approximately 5-6 orders of magnitude and achieve testing errors 2-3 orders of magnitude smaller than the vanilla DNN. 

\begin{table}[h]	
\setlength{\abovecaptionskip}{0cm}
		\setlength{\belowcaptionskip}{0.2cm}
\caption{1D multiscale functions Eqs.~\eqref{eq:1d-multiscale} and \eqref{eq:1d-more-multiscale}: The performance of the predictions using different DNNs. \textcolor{black}{$T$ represents the wall-clock training time.} }
\label{Table 1D multiscale function:}
\centering
\begin{adjustbox}{max width=\textwidth}
\begin{tabular}{c|c|c|c|c|c|c}
\hline 
\multirow{2}{*}{\makecell{DNN Architecture}} & \multicolumn{3}{c|}{Eq.~\eqref{eq:1d-multiscale}} & \multicolumn{3}{c}{Eq.~\eqref{eq:1d-more-multiscale}}\\
\cline{2-7}
& $\mathcal{L}(\bm{\theta})$ & $\left\|\epsilon\right\|_2$ & $T$ (min) & $\mathcal{L}(\bm{\theta})$ & $\left\|\epsilon\right\|_2$ & $T$ (min) \\
\hline  
Vanilla DNN & $5.01 \times 10^{-3}$ 
          & $9.94 \times 10^{-2}$ 
          & 1.25
          & $5.06 \times 10^{-3}$ 
          & $9.94 \times 10^{-2}$ 
          & 1.28
          \\
\hline 
MscaleDNN \cite{CiCP281970} & $2.99 \times 10^{-9}$ 
                         & $8.72 \times 10^{-5}$ 
                         & 9.99
                         & $2.57 \times 10^{-4}$ 
                         & $2.06 \times 10^{-2}$
                         & 10.12 \\
\hline 
MFF \cite{WANG2021113938} & $1.57 \times 10^{-8}$ 
                         & $1.72 \times 10^{-4}$ 
                         & 1.47
                         & $2.09 \times 10^{-3}$ 
                         & $6.24 \times 10^{-2}$
                         & 1.51\\
\hline
RepNN     & $1.02 \times 10^{-8}$ 
                         & $1.41 \times 10^{-4}$ 
                         & 1.45
                         & $5.33 \times 10^{-8}$ 
                         & $3.23 \times 10^{-4}$
                         & 1.45 \\
\hline
\end{tabular}
\end{adjustbox}
\end{table}

We proceed to consider a case with more pronounced multiscale features:
\begin{equation}\label{eq:1d-more-multiscale}
u(x) = \sin(2 \pi x) + 0.1 \sin(160\pi x), \quad x \in [-1, 1].
\end{equation}
We continue to use the four DNN models utilized in the above case to approximate the function Eq.~\eqref{eq:1d-more-multiscale}.

We present in Table~\ref{Table 1D multiscale function:} the results for the more challenging case Eq.~\eqref{eq:1d-more-multiscale}. The proposed RepNN maintains strong performance, reducing the training loss by approximately 5 orders of magnitude and the testing error by roughly 2 orders of magnitude compared to the vanilla DNN, consistent with the results obtained on Eq.~\eqref{eq:1d-multiscale}. In contrast, the performance of MscaleDNN and MFF degrades substantially on this more oscillatory function: the testing error of MscaleDNN increases to $2.06\times10^{-2}$, approximately 3 orders of magnitude larger than in the first case, while the MFF yields a testing error of $6.24\times10^{-2}$, comparable to that of the vanilla DNN. As reported in Refs.~\cite{WANG2021113938,HUANG2025117751}, the accuracy of MscaleDNN and MFF depends critically on the prescribed scaling mappings, indicating that appropriate hyperparameters must be carefully selected for each problem. By contrast, the RepNN adaptively tunes $W_1$ during training and therefore remains robust to changes in the frequency characteristics of the target function. Furthermore, the RepNN and MFF incur computational costs similar to the vanilla DNN, while the MscaleDNN is substantially more expensive.

We note that with carefully designed scaling in MscaleDNN or feature expansion in MFF, the two models can achieve similar testing errors as the RepNN, as shown in Table \ref{Table 1D multiscale function 2:} and will not be discussed in detail here.


\begin{figure}[!ht]
 \centering
 \includegraphics[width=1.0\textwidth]{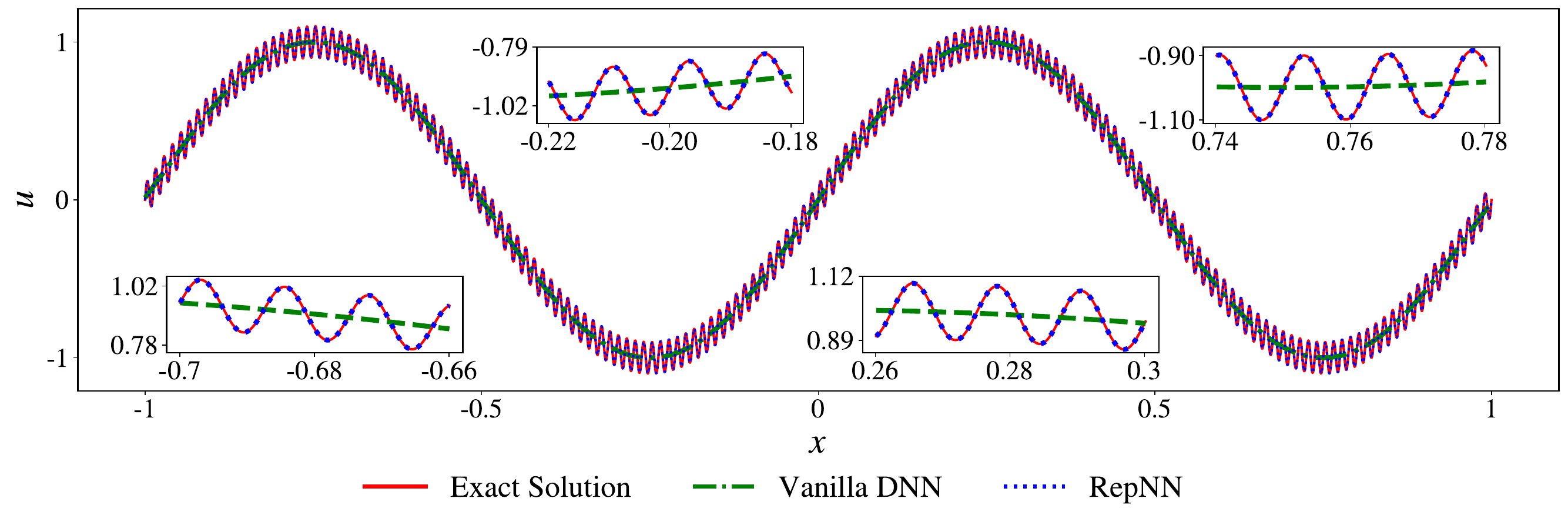} 
 \caption{1D multiscale function Eq.~\eqref{eq:1d-more-multiscale}: The exact solution and predicted solutions obtained from the vanilla DNN and RepNN.}
 \label{fig fit :one-dimensional multiscale function}
\end{figure}

\begin{table}[h]	
\setlength{\abovecaptionskip}{0cm}
		\setlength{\belowcaptionskip}{0.2cm}
\caption{1D multiscale function Eq.~\eqref{eq:1d-more-multiscale}: The performance of the predictions using MscaleDNN with scale coefficients of $\{1,2,4,8,16\}$, and MFF with the Fourier feature mapping $\gamma_x = 20.0$. Note that only the scaling mappings of these two DNNs are changed, while all other settings remain exactly the same as those in Table~\ref{Table 1D multiscale function:}.}
\label{Table 1D multiscale function 2:}
\centering
\begin{adjustbox}{max width=\textwidth}
\begin{tabular}{c|c|c|c|c|c}
\hline 
\multicolumn{3}{c|}{ MscaleDNN \cite{CiCP281970} } & \multicolumn{3}{c}{MFF \cite{WANG2021113938}}\\
\hline  
 $\mathcal{L}(\bm{\theta})$ & $\left\|\epsilon\right\|_2$ & $T$ (min) & $\mathcal{L}(\bm{\theta})$ & $\left\|\epsilon\right\|_2$ & $T$ (min) \\
\hline  
 $1.79 \times 10^{-7}$ 
& $6.04 \times 10^{-4}$ 
& 14.99
& $3.21 \times 10^{-8}$ 
& $2.50 \times 10^{-4}$
& 1.46 \\
\hline
\end{tabular}
\end{adjustbox}
\end{table}

\subsubsection{2D non-separable high-frequency function}
\label{section: 2D nonseparable function}

\textcolor{black}{
To demonstrate the effectiveness of the TNN equipped with the proposed reparameterization for  non-separable multiscale target functions, we proceed to  consider the following two-dimensional test case adopted from Ref.~\cite{WANG2024113112}:
\begin{equation}\label{2D non-separable function}
    u(x_1,x_2) = \sin\frac{30 \pi x_1}{1+9x^2_2}, \quad (x_1, x_2) \in [-1,1]^2.
\end{equation}
The target function is intrinsically non-separable and exhibits spatially varying frequency, as shown in Fig.~\ref{fig fitting : non sep}.} In this specific case, the minibatch training strategy is employed, and the batch size is set to \(N_u=4096\). The training points are randomly sampled within the computational domain \([-1,1]^2\) at each training step. In addition, we will present the performance of vanilla DNN and TNN \cite{TensorNN} for comparison.


\begin{table}[!h]	
\setlength{\abovecaptionskip}{0cm}
		\setlength{\belowcaptionskip}{0.2cm}
\caption{2D non-separable high-frequency function Eq.~\eqref{2D non-separable function}: The performance of the predictions obtained from the different DNNs.}
\label{Table 2D non-separable function}
\centering
\begin{adjustbox}{max width=\textwidth}
\begin{tabular}{c|c|c|c}
\hline 
DNN Architecture & $\mathcal{L}(\bm{\theta})$ & $\left\|\epsilon\right\|_2$ & $T$ (min)\\
\hline 
Vanilla DNN & $4.96 \times 10^{-1}$ & $1.00 \times 10^{0}$ & 1.98\\
\hline 
Vanilla TNN \cite{TensorNN} & $4.96 \times 10^{-1}$ & $ 1.00 \times 10^{0}$ & 3.91\\
\hline
RepNN (w/o TNN)& $9.55 \times 10^{-4}$ & $4.35 \times 10^{-2}$ & 3.04 \\
\hline
RepNN (w/ TNN)& $1.75 \times 10^{-4}$ & $ 1.86 \times 10^{-2}$ & 5.05 \\
\hline
\end{tabular}
\end{adjustbox}
\end{table}

\textcolor{black}{
As shown in Table~\ref{Table 2D non-separable function}, both RepNN variants, i.e. the FNN-based RepNN where only the first layers is reparameterized and the TNN-based RepNN, achieve training losses approximately 3 orders of magnitude lower than those of the vanilla DNN and the vanilla TNN, together with reductions of about 2 orders of magnitude in the testing error. Compared with the RepNN without TNN, the RepNN with TNN further improves the prediction accuracy through tensor decomposition, at the cost of increased training time due to more trainable parameters. As illustrated in Fig.~\ref{fig fitting : non sep}, the predictions from the RepNN with TNN agree closely with the exact solution over the computational domain.}

\begin{figure}[!ht]
 \centering
 \includegraphics[width=1.0\textwidth]{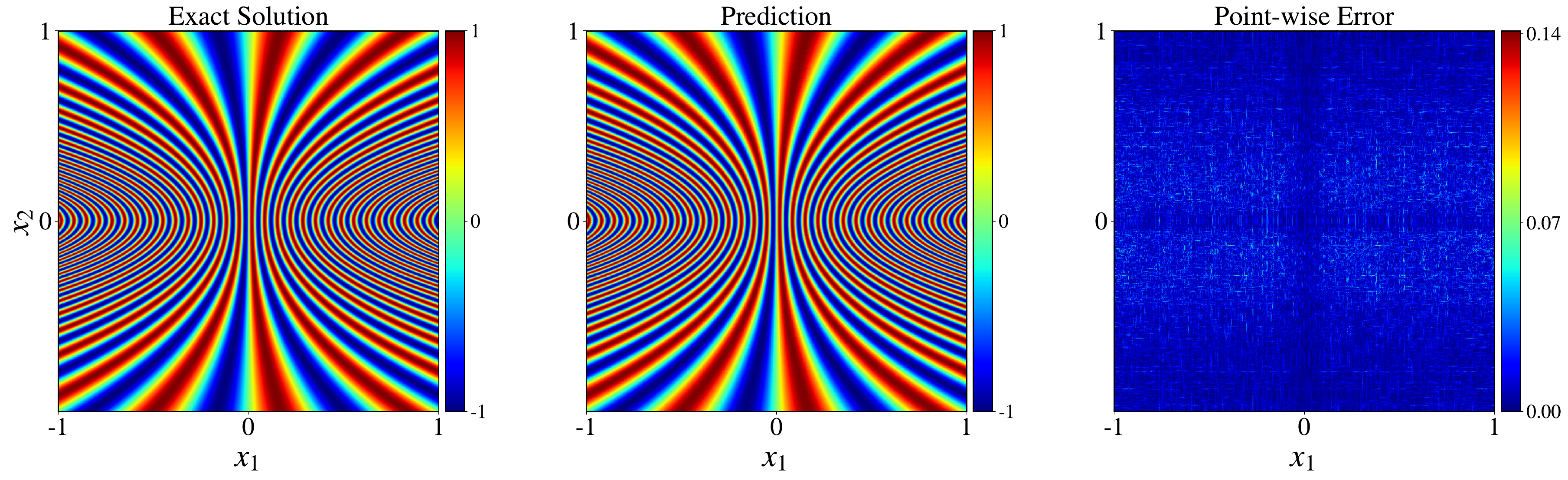} 
 \caption{2D non-separable high-frequency function Eq.~\eqref{2D non-separable function}: The exact solution and the prediction and the point-wise absolute error using the RepNN (w/ TNN).}
 \label{fig fitting : non sep}
\end{figure}

\subsubsection{4D high-frequency function}
\label{section: 4D high-frequency function}
We proceed to consider the following 4D high-frequency function adopted from Ref.~\cite{CiCP281970}:
\begin{equation}\label{eq:4d frequency}
u(\bm{x}) = \sum_{j=1}^4 e^{-x_j^2}\sin(\mu_j x_j^2),  
\quad x_j \in [-1,1],
\end{equation}
where \(\bm{x} = (x_1,\ldots,x_4)\), and the parameters are given by \((\mu_1,\mu_2,\mu_3,\mu_4)=(90,100,110,120)\pi\).
As shown in Fig.~\ref{fig fit :4-dimensional function}, the target function is relatively smooth near \(\bm{x}=0\) and becomes increasingly oscillatory away from the origin, exhibiting typical high-frequency characteristics. Moreover, larger values of \(\mu_j\) lead to stronger oscillations.

We compare the vanilla DNN, the vanilla TNN \cite{TensorNN}, and the proposed RepNN in approximating this function Eq.~\eqref{eq:4d frequency}. The DNNs are trained using a minibatch strategy, where \(N_u=4096\) training points are randomly sampled from \([-1,1]^4\) at each iteration.


\begin{table}[!h]	
\setlength{\abovecaptionskip}{0cm}
		\setlength{\belowcaptionskip}{0.2cm}
\caption{4D high-frequency function Eq.~\eqref{eq:4d frequency}: The performance of the predictions obtained from the different DNNs.}
\label{Table 4D high function}
\centering
\begin{adjustbox}{max width=\textwidth}
\begin{tabular}{c|c|c|c}
\hline 
DNN Architecture & $\mathcal{L}(\bm{\theta})$ & $\left\|\epsilon\right\|_2$ & $T$ (min)\\
\hline 
Vanilla DNN & $6.65 \times 10^{-1}$ & $7.52 \times 10^{-1}$ & 5.04 \\
\hline 
Vanilla TNN \cite{TensorNN} & $9.86 \times 10^{-1}$ & $ 9.21 \times 10^{-1}$ & 8.35 \\
\hline
RepNN & $1.65 \times 10^{-5}$ & $3.76 \times 10^{-3}$ & 9.19\\
\hline
\end{tabular}
\end{adjustbox}
\end{table}

\textcolor{black}{
We summarize in Table~\ref{Table 4D high function} the results obtained by different models. As observed, the RepNN reduces the training loss by approximately 4 orders of magnitude and achieves a testing error roughly 2 orders of magnitude smaller than the vanilla DNN and the vanilla TNN.  Fig.~\ref{fig fit :4-dimensional function} shows that the RepNN produces an accurate prediction on the slice $x_2=x_3=x_4=0$, with the maximum absolute error remaining below \(0.008\).
Although the vanilla TNN employs a tensor product structure that provides an efficient representation for the separable function  Eq.~\eqref{eq:4d frequency}, it still fails to resolve the high-frequency oscillatory features when initialized using Xavier initialization \cite{glorot}. Similarly, the vanilla DNN yields a large testing error.
}

\begin{figure}[!ht]
 \centering
 \includegraphics[width=1.0\textwidth]{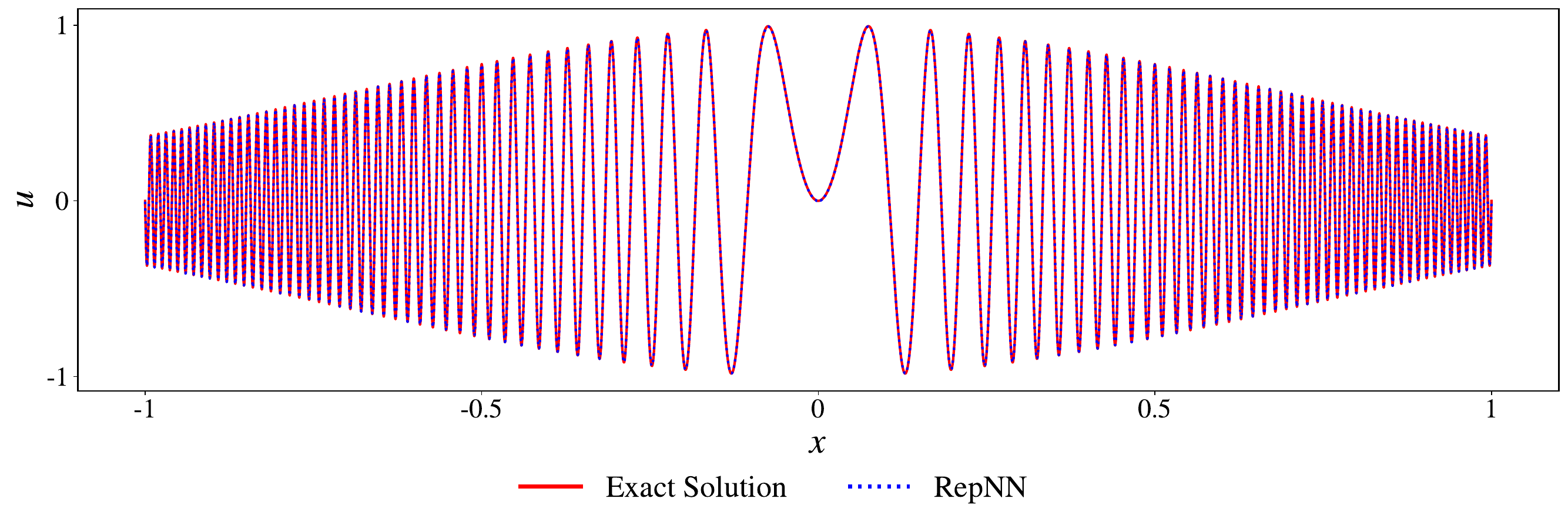} 
 \caption{4D high-frequency function Eq.~\eqref{eq:4d frequency}: The exact solution and the prediction using the RepNN on the slice $x_2 = x_3 = x_4 = 0$.}
 \label{fig fit :4-dimensional function}
\end{figure}

\subsection{Solving forward and inverse PDE problems using PINNs}\label{section: PDE problem}

\subsubsection{Forward Klein-Gordon equation}\label{section: Klein-Gordon equation}

The Klein-Gordon equation serves as the relativistic counterpart to the Schr\"{o}dinger equation and has attracted considerable attention in both analytical and numerical studies \cite{MA202641}.
The particular problem we consider here is the following one-dimensional nonlinear Klein-Gordon equation:
\begin{equation}\label{PDE: KG}
\begin{cases}
u_{tt} + \alpha u_{xx}+\beta u +\gamma u^k=f(x,t), & (x,t)\in \Omega \times (0,T),\\
u(x,0) = \cos(4 \pi x) + 0.1x(1-x), & x\in \Omega ,\\
u_t(x,0)=0, & x\in \Omega,\\
u(x,t)=h(x,t), & (x,t)\in \partial \Omega \times [0,T] .\\
\end{cases}
\end{equation}
Following Refs.~\cite{doi1043,WANG2024113112}, we set \(\Omega=[0,1]\), \(T=1\), \(\alpha=-1\), \(\beta=0\), \(\gamma=1\), and \(k=3\). We assume that the solution $u$ to the above equation is expressed as:
\begin{equation}\label{eq:kg exact solution}
u(x,t)=\cos(4 \pi x) \cos (180\pi t)+0.1x(1-x)\cos(2\pi t),
\end{equation}
from which the source term $f(x,t)$ and the corresponding initial/boundary conditions are derived analytically based on Eq.~\eqref{PDE: KG}.

As can be seen from Eq.~\eqref{eq:kg exact solution} and Fig.~\ref{fig PDE :KG exact}, the solution exhibits pronounced multiscale characteristics. In particular, it oscillates much more rapidly in the temporal direction than in the spatial direction, making it a challenging benchmark for evaluating the capability of neural networks to capture anisotropic high-frequency and multiscale features. To address such challenges, most existing approaches mitigate spectral bias by employing direction-dependent frequency embeddings \cite{WANG2024113112,ren2025rmed}. Although effective, these methods typically require problem-specific choices of frequency scales and embedding parameters, which increase the difficulty of hyperparameter tuning and may limit their applicability when prior knowledge of the underlying frequency content is unavailable.

\begin{figure}[!ht]
 \centering
 \includegraphics[width=1.0\textwidth]{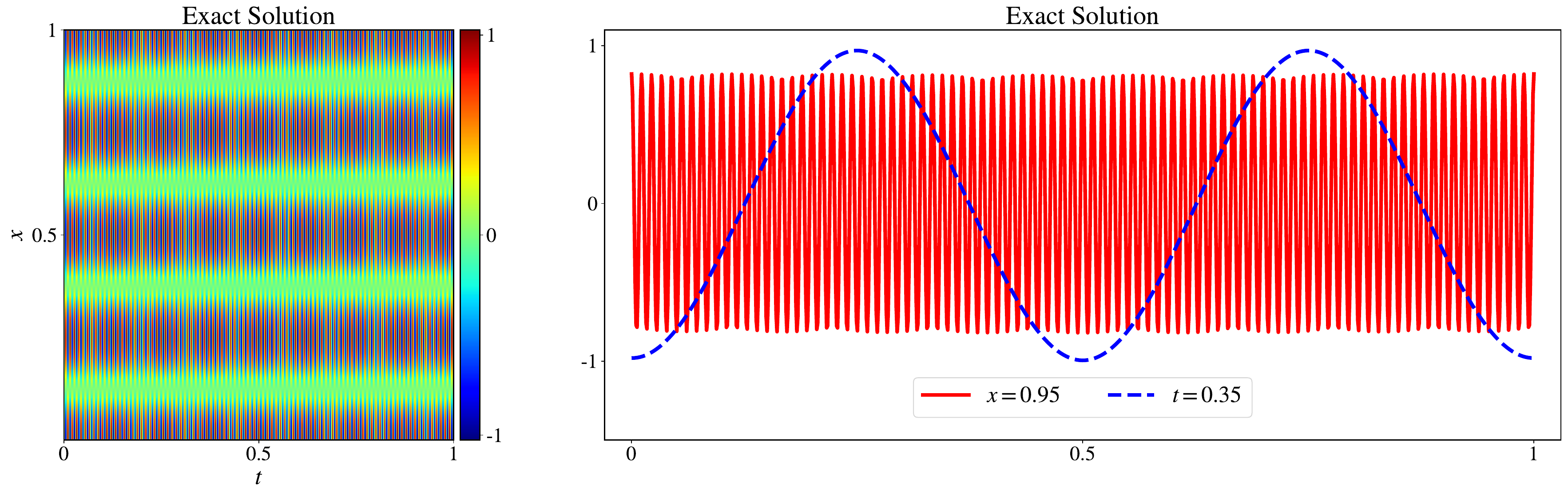} 
 \caption{PINNs for Klein-Gordon equation: The exact solution of Eq.~\eqref{PDE: KG}.}
 \label{fig PDE :KG exact}
\end{figure}

\begin{table}[h]	
\setlength{\abovecaptionskip}{0cm}
		\setlength{\belowcaptionskip}{0.2cm}
\caption{Klein--Gordon equation Eq.~\eqref{PDE: KG}: The performance of PINN methods with different DNN architectures for solving Eq.~\eqref{PDE: KG}. In MFF, \(\gamma_x\) and \(\gamma_t\) denote the embedding frequencies in space and time, respectively.}
\label{Table Klein-Gordon equation}
\centering
\begin{adjustbox}{max width=\textwidth}
\begin{tabular}{c|c|c}
\hline 
DNN Architecture & $\left\|\epsilon\right\|_2$ & $T$ (min) \\
\hline 
Vanilla DNN & $ 1.66 \times 10^{0}$ & 12.78\\
\hline
MFF$^1$ \cite{WANG2021113938} ($\gamma_x = 1$, $\gamma_t = 1$) & $ 5.93 \times 10^{0}$ & 23.12\\
\hline
MFF$^2$ \cite{WANG2021113938} ($\gamma_x = 1$, $\gamma_t = 150 $) & $5.27 \times 10^{-2}$& 22.80\\
\hline
MFF$^3$ \cite{WANG2021113938} ($\gamma_x = 150$, $\gamma_t = 150 $) & $ 7.66 \times 10^{-2}$ & 22.56 \\
\hline
RepNN & $ 4.34  \times 10^{-2}$ & 23.89 \\
\hline
\end{tabular}
\end{adjustbox}
\end{table}

\begin{figure}[!ht]
 \centering
 \includegraphics[width=1.0\textwidth]{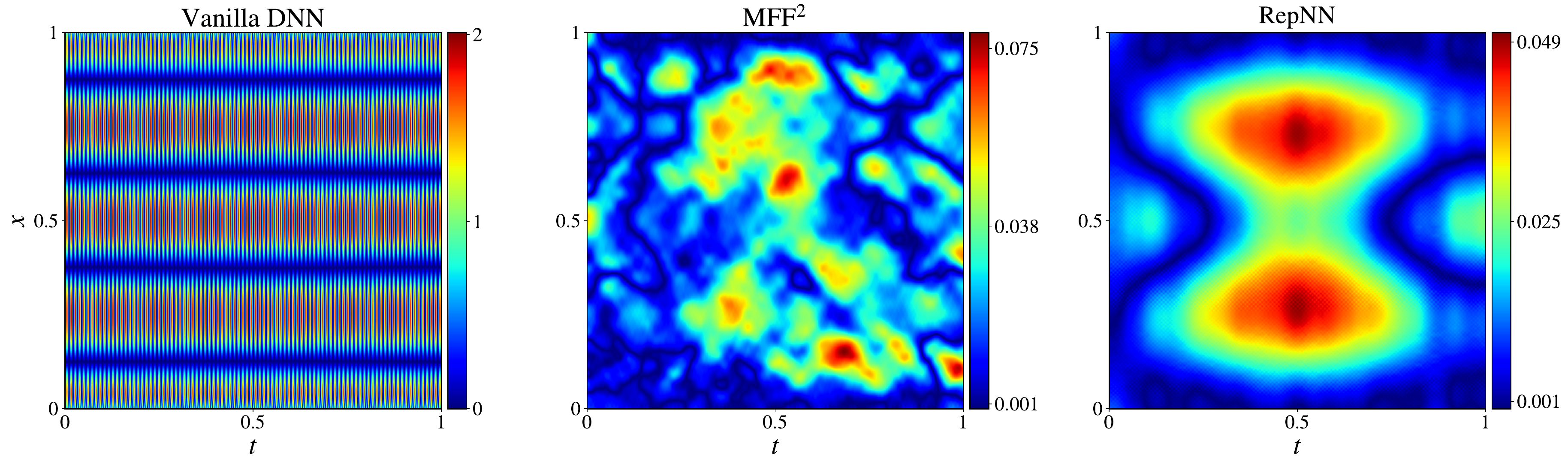} 
 \caption{Klein--Gordon equation Eq.~\eqref{PDE: KG}: The point-wise absolute errors of PINN methods with different DNN architectures for solving Eq.~\eqref{PDE: KG}.}
 \label{fig PDE :KG error}
\end{figure}

Table~\ref{Table Klein-Gordon equation} summarizes the results obtained by PINNs with different DNNs for solving Eq.~\eqref{PDE: KG}. As observed, the proposed RepNN (1) achieves the smallest relative \(L_2\) error of \(4.34\times 10^{-2}\) among all compared methods, and (2) reduces the testing error by approximately 2 orders of magnitude compared to the vanilla DNN. Fig.~\ref{fig PDE :KG error} further shows that the point-wise absolute errors obtained by the RepNN are significantly smaller than those of the vanilla DNN and the MFF. These results indicate that the RepNN is highly effective in resolving the anisotropic multiscale solution.

For the MFF, the prediction accuracy is strongly affected by the prescribed embedding frequencies. The relative \(L_2\) errors obtained with \(\gamma_x=\gamma_t=1\), \(\gamma_x=1\) and \(\gamma_t=150\), and \(\gamma_x=\gamma_t=150\) are \(5.93\times 10^{0}\), \(5.27\times 10^{-2}\), and \(7.66\times 10^{-2}\), respectively, suggesting that embedding frequencies must be carefully selected for the MFF. In contrast, the proposed RepNN starts from the general initialization in Eq.~\eqref{numerical wb} and automatically adapts the parameters during training. Consequently, the RepNN achieves improved accuracy without introducing substantial additional computational overhead.


\subsubsection{Forward Helmholtz equation}\label{section: Helmholtz equation}


The Helmholtz equation is a fundamental elliptic PDE that finds extensive applications in acoustics, electromagnetics, and seismic imaging \cite{101093gjiggac399,yang2026adtworks}. In this section, we consider the following 2D Helmholtz equation:
\begin{equation}
\begin{cases}\label{PDE hel}
 \Delta u(x,y) + u(x,y)=f(x,y),& (x,y)\in \Omega = (-1,1)\times(-1,1),\\
 u(x,y)=0,& (x,y)\in \partial \Omega.
\end{cases}
\end{equation}
The manufactured solution is given by
\begin{equation}\label{hel exact}
 u(x,y) = \sin(50 \pi x) \sin(50 \pi y).
\end{equation}
The corresponding forcing term $f(x,y)$ is derived from Eq.~\eqref{hel exact}.

As shown in Eq.~\eqref{hel exact}, the solution contains high-frequency oscillations in both spatial directions. Due to the spectral bias, obtaining accurate predictions with vanilla PINNs is challenging~\cite{zhltz}. In this example, we compare PINNs equipped with different DNNs, including the vanilla DNN, SIREN~\cite{sitzmann2020implicit}, MFF~\cite{WANG2021113938}, and the proposed RepNN.

\begin{table}[h]	
\setlength{\abovecaptionskip}{0cm}
		\setlength{\belowcaptionskip}{0.2cm}
\caption{Helmholtz equation Eq.~\eqref{PDE hel}: The performance of PINN methods with different DNN architectures for solving Eq.~\eqref{PDE hel}. 
\textcolor{black}{SIREN~\cite{sitzmann2020implicit} employs the sine function as a periodic activation function, with $w_0$ denoting the frequency scaling factor applied to the first sinusoidal layer; ``200k Adam'' denotes training with the Adam optimizer for 200,000 iterations.}
In MFF, \(\gamma_x\) and \(\gamma_y\) denote the embedding frequencies along the spatial coordinates \(x\) and \(y\), respectively.}
\label{Table Table Helmholtz equation}
\centering
\begin{adjustbox}{max width=\textwidth}
\begin{tabular}{c|c|c}
\hline 
DNN Architecture & $\left\|\epsilon\right\|_2$ & $T$ (min)  \\
\hline 
Vanilla DNN & $ 1.00 \times 10^{0}$ & 10.91  \\
\hline
SIREN \cite{sitzmann2020implicit} ($w_0=50$) & $ 1.02 \times 10^{-1}$ & 13.93  \\
\hline
SIREN \cite{sitzmann2020implicit} ($w_0=50$, 200k Adam) & $ 5.77 \times 10^{-2}$ & 56.21  \\
\hline
MFF \cite{WANG2021113938} ($\gamma_x = 50$, $\gamma_y = 50 $) & $1.27 \times 10^{-3}$ & 21.49\\
\hline
RepNN  & $ 6.36 \times 10^{-5}$ & 22.52 \\
\hline
\end{tabular}
\end{adjustbox}
\end{table}

We present in Table~\ref{Table Table Helmholtz equation} the results obtained by PINNs with different DNNs for solving Eq.~\eqref{PDE hel}. As observed, the proposed RepNN achieves the smallest relative \(L_2\) error of \(6.36 \times 10^{-5}\) among all compared DNNs. Compared to the vanilla DNN, the RepNN reduces the testing error by approximately 4 orders of magnitude. Fig.~\ref{fig PDE :Helmholtz} further shows that the point-wise absolute error obtained by the RepNN remains small over the computational domain, with the maximum prediction error less than \(1.44\times 10^{-4}\). The results confirm that the RepNN can effectively capture the high-frequency components of the Helmholtz solution.

In contrast, the vanilla DNN fails to accurately approximate the high-frequency solution, leading to a relative \(L_2\) error of \(1.00\times 10^{0}\). 
\textcolor{black}{The SIREN reduces the error to \(1.02\times10^{-1}\), and extending the Adam training to 200,000 iterations further decreases the error to \(5.77\times10^{-2}\).}  The MFF with \(\gamma_x=\gamma_y=50\) achieves a substantially lower error of \(1.27\times10^{-3}\). Nevertheless, the proposed RepNN achieves the best accuracy among all tested architectures.

\begin{figure}[!ht]
 \centering
 \includegraphics[width=1.0\textwidth]{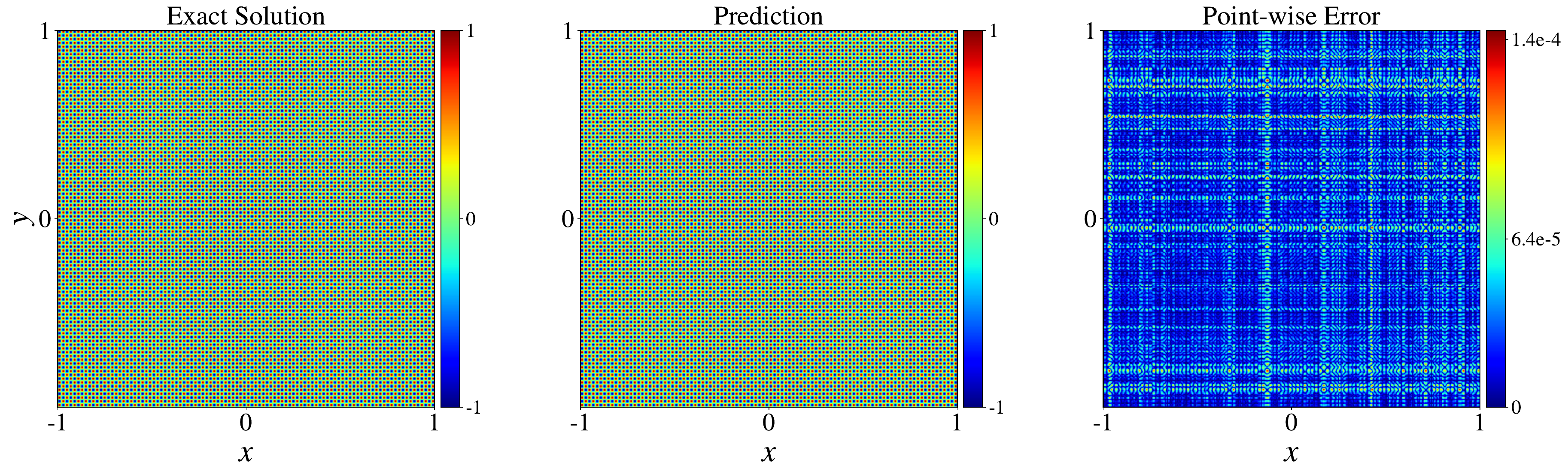} 
 \caption{Helmholtz equation Eq.~\eqref{PDE hel}: The exact solution of Eq.~\eqref{PDE hel} and the prediction and the point-wise absolute error of PINNs using the RepNN.}
 \label{fig PDE :Helmholtz}
\end{figure}

\subsubsection{Inverse Gray--Scott equation}\label{section: Gray--Scott equation}

The Gray--Scott model is a prototypical reaction-diffusion system that governs the spatiotemporal evolution of two chemical species and is widely employed to simulate pattern formation and related phenomena. Accurate identification of unknown parameters is important for understanding the underlying dynamics and enhancing the reliability of the model. Here, we consider a 2D Gray--Scott model \cite{WANG2021113938} with periodic boundary conditions on the domain $\Omega = (-1,1)^2$:
\begin{equation}\label{PDE GS MODEL}
\begin{cases}
\displaystyle
\frac{\partial u}{\partial t} = D_u \Delta u - u v^2 + F(1-u), 
& (x,y,t)\in\Omega\times(0,T], \\[6pt]

\displaystyle
\frac{\partial v}{\partial t} = D_v \Delta v + u v^2 - (F+k)v, 
& (x,y,t)\in\Omega\times(0,T], \\[6pt]

u(x,y,0)=1-\exp\!\left(-80\bigl((x+0.05)^2+(y+0.02)^2\bigr)\right), 
& (x,y) \in \Omega,\\[6pt]

v(x,y,0)=\exp\!\left(-80\bigl((x-0.05)^2+(y-0.02)^2\bigr)\right),   
& (x,y) \in \Omega,
\end{cases}
\end{equation}
where $u(x,y,t)$ and $v(x,y,t)$ represent the concentrations of the two chemical species, $F$ is the constant supply of species $u$, and $k$ is the decay rate of species $v$, and $D_u$ and $D_v$ denote the corresponding diffusion coefficients.

For this model, we set \(F=0.04\), \(k=0.06\), \(D_u=2 \times 10^{-5}\), and \(D_v=1 \times 10^{-5}\). Eq.~\eqref{PDE GS MODEL} is integrated up to the final time \(T=4000\) using the Chebfun package \cite{driscoll2014chebfun}. Specifically, the training data \(\{(x^i,y^i,t^i),(u^i,v^i)\}_{i=1}^N\) are obtained by sampling the reference solution over the time interval \([3500,4000]\) with a temporal step size of \(\Delta t=10\). All numerical settings and parameter choices described above are consistent with those reported in Ref.~\cite{WANG2021113938}.

In this example, we are interested in identifying the unknown physical parameter values: the rates \(F\) and \(k\), and the diffusion coefficients \(D_u\) and \(D_v\). We adopt the RepNN with two neurons in the output layer, yielding the approximations \(u_{\bm{\theta}}(x,y,t)\) and \(v_{\bm{\theta}}(x,y,t)\) for the concentration fields \(u\) and \(v\), respectively.

The trainable rate parameters \(F\) and \(k\) are both initialized to \(0.5\). Furthermore, to strictly enforce the positivity of the diffusion coefficients, we set \(D_u=\exp(\epsilon_u)\) and \(D_v=\exp(\epsilon_v)\), where \(\epsilon_u\) and \(\epsilon_v\) are initialized to \(-10\), following the setting in Ref.~\cite{WANG2021113938}. We use  Adam to minimize the loss function for \(200{,}000\) iterations. The trainable variables include the network parameters \(\bm{\theta}\), the reaction rates \(F\) and \(k\), and the auxiliary variables \(\epsilon_u\) and \(\epsilon_v\).

Tables~\ref{Table GS equation 2} and~\ref{Table GS diff} summarize the results obtained by PINNs with different DNNs for solving Eq.~\eqref{PDE GS MODEL}. The proposed RepNN achieves relative \(L_2\) errors of \(2.38\times 10^{-3}\) and \(1.34\times 10^{-2}\) for the concentration fields \(u\) and \(v\), respectively. Compared to the vanilla DNN, the RepNN reduces the testing error by approximately 2 orders of magnitude for \(u\) and more than 1 order of magnitude for \(v\). Fig.~\ref{fig PDE :GS} further shows that the predictions obtained by the RepNN agree well with the reference solutions at \(T=4000\). The corresponding point-wise absolute errors remain below \(0.03\) and \(0.04\) for \(u\) and \(v\), respectively. \textcolor{black}{In contrast, the vanilla DNN fails to capture the fine-scale patterns in both concentration fields and produces smooth predictions.}

\begin{table}[h]	
\setlength{\abovecaptionskip}{0cm}
		\setlength{\belowcaptionskip}{0.2cm}
\caption{Gray--Scott equation Eq.~\eqref{PDE GS MODEL}: The performance of PINN methods with the vanilla DNN and the RepNN for solving Eq.~\eqref{PDE GS MODEL} over the time interval $[3500,4000]$ with temporal step size $\Delta t = 10$.}
\label{Table GS equation 2}
\centering
\begin{adjustbox}{max width=\textwidth}
\begin{tabular}{c|c|c|c}
\hline 
DNN Architecture & $\left\|\epsilon_u\right\|_2$ & $\left\|\epsilon_v\right\|_2$ & $T$ (min) \\
\hline 
Vanilla DNN & $ 2.05 \times 10^{-1}$ & $ 4.89 \times 10^{-1}$  & 202.43\\
\hline 
RepNN & $ 2.38 \times 10^{-3}$ & $ 1.34 \times 10^{-2}$ &  405.27\\
\hline
\end{tabular}
\end{adjustbox}
\end{table}

\begin{figure}[!ht]
 \centering
 \includegraphics[width=1.0\textwidth]{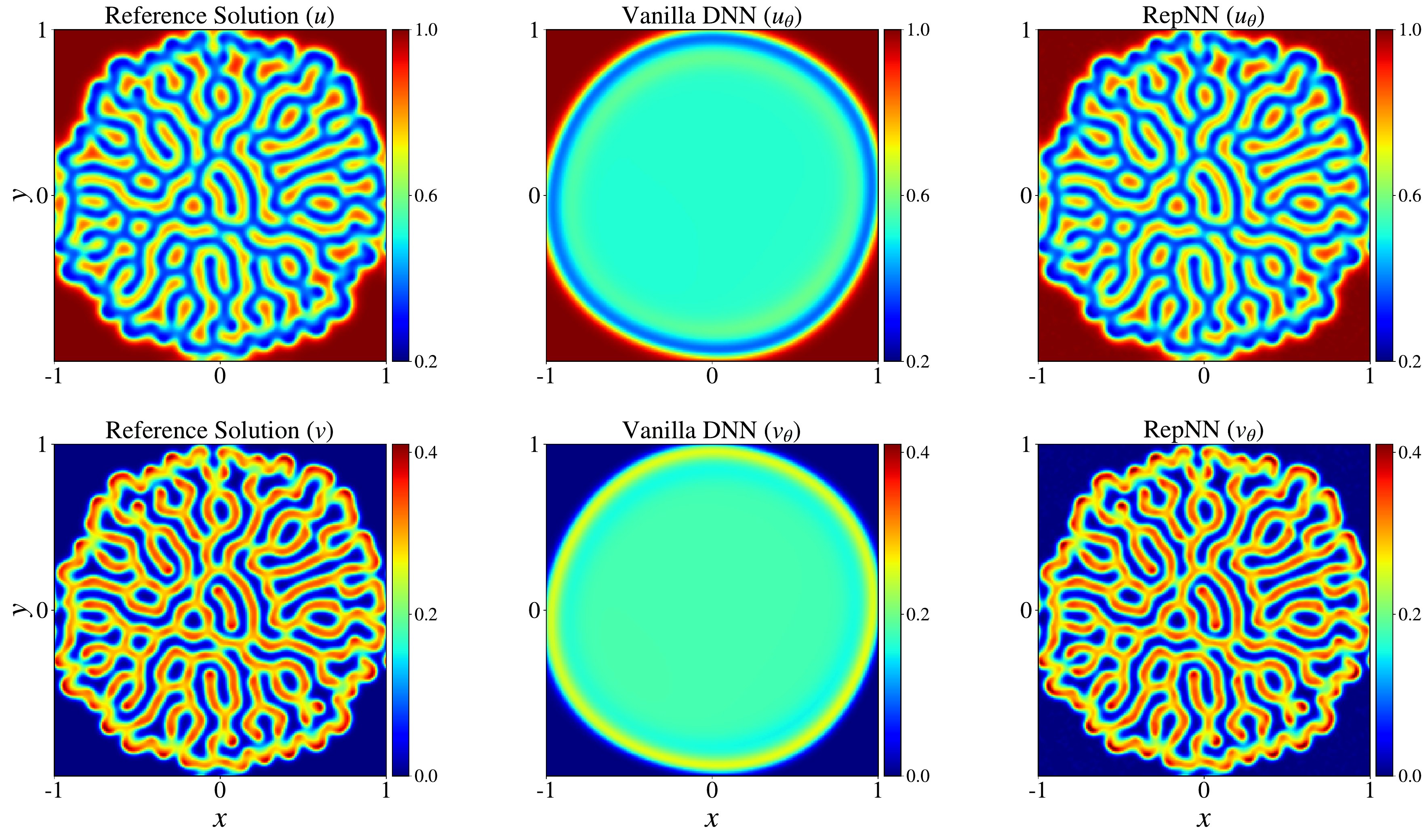} 
 \caption{\textcolor{black}{Gray--Scott equation Eq.~\eqref{PDE GS MODEL}: Reference solutions and corresponding predictions obtained by PINNs using the vanilla DNN and RepNN for the concentration fields $u$ and $v$ at $T=4000$.}}
 \label{fig PDE :GS}
\end{figure}

For the inverse problem, the RepNN also accurately identifies the unknown physical parameters. As shown in Table~\ref{Table GS diff} and Fig.~\ref{fig PDE :GS_Parameters}, the identified values of \(F\), \(k\), \(D_u\), and \(D_v\) are \(4.00\times 10^{-2}\), \(5.99\times 10^{-2}\), \(1.96\times 10^{-5}\), and \(9.95\times 10^{-6}\), respectively, which are in excellent agreement with the exact values. In contrast, the vanilla DNN gives much less accurate parameter estimates. In particular, the identified value of \(D_v\) is more than twice the exact value.

Compared with the MFF results reported in Ref.~\cite{WANG2021113938}, the proposed RepNN gives more accurate estimates for the diffusion coefficients. The absolute errors of \(D_u\) and \(D_v\) are reduced by about \(1.25\) and \(6\) times, respectively. Moreover, the RepNN also accurately recovers the reaction rates \(F\) and \(k\). These results verify the effectiveness of RepNN in identifying unknown physical parameters.


\begin{table}[h]	
\setlength{\abovecaptionskip}{0cm}
		\setlength{\belowcaptionskip}{0.2cm}
\caption{Gray--Scott equation Eq.~\eqref{PDE GS MODEL}: The approximate values of the identified physical parameters obtained with different DNN architectures. The exact values are \(F=0.04\), \(k=0.06\), \(D_u=2 \times 10^{-5}\), and \(D_v=1 \times 10^{-5}\).}
\label{Table GS diff}
\centering
\begin{adjustbox}{max width=\textwidth}
\begin{tabular}{c|c|c|c|c}
\hline 
DNN Architecture & $F$ value & $k$ value & $D_u$ value & $D_v$ value \\
\hline 
Vanilla DNN & $ 3.41 \times 10^{-2}$ & $ 5.78 \times 10^{-2}$ & $ 2.13 \times 10^{-5}$ & $ 2.30 \times 10^{-5}$ \\
\hline 
MFF \cite{WANG2021113938} & -- & -- & $ 1.95\times 10^{-5}$ & $ 9.70 \times 10^{-6}$ \\
\hline 
RepNN & $ 4.00 \times 10^{-2}$ & $ 5.99 \times 10^{-2}$ & $ 1.96 \times 10^{-5}$ & $ 9.95 \times 10^{-6}$ \\
\hline 
\end{tabular}
\end{adjustbox}
\end{table}

\begin{figure}[!ht]
 \centering
 \includegraphics[width=1.0\textwidth]{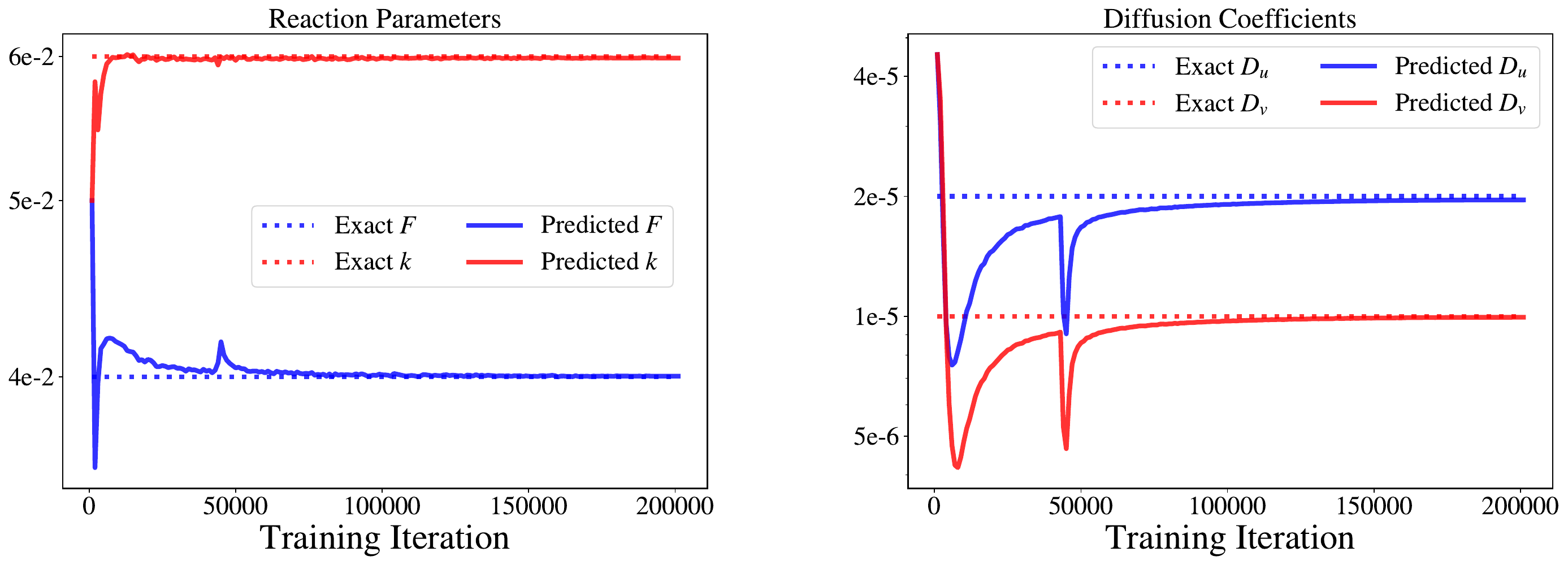} 
 \caption{Gray--Scott equation Eq.~\eqref{PDE GS MODEL}: Evolution of the inferred reaction parameters $F$ and $k$, and the diffusion coefficients $D_u$ and $D_v$, identified by the RepNN.}
 \label{fig PDE :GS_Parameters}
\end{figure}

\subsection{Earthquake problem using neural operator}\label{section: Earthquake problem}
The final numerical example is designed to show the effectiveness of RepNN in enhancing neural operators.

We consider the prediction of the dynamic response of a multi-story building subjected to seismic loading. The governing equation of motion, adopted from Refs.~\cite{liCausalityear,khodakarami2026spectral}, is given by
\begin{equation}\label{PDE Earthquake}
\mathbf{M}\ddot{\bm{x}} + \mathbf{C}\dot{\bm{x}} + \mathbf{K}\bm{x} = \mathbf{M}\boldsymbol{\iota}\ddot{u}_g(t),
\end{equation}
where \(\mathbf{M}\), \(\mathbf{C}\), and \(\mathbf{K}\) denote the mass, damping, and stiffness matrices obtained from finite element discretization, respectively; $\bm{x}(t)$ is the displacement vector representing the response of the top floor of the building; $\boldsymbol{\iota}$ is the influence vector that distributes the ground acceleration to the degrees of freedom; and $\ddot{u}_g(t)$ represents the ground acceleration induced by the earthquake.

In this study, the dataset consists of 100 earthquake ground motion records sourced from the Pacific Earthquake Engineering Research Center database (\url{https://peer.berkeley.edu/}). Each record contains the time history of the structural displacement vector $\bm{x}(t)$ and the corresponding ground acceleration $\ddot{u}_g(t)$. Records with time steps $\delta t < 0.02$ sec are first processed using a Butterworth filter with frequency (0.1-24.9) Hz and then uniformly resampled to $\delta t = 0.02$ sec. All resulting records span 50 seconds at a sampling frequency of 50 Hz. 
Fig.~\ref{fig eqrth :exact} shows representative examples of the displacement vector \(\bm{x}(t)\) and the associated ground acceleration \(\ddot{u}_g(t)\) from the dataset.

\begin{figure}[!ht]
 \centering
 \includegraphics[width=1.0\textwidth]{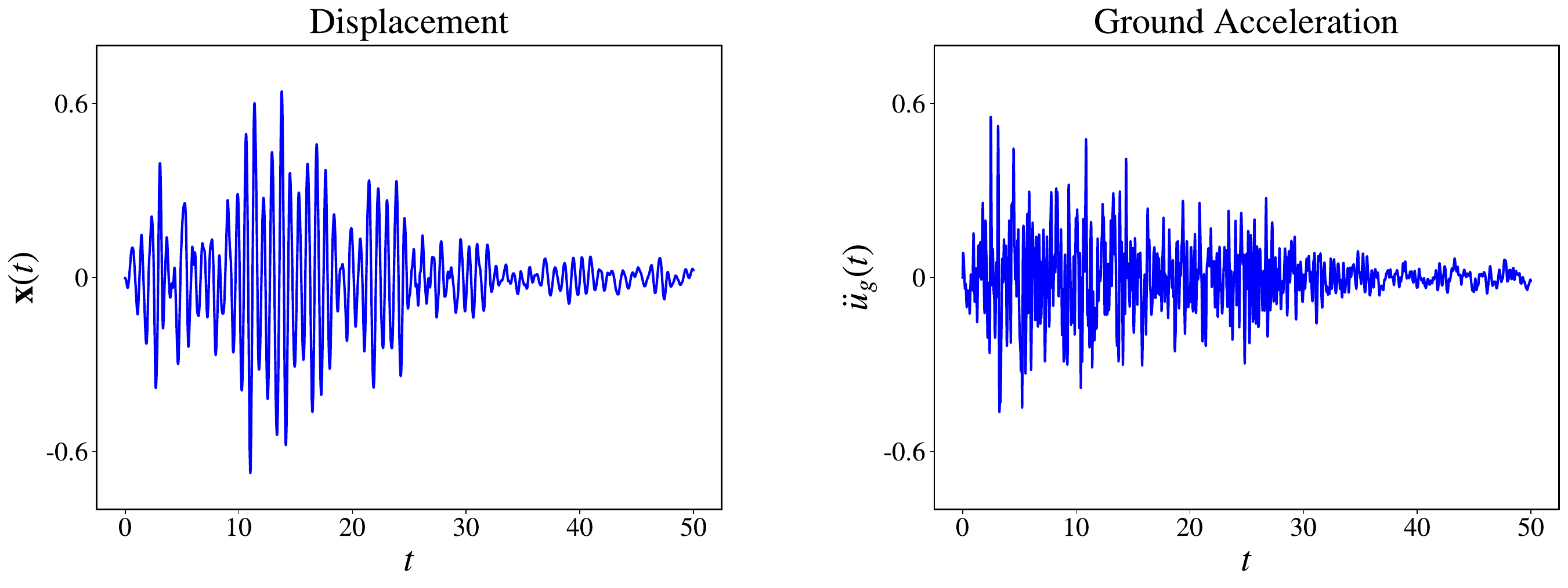} 
\caption{Earthquake problem Eq.~\eqref{PDE Earthquake}: The time histories of the displacement vector $\bm{x}(t)$ and the associated ground acceleration $\ddot{u}_g(t)$ sampled from the dataset.}
 \label{fig eqrth :exact}
\end{figure}

The objective of this example is to learn the operator that maps the ground acceleration \(\ddot{u}_g(t)\) to the corresponding displacement vector \(\bm{x}(t)\):
\begin{equation}\label{Earthquake operator}
\mathcal{G} : \ddot{u}_g(t) \mapsto \bm{x}(t).
\end{equation}
This is a challenging multiscale operator learning task, since the structural response contains both low-frequency global vibration modes and high-frequency components induced by impulsive ground motions~\cite{khodakarami2026spectral}, as shown in Fig.~\ref{fig eqrth :exact}.

We apply DeepONet~\cite{lu2021learning} to learn the target operator and consider three methods: vanilla DeepONet, improved DeepONet, and improved causality-DeepONet. The vanilla DeepONet adopts the vanilla DNN as its trunk network, whereas the improved DeepONet and improved causality-DeepONet use the RepNN as their trunk networks. The improved causality-DeepONet further employs the causal windowing approach  to enforce temporal causality following \cite{liCausalityear}. For all branch and trunk networks, the number of neurons in the output layer is set to 1024.
We use  Adam to minimize the loss function for \(200{,}000\) iterations. In our experiments, the dataset is split into 90 samples for training and 10 samples for testing.

\begin{table}[h]	
\setlength{\abovecaptionskip}{0cm}
		\setlength{\belowcaptionskip}{0.2cm}
\caption{Earthquake problem Eq.~\eqref{PDE Earthquake}: The performance of different neural operators for learning the target operator \eqref{Earthquake operator}. For the Fourier Neural Operator (FNO), we employ four Fourier layers, each with 32 channels and 300 Fourier modes. All other experimental settings follow those used for DeepONet. Notably, FNO and DeepONet share the same order of magnitude in terms of trainable parameters, comprising 2,466,273 and 1,909,148 parameters, respectively.}
\label{Table Earthquake}
\centering
\begin{adjustbox}{max width=\textwidth}
\begin{tabular}{c|c|c}
\hline 
Method & $\left\|\epsilon_{train}\right\|_2$ & $\left\|\epsilon_{test}\right\|_2$ \\
\hline 
Vanilla DeepONet & $ 1.00\times 10^{0}$ & $ 1.00\times 10^{0}$\\
\hline 
Improved DeepONet & $ (9.51 \pm 0.76) \times 10^{-2}$ &  
$ (7.01 \pm 0.77) \times 10^{-1}$ \\
\hline 
Improved Causality-DeepONet & $ (1.27 \pm 0.13) \times 10^{-2}$ &  
$ (1.49 \pm 0.15) \times 10^{-2}$ \\
\hline 
Fourier Neural Operator \cite{li2021ftric} & $ (4.32 \pm 1.35) \times 10^{-3}$ & $ (2.77 \pm 1.20) \times 10^{-2}$ \\
\hline 
Fourier Neural Operator \cite{li2021ftric} (SOAP \cite{ICLR2025_e9886640}) & $ (4.26 \pm 1.39) \times 10^{-3}$ & $ (2.47 \pm 0.75) \times 10^{-2}$ \\
\hline 
Improved Causality-DeepONet (SOAP \cite{ICLR2025_e9886640}) & $ (7.05 \pm 0.63) \times 10^{-4}$ &  
$ (1.54 \pm 0.15) \times 10^{-3}$ \\
\hline
\end{tabular}
\end{adjustbox}
\end{table}

\begin{figure}[!ht]
 \centering
 \includegraphics[width=1.0\textwidth]{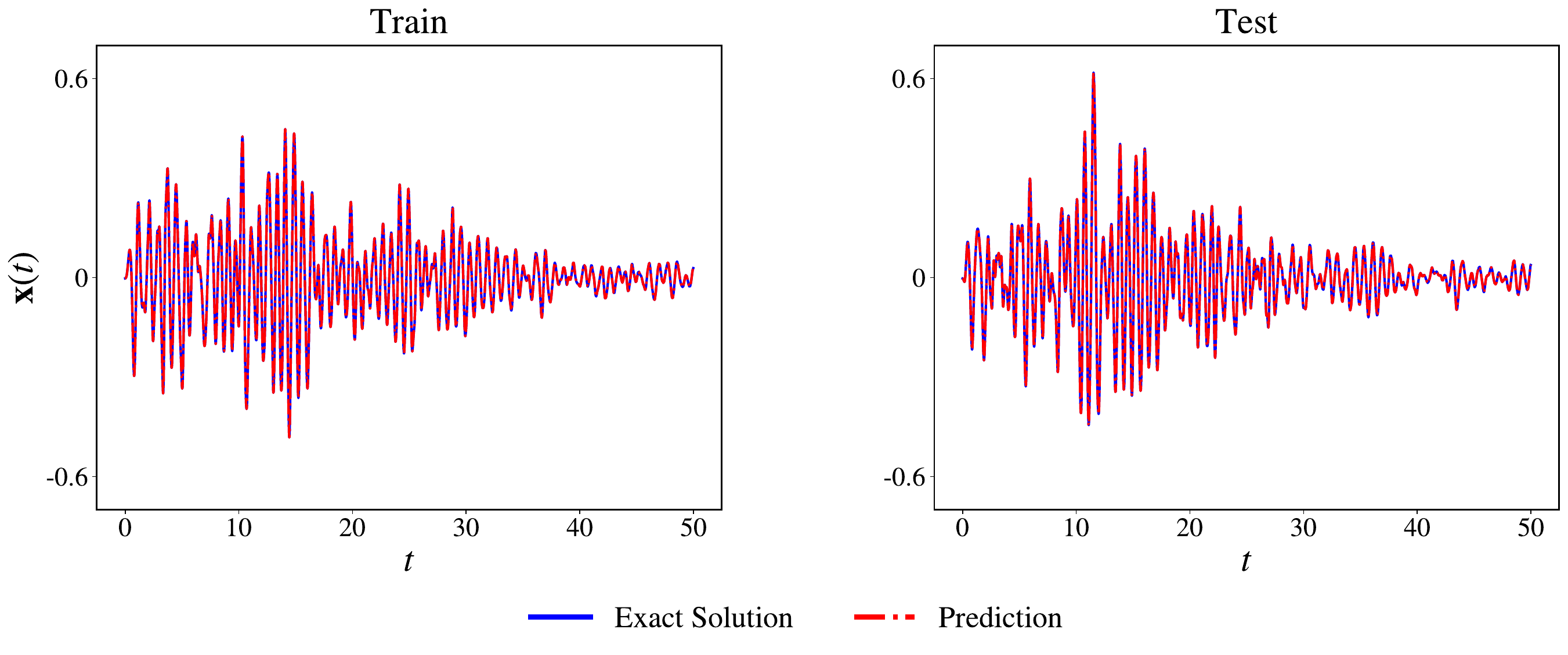} 
 \caption{Earthquake problem Eq.~\eqref{PDE Earthquake}: Exact displacements \(\bm{x}(t)\) and the corresponding predictions of the improved causality-DeepONet (SOAP) for training and test samples in the worst case.}
 \label{fig eqrth :pre}
\end{figure}

Following the definition of the relative \(L_2\) error in Eq.~\eqref{Numerical_examples_L2norm}, we report in Table~\ref{Table Earthquake} the mean and standard deviation of the relative \(L^2\) errors obtained by all methods on the training and test datasets. As observed, the improved DeepONet reduces the training error to \((9.51 \pm 0.76)\times 10^{-2}\), which is about 1 order of magnitude smaller than that of the vanilla DeepONet. However, its testing error remains at \((7.01 \pm 0.77)\times 10^{-1}\). In contrast, the improved causality-DeepONet further reduces the training and testing errors to \((1.27 \pm 0.13)\times 10^{-2}\) and \((1.49 \pm 0.15)\times 10^{-2}\), respectively, thereby showing good generalization capability on the unseen test set.

Moreover, as observed in ~\cite{khodakarami2026spectral}, higher-order optimizers can effectively mitigate spectral bias in DeepONet. Therefore, we further train the improved causality-DeepONet using the quasi second-order optimizer SOAP~\cite{ICLR2025_e9886640}, while keeping all other experimental settings unchanged. As shown in Table~\ref{Table Earthquake}, SOAP reduces the testing error of the improved causality-DeepONet from \((1.49 \pm 0.15)\times 10^{-2}\) to \((1.54 \pm 0.15)\times 10^{-3}\), leading to nearly one order of magnitude improvement. By contrast, replacing Adam with SOAP yields only marginal improvement for the FNO, reducing the testing error from \((2.77 \pm 1.20)\times 10^{-2}\) to \((2.47 \pm 0.75)\times 10^{-2}\). In addition, the improved causality-DeepONet with SOAP achieves a testing error more than one order of magnitude smaller than that of the FNO with SOAP. 
Fig.~\ref{fig eqrth :pre} further illustrates the worst-case predictions of the improved causality-DeepONet with SOAP over the dataset. The predicted displacements \(\bm{x}(t)\) remain in close agreement with the exact solutions, despite their highly oscillatory behavior. The maximum prediction errors are less than \(4.55 \times 10^{-4}\) and \(8.78 \times 10^{-4}\) for the training and test samples, respectively.
These results demonstrate the effectiveness of the RepNN in enhancing neural operators for real-world multiscale operator learning problems.

\section{Conclusion}\label{section: Conclusion}

In this paper, we have investigated recent efforts to overcome the spectral bias of DNNs and proposed RepNN, a reparameterized NN model for high-frequency and multiscale problems. The key idea is to reparameterize the weights and biases in the first hidden layer, enabling effective control of the initial slope scale and appropriate partitioning of the computational domain.

We first introduced a reparameterized ReLU NN and extended it to more general activation functions such as \(\tanh\) for PDE-related problems, and further developed a reparameterized tensor DNN for high-dimensional settings. By treating the reparameterized weights and biases as trainable parameters, RepNN provides an adaptive frequency scaling mechanism during training. We derived quantitative slope magnitude estimates and clarified the initialization needed to achieve the desired scaling. We also provided a gradient-dynamics-based intuition as \textcolor{black}{a} possible explanation for why the method may succeed, which shows that \(\tanh\) saturation naturally constrains gradient magnitudes at initialization, preventing explosion while sparse active pathways maintain trainability. 


Extensive numerical experiments - including function approximation, forward and inverse PDE problems, and operator learning - validate the {\emph{effectiveness and robustness}} of RepNN, demonstrating that it provides a principled, lightweight, and flexible approach for applying DNNs to multiscale problems.

For future work, RepNN can be employed to solve more realistic high-dimensional multiscale PDEs, such as the Boltzmann transport equation~\cite{LIN2025114364}. In parallel, the impact of optimization dynamics and physics-based loss formulations on RepNN merits further investigation. Finally, developing localized feature embeddings to complement the global reparameterization of the first hidden layer is an important direction for capturing spatially varying multiscale structures.

\section*{Acknowledgements}
Y.~W.~and X.~M.~acknowledge the support of the National Natural Science Foundation of China (No. 12201229). X.~M.~also acknowledges the support of the Xiaomi Young Talents Program.  

\appendix
\section{Proof of Lemma \ref{Lemma1}}\label{Appendix: Proof}
\begin{proof}
By linearity of expectation,
\[
\mathbb{E}[S_m] = 0.
\]
Moreover, since \(\mathbb{E}[S_m]=0\), we have
\[
\mathrm{Var}(S_m)
=
\mathbb{E}[S_m^2].
\]
Now,
\[
S_m^2
=
\left(\sum_{i=1}^m Y_i\right)^2
=
\sum_{i=1}^m Y_i^2
+
2\sum_{1\le i<j\le m}Y_iY_j.
\]
Taking expectations gives
\[
\mathbb{E}[S_m^2]
=
\sum_{i=1}^m \mathbb{E}[Y_i^2]
+
2\sum_{1\le i<j\le m}\mathbb{E}[Y_iY_j].
\]
Since \(\mathbb{E}[Y_i]=0\) and \(\mathrm{Var}(Y_i)=\sigma^2\), we have
\[
\mathbb{E}[Y_i^2]=\sigma^2.
\]
Also, by assumption, \(\mathbb{E}[Y_iY_j]=0\) for all \(i\neq j\). Therefore,
\[
\mathrm{Var}(S_m)
=
\mathbb{E}[S_m^2]
=
m\sigma^2.
\]
By Chebyshev's inequality, for any $\varepsilon > 0$,
\[
\mathbb{P}(|S_m| \ge k \sqrt{m}\,\sigma) \le \frac{1}{k^2}.
\]
Taking $k = 1/\sqrt{\varepsilon}$, we obtain
\[
\mathbb{P}\bigl(|S_m| \le C_\varepsilon \sqrt{m}\,\sigma\bigr) \ge 1 - \varepsilon,
\]
where $C_\varepsilon = 1/\sqrt{\varepsilon}$. This implies that
\[
S_m = \mathcal{O}_p(\sqrt{m}\,\sigma).
\]
\end{proof}

\section{Effect of the NN architectures on the predicted accuracy}\label{sec:nn_architec}

We employ the following  case to test the effect of the architectures of RepNN, i.e., width/depth and activations,  on the predicted accuracy:
\begin{equation}\label{Methodology: higher-frequency}
u(x) = \sin (30 \pi x), \quad x\in[-1,1].
\end{equation}
We employ the reparameterized shallow ReLU NN Eq.~\eqref{eq:shallow-dnn-reparameterization} and its deep counterpart Eq.~\eqref{eq:deep-dnn-reparameterization} to fit the target function.
For the shallow NN, we keep all hyperparameter settings the same as those used for approximating function Eq.~\eqref{Methodology: high-frequency}, except that the number of hidden neurons is set to $m=400$ and $m=20000$, respectively. For the deep reparameterized ReLU DNN, we use 4 hidden layers with 100 neurons per layer. The parameters $(W_{1}, \tilde b_{1})$ are initialized in the same way as in Eq.~\eqref{Methodology: wswmb1}, while the remaining parameters follow He initialization. The minibatch training strategy is employed, and the batch size used in the loss function Eq.~\eqref{eq:dnn-mse-loss} is set to \(N=128\). Based on the estimates in Eq.~\eqref{magnitude2 relu:u}, in an always-activated local region, the reparameterized ReLU DNN yields
\(u_{\bm{\theta}}(x) = \mathcal{O}_p\!\left(\nu _s H^{-1}\right)\) and 
\(\frac{\mathrm{d} u_{\bm{\theta}}}{\mathrm{d} x} = \mathcal{O}_p(\nu_s)\).

\begin{figure}[!ht]
 \centering
 \includegraphics[width=1.0\textwidth]{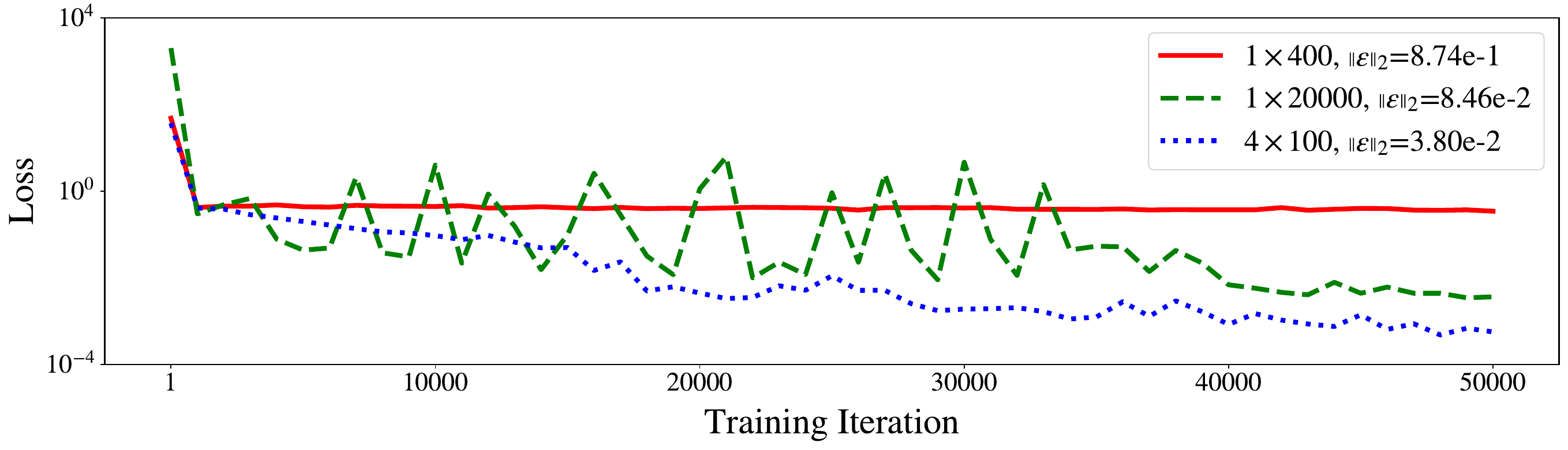} 
 \caption{1D high-frequency function Eq.~\eqref{Methodology: higher-frequency}: Training loss curves and relative $L_2$ errors $\|\epsilon\|_2$ of the reparameterized ReLU neural networks with different network sizes. The shallow networks Eq.~\eqref{eq:shallow-dnn-reparameterization} use $m=400$ and $m=20000$ neurons (denoted by $1 \times 400$ and $1 \times 20000$), while the deep neural network Eq.~\eqref{eq:deep-dnn-reparameterization} consists of 4 hidden layers with 100 neurons per layer (denoted by $4 \times 100$). The definition of $\|\epsilon\|_2$ is given in Eq.~\eqref{Numerical_examples_L2norm}.}
 \label{fig Methodology: relu loss}
\end{figure}

Fig.~\ref{fig Methodology: relu loss} shows the evolution of the training loss for reparameterized ReLU NNs with different network sizes during optimization. The results show that widening the neural network significantly improves its performance, while increasing the network depth is even more effective. 

\begin{figure}[!ht]
 \centering
 \includegraphics[width=1.0\textwidth]{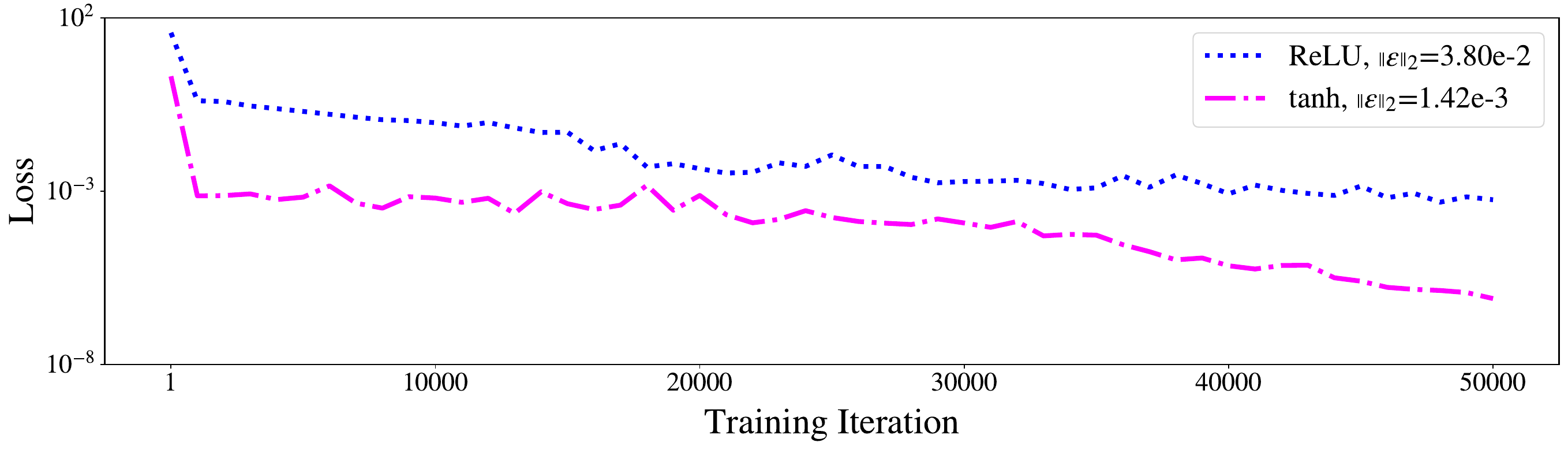} 
 \caption{1D high-frequency function Eq.~\eqref{Methodology: higher-frequency}: Training loss curves and relative $L_2$ errors $\|\epsilon\|_2$ of the reparameterized DNNs with ReLU and $\tanh$ activation functions.}
 \label{fig Methodology: relu and tanh loss}
\end{figure}

We further compare the RepNN with $\relu$ and $\tanh$  activations using the same test case. In particular, both RepNNs consist of 4 hidden layers with 100 neurons per layer.  As observed in Fig.~\ref{fig Methodology: relu and tanh loss}, the RepNN with the \(\tanh\) activation function exhibits faster convergence and achieves improved approximation performance compared to RepNN with $\relu$. In addition, further investigation of RepNN with additional activation functions is provided in \ref{Appendix: different activation function}.

\begin{figure}[!ht]
    \centering
    \includegraphics[width=1.0\textwidth]{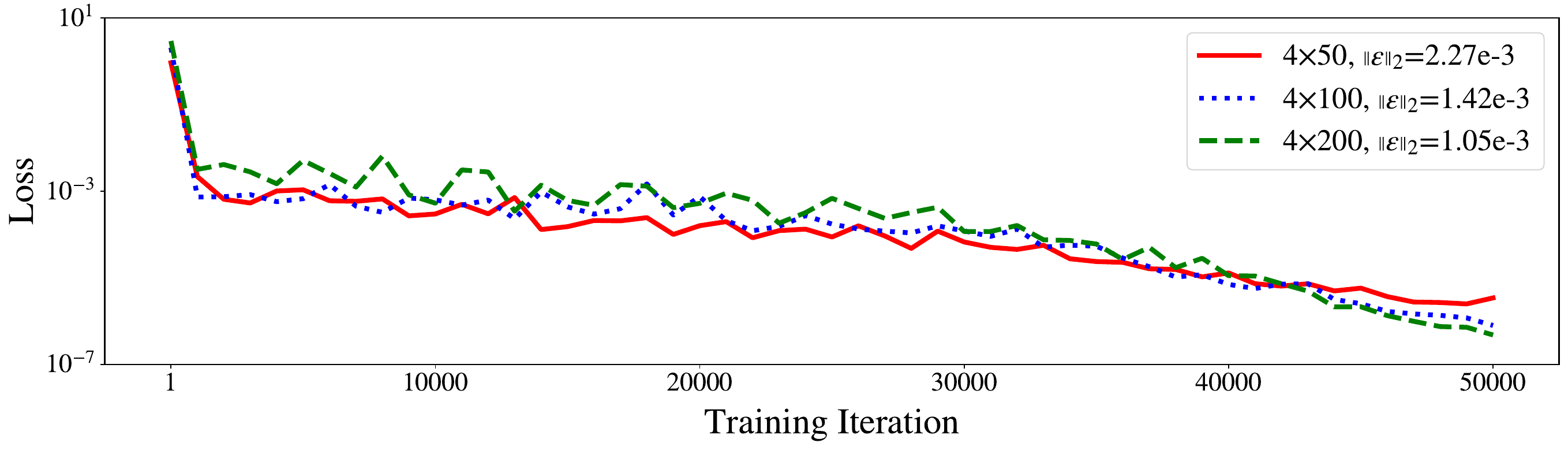} 
    \caption{1D high-frequency function Eq.~\eqref{Methodology: higher-frequency}: Training loss curves and relative $L_2$ errors of the reparameterized DNN Eq.~\eqref{eq:general-deep-dnn-reparameterization} with \(\tanh\) for different network sizes.}

    \label{fig Methodology: network size loss and error}
\end{figure}

Next,  we examine the impact of network size on the performance of the RepNN Eq.~\eqref{eq:general-deep-dnn-reparameterization} with the \(\tanh\) activation function.
We consider three network sizes, each consisting of four hidden layers with 50, 100, or 200 neurons per layer, denoted by \(4 \times 50\), \(4 \times 100\), and \(4 \times 200\), respectively. The three networks are used to approximate the high-frequency function Eq.~\eqref{Methodology: higher-frequency}. The parameters $(W_{1}, \tilde b_{1})$ are initialized in the same manner as in  Eq.~\eqref{Methodology: wswmb1}, while the remaining weights are sampled from Gaussian distributions and the biases are initialized to \(0\):
\begin{equation}
 W_{ij}^{(l)} \sim \mathcal{N}(0,1), \quad  
 W_{ij}^{(K)} \sim \mathcal{N}(0,0.01), \quad 
 b_i^{(l)} = 0, 
 \quad l = 2, \cdots, K-1.
\end{equation}
The other experimental settings are the same as those in the experiments described above.
Based on the analysis in Section~\ref{Section: redeepNN}, this initialization yields \(u_{\bm{\theta}}(x) = \mathcal{O}_p(1)\),  
\(\frac{\mathrm{d} u_{\bm{\theta}}}{\mathrm{d} x}|_{\mathrm{lin}} = \mathcal{O}_p(\nu_s H^{\frac{3}{2}})\) and \(\frac{\mathrm{d} u_{\bm{\theta}}}{\mathrm{d} x}|_{\mathrm{sat}} = \mathcal{O}_p(\nu _s H^{-\frac{1}{2}})\).

Fig.~\ref{fig Methodology: network size loss and error} shows the training loss curves of the reparameterized DNN with the \(\tanh\) activation function for different network sizes during optimization. As the number of neurons in each hidden layer increases, the training loss decreases and the prediction accuracy improves. 

\section{Frequency-wise training behavior of the RepNN}\label{FWTBDNN}

\textcolor{black}{To showcase the learning pattern of the RepNN, we consider to employ it to approximate the following one-dimensional multi-frequency target function as a representative empirical study:
\begin{equation}\label{three_frequency}
    u(x)=\sin(4 \pi x) + \sin(44 \pi x) + \sin(84 \pi x), \quad x \in [0, 1].
\end{equation}
The Fourier spectrum of this function contains three nonzero modes, which represent low-, intermediate-, and high-frequency components.} 

\begin{figure}[!ht]
    \centering
    \includegraphics[width=1.0\textwidth]{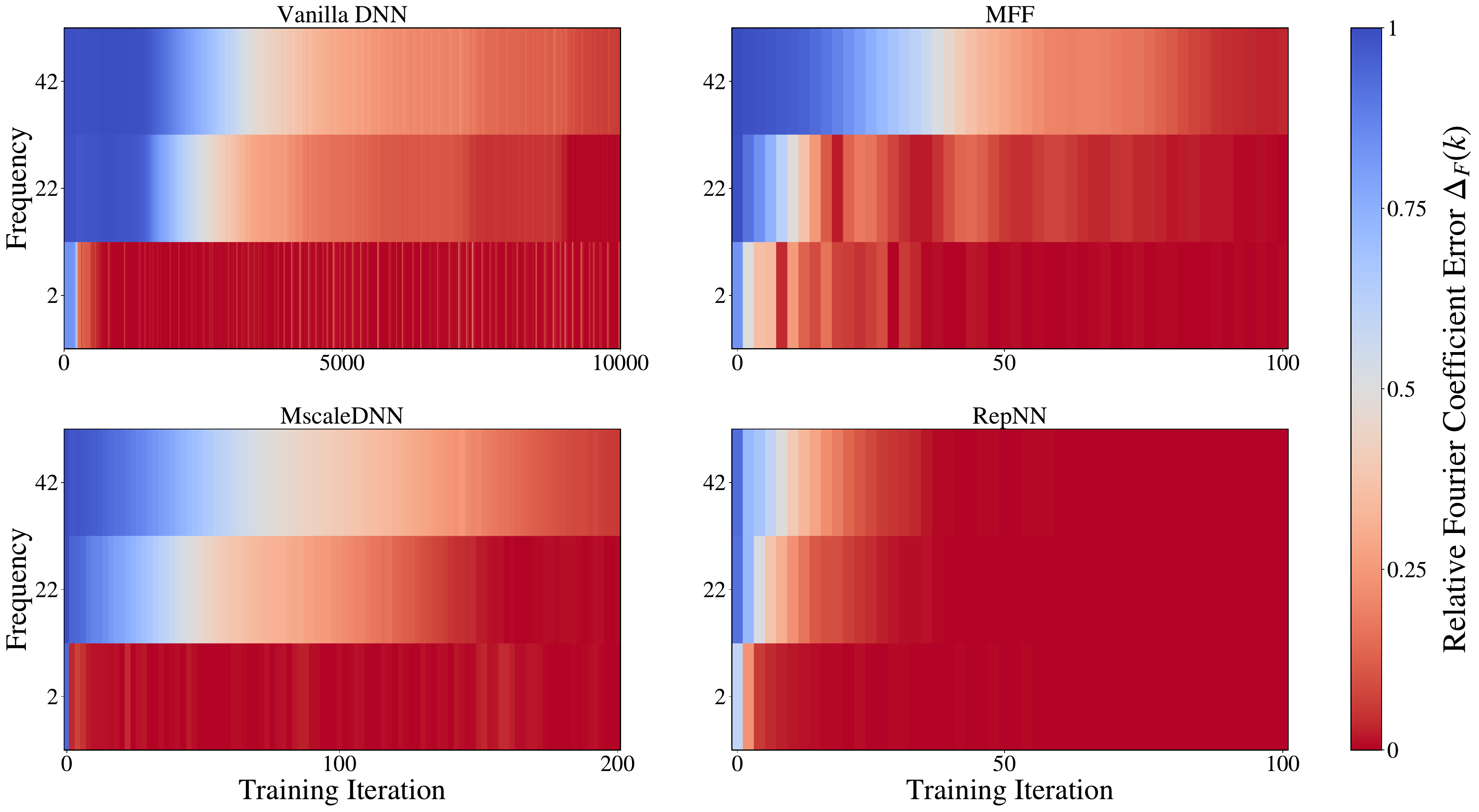} 
    \caption{1D multi-frequency function Eq.~\eqref{three_frequency}: Frequency-wise training behavior for different DNNs. Each heatmap reports the relative Fourier coefficient error $\Delta_F(k)=|\hat{f}_k^{nn}-\hat{f}_k|/|\hat{f}_k|$, where $\hat{f}_k^{nn}$ and $\hat{f}_k$ denote the Fourier coefficients obtained by the discrete Fourier transform of the DNN approximation and the target function, respectively, at the three target frequency modes $k\in\{2,22,42\}$.} 
    \label{fig: rfc_error}
\end{figure}

\textcolor{black}{We employ four DNN models to approximate the target function Eq.~\eqref{three_frequency}, i.e., the vanilla DNN, MFF \cite{WANG2021113938} with the Fourier feature mapping parameter $\gamma_x=20$, MscaleDNN \cite{CiCP281970} with scale coefficients $\{1,2,4,8,16\}$, and the proposed RepNN. The training dataset $\{x_i,u(x_i)\}_{i=1}^{1280}$ is obtained by uniformly sampling 1280 points from the computational domain. All other experimental settings are consistent with those described in Section~\ref{section: Numerical examples}.}

\textcolor{black}{Fig.~\ref{fig: rfc_error} shows the decay of the frequency-wise errors for different DNNs. The vanilla DNN exhibits a clear spectral bias, where low-frequency components are learned first, while the high-frequency components converge much more slowly compared to the remaining models.  Similar to MFF and MscaleDNN, RepNN does not completely change the low-to-high frequency learning pattern, but it significantly accelerates the convergence of high-frequency components within only a few iterations. The above results demonstrate that RepNN effectively alleviates the difficulty of the vanilla DNN in  capturing high-frequency oscillatory features for this specific case.}



\section{Physics-informed neural networks}\label{Appendix: pinn}

Physics-informed neural networks (PINNs) constitute a popular class of DNN-based methods for solving partial differential equations.

Consider a PDE problem in the following general form:
\begin{equation}\label{Appendix: pde function}
\begin{cases}
\mathcal{N}\bigl(u(\bm{x});\lambda\bigr)=f(\bm{x}), 
\quad \bm{x}\in \Omega \subset \mathbb{R}^{d}, \\
\mathcal{B}\bigl(u(\bm{x})\bigr)=g(\bm{x}), 
\quad \bm{x}\in \partial\Omega.
\end{cases}
\end{equation}
where $\mathcal{N}$ and $\mathcal{B}$ denote the differential operator and the boundary condition operator, respectively. Here, $u(\bm{x})$ is the solution evaluated at the location $\bm{x}$, $\lambda$ represents the model parameters, and $f(\bm{x})$ and $g(\bm{x})$ are prescribed functions.

A vanilla PINN uses a DNN $u_{\bm{\theta}}(\bm{x})$ to approximate the solution of Eq.~\eqref{Appendix: pde function}, where $\bm{\theta}$ denotes the trainable parameters. The associated MSE loss is given by
\begin{equation}\label{Appendix: loss_std}
\begin{cases}
\displaystyle
\mathcal{L}(\bm{\theta})
=
\mathcal{L}_{\mathrm{PDE}}(\bm{\theta})
+
\mathcal{L}_{\mathrm{BC}}(\bm{\theta})
+
\mathcal{L}_{\mathrm{Data}}(\bm{\theta}),\\[0.6em]
\displaystyle
\mathcal{L}_{\mathrm{PDE}}(\bm{\theta})
=
\frac{1}{N_r}
\sum_{i=1}^{N_r}
\bigl|
\mathcal{N}\bigl(u_{\bm{\theta}}(\bm{x}_i^r);\lambda\bigr)
-
f(\bm{x}_i^r)
\bigr|^2,\\[0.6em]
\displaystyle
\mathcal{L}_{\mathrm{BC}}(\bm{\theta})
=
\frac{1}{N_b}
\sum_{i=1}^{N_b}
\bigl|
\mathcal{B}\bigl(u_{\bm{\theta}}(\bm{x}_i^b)\bigr)
-
g(\bm{x}_i^b)
\bigr|^2,\\[0.6em]
\displaystyle
\mathcal{L}_{\mathrm{Data}}(\bm{\theta})
=
\frac{1}{N_d}
\sum_{i=1}^{N_d}
\bigl|
u_{\bm{\theta}}(\bm{x}_i^d)
-
u(\bm{x}_i^d)
\bigr|^2.
\end{cases}
\end{equation}
Here, the sets $\{\bm{x}_i^r\}_{i=1}^{N_r}$, $\{\bm{x}_i^b\}_{i=1}^{N_b}$, and $\{\bm{x}_i^d\}_{i=1}^{N_d}$ denote the sampling points used to evaluate the PDE residual loss $\mathcal{L}_{\mathrm{PDE}}(\bm{\theta})$, the boundary condition loss $\mathcal{L}_{\mathrm{BC}}(\bm{\theta})$, and the data mismatch loss $\mathcal{L}_{\mathrm{Data}}(\bm{\theta})$, respectively.  {For time-dependent problems, the temporal coordinate $t$ can be regarded as a component of $\bm{x}$, and the corresponding initial conditions can be implemented in the same way for boundary conditions.}

Within the PINN framework, large differences in magnitude among the loss terms may adversely affect the training process. To alleviate this issue, we adopt the regularized loss formulation proposed in Ref.~\cite{WANG2024113112}. Table~\ref{Appendix: Table loss} summarizes the corresponding regularized loss functions and batch sizes used for the PDE problems considered in Section~\ref{section: PDE problem}.

\begin{table}[h]	
\setlength{\abovecaptionskip}{0cm}
\setlength{\belowcaptionskip}{0.2cm}
\caption{Configuration of the regularized loss functions used for the PDE problems considered in Section~\ref{section: PDE problem}. The same batch size is used for all loss terms within each problem. For the Klein--Gordon equation in Section~\ref{section: Klein-Gordon equation}, the initial condition $u(x,0)$ in Eq.~\eqref{PDE: KG} is treated as a special boundary condition, whereas $\mathcal{L}_{u_t}(\bm{\theta})$ denotes the loss term associated with the initial condition $u_t(x,0)$ in Eq.~\eqref{PDE: KG}.}
\label{Appendix: Table loss}
\centering
\begin{adjustbox}{max width=\textwidth}
\begin{tabular}{c|c|c}
\hline 
Section & Loss Function $\mathcal{L}(\bm{\theta})$ 
   & Batch Size \\
\hline 
Section~\ref{section: Klein-Gordon equation} &
$
\bigl(\mathcal{L}_{u_t}(\bm{\theta})\bigr)^{\frac{1}{2}}
+
\bigl(\mathcal{L}_{\mathrm{PDE}}(\bm{\theta})\bigr)^{\frac{1}{3}}
+
\bigl(\mathcal{L}_{\mathrm{BC}}(\bm{\theta})\bigr)^{\frac{1}{3}}
$
& 4096 \\
\hline 
Section~\ref{section: Helmholtz equation} 
&
$
\bigl(\mathcal{L}_{\mathrm{PDE}}(\bm{\theta})\bigr)^{\frac{1}{3}}
+
\bigl(\mathcal{L}_{\mathrm{BC}}(\bm{\theta})\bigr)^{\frac{1}{3}}
$ 
& 5120 \\
\hline 
Section~\ref{section: Gray--Scott equation} &  
$
\bigl(\mathcal{L}_{\mathrm{PDE},u}(\bm{\theta})\bigr)^{\frac{1}{4}}
+
\bigl(\mathcal{L}_{\mathrm{PDE},v}(\bm{\theta})\bigr)^{\frac{1}{4}}
+
\bigl(\mathcal{L}_{\mathrm{Data},u}(\bm{\theta})\bigr)^{\frac{1}{4}}
+
\bigl(\mathcal{L}_{\mathrm{Data},v}(\bm{\theta})\bigr)^{\frac{1}{4}}
$
& 10000 \\
\hline
\end{tabular}
\end{adjustbox}
\end{table}

\section{Deep operator network}\label{Appendix: deepOnet}

The deep operator network (DeepONet) is a neural operator framework designed to learn nonlinear mappings between infinite-dimensional Banach spaces.

Let $\mathcal{S}$ and $\mathcal{V}$ denote the input and output function spaces, respectively. The target operator is defined as
\begin{equation}
 \mathcal{G}: \mathcal{S} \to \mathcal{V}, 
 \qquad v = \mathcal{G}(s),
\end{equation}
where $s \in \mathcal{S}$ is an input function and $v \in \mathcal{V}$ is the corresponding output function. For a query location $\bm{y}$, the value of the output function is given by
\begin{equation}
 v(\bm{y}) = \mathcal{G}(s)(\bm{y}).
\end{equation}

An unstacked DeepONet, denoted by $\mathcal{G}_{\bm{\theta}}$, is trained to approximate the target operator $\mathcal{G}$. The DeepONet consists of two subnetworks with the same output dimension: a branch network $\mathcal{N}_{b}$ and a trunk network $\mathcal{N}_{t}$. The branch network encodes the discretized input function
$
\bigl(s(\bm{x}_{1}), s(\bm{x}_{2}), \ldots, s(\bm{x}_{m})\bigr)
$
into a set of coefficients, while the trunk network encodes the query location $\bm{y}$ into a set of coordinate-dependent basis functions. Thus, the DeepONet $\mathcal{G}_{\bm{\theta}}$ evaluated at a point $\bm{y}$ can be expressed as
\begin{equation}
 \mathcal{G}_{\bm{\theta}}(s)(\bm{y})
=
\underbrace{
\mathcal{N}_{b}\bigl(
s(\bm{x}_{1}), s(\bm{x}_{2}), \ldots, s(\bm{x}_{m})
\bigr)^{T}
}_{\textit{branch net}}
\underbrace{
\mathcal{N}_{t}(\bm{y})
}_{\textit{trunk net}}.
\end{equation}
Here, $\bm{\theta}$ denotes all trainable parameters of the two subnetworks. Once trained, for a given input function $s \in \mathcal{S}$, the DeepONet $\mathcal{G}_{\bm{\theta}}$ can efficiently evaluate the corresponding output function $v \in \mathcal{V}$ at arbitrary query locations.

For the earthquake problem in Section~\ref{section: Earthquake problem}, let $\bm{x}^m_{\bm{\theta}}(t)$ denote the DeepONet prediction corresponding to the $m$-th ground acceleration input $\ddot{u}^m_g(t)$, where ${\bm{\theta}}$ represents the trainable parameters of both the branch and trunk networks. The loss function is defined as
\begin{equation}\label{Earthquake: loss}
\mathcal{L}({\bm{\theta}}) = \frac{1}{M} \sum_{m=1}^M \bigl\| \bm{x}^m_{\bm{\theta}}(t) - \bm{x}^m(t) \bigr\|^2,
\end{equation}
where $\bm{x}^m(t)$ is the corresponding ground-truth displacement for the $m$-th input. Here, \(M=2\) is the batch size, and \(\|\cdot\|^2\) denotes the MSE loss for one sample.

\section{Gradient dynamics and the role of tanh saturation}\label{Section: gradient dynamics}

Consider the backpropagation chain for the gradient with respect to $W_1$. For a scalar output and MSE loss, we have:
\begin{equation}\label{eq:backprop}
\frac{\partial \mathcal{L}}{\partial W_1} = 2(u_{\bm{\theta}} - u) \cdot W_K \cdot \mathbf{J}_{K-1} \cdot W_{K-1} \cdot \mathbf{J}_{K-2} \cdots W_2 \cdot \mathbf{J}_1 \cdot {\frac{\partial z_1}{\partial W_1}},
\end{equation}
where $\mathbf{J}_l = \operatorname{diag}\bigl(\tanh'(z_l[1]), \ldots, \tanh'(z_l[H])\bigr)$ is the Jacobian of the activation function at layer $l$.

At initialization, the first-layer pre-activations satisfy $z_1 \sim \mathcal{N}(0, \nu_s^2)$ with $\nu_s=10$. For typical inputs $|x| \lesssim 1$, most neurons are deep in saturation: $|z_1| \gg 2$, yielding $\tanh'(z_1) \approx 0$. The saturated outputs $\tanh(z_1) \approx \pm 1$ then drive subsequent layers: each pre-activation $z_l[j] \approx \sum_i \pm W_l[j,i]$ is approximately $\mathcal{N}(0, H)$ since $W_l[j,i] \sim \mathcal{N}(0,1)$. With $H=300$, we obtain $z_l \approx \mathcal{O}(H^{\frac{1}{2}}) \approx 17$, for which $\tanh'(17) \approx 4\times 10^{-15}$.

Thus, at initialization, each $\mathbf{J}_l$ for $l \ge 2$ is near-zero. A naive estimate would suggest gradient attenuation by $\sim(10^{-15})^{K-2}$, far below float32 underflow ($\sim 10^{-38}$). Yet training succeeds in practice. The resolution lies in a small fraction of \textbf{active pathways}: approximately $9\%$ of neurons per layer satisfy $|z| < 2$, for which $\tanh'(z) \in [0.07, 1.0]$. These active neurons form sparse channels that carry gradient through the depth. With $H=300$, the number of fully active pathways through $K-1=4$ hidden layers is roughly $(0.09)^4 \times H^4 \approx 5 \times 10^5$, which is sufficient for reliable gradient propagation.

Two additional factors \textcolor{black}{support} stable training. First, the Adam optimizer maintains per-parameter adaptive learning rates via the second-moment estimate, partially compensating for small gradient magnitudes on saturated parameters. Second, as training proceeds, the weights and biases adapt to bring more neurons into the active regime in regions where the target function requires high-frequency resolution. This dynamic is the mechanism behind the adaptive frequency scaling property of RepNN.

The dual slope estimates $\mathcal{O}_p(\nu_s H^{\frac{3}{2}})$ (linear) and {$\mathcal{O}_p(\nu_s H^{-\frac{1}{2}})$}  (saturation) from Section~\ref{Section: redeepNN} capture this dynamic precisely. The linear estimate represents the local learning regime --- the maximum slopes achievable near the origin, which enable high-frequency fitting. The saturation estimate represents the global regime at initialization, where $\tanh$ saturation constrains gradient magnitudes and prevents explosion. The $\tanh$ activation thus serves as a built-in regularizer: it permits large slopes locally for expressivity while globally bounding the gradient scale.

\textcolor{black}{To provide further empirical support for the pathway-based argument, we examine the gradient dynamics of RepNN during training. Specifically, we consider the one-dimensional high-frequency regression problem in Eq.~\eqref{sec100pi High-frequency} and monitor the gradient norms of trainable parameters across different layers. For a trainable parameter tensor $P$, the gradient norm is defined as}
\begin{equation}
    G_{P}
    =
    \sqrt{
    \sum_{i=1}^{N_{P}}
    \left(
    \frac{\partial \mathcal{L}}
    {\partial P_i}
    \right)^2
    },
\end{equation}
where $N_{P}$ denotes the total number of trainable entries in $P$, and $P_i$ denotes its $i$-th scalar entry. 

\begin{figure}[!ht]
    \centering
    \includegraphics[width=1.0\textwidth]{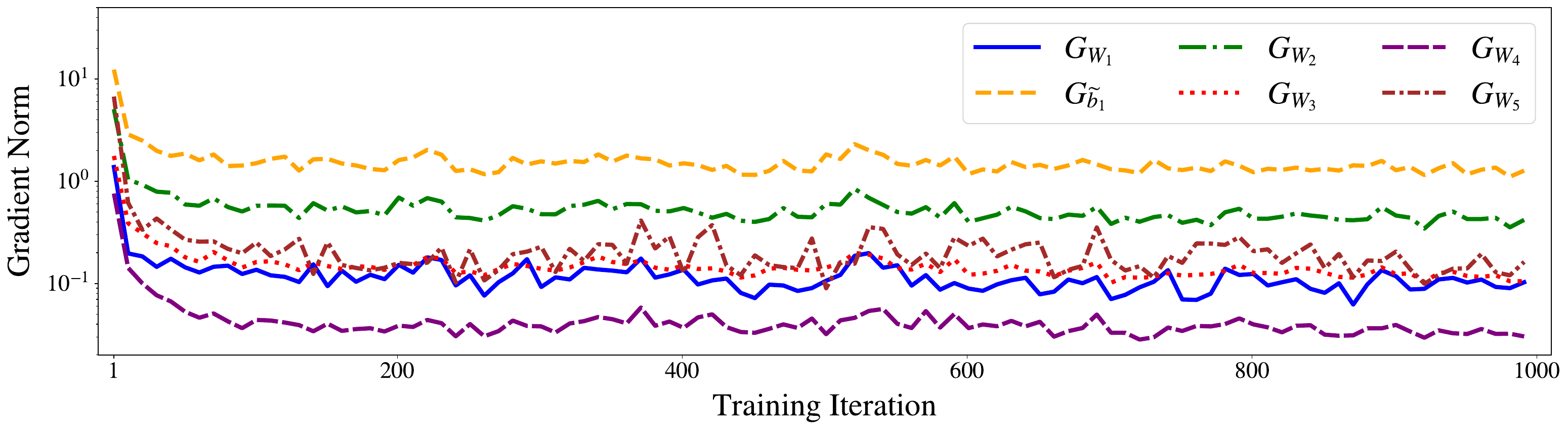} 
    \caption{1D high-frequency function Eq.~\eqref{sec100pi High-frequency}: Gradient norms of the trainable parameters $W_1,\ldots,W_5$ and $\widetilde{b}_1$ in RepNN over the first 1000 training iterations.}
    \label{fig: gradient dynamics of RepNN}
\end{figure}

{\color{black}We track the gradient norms of 
$W_1,\ldots,W_5$ and $\widetilde{b}_1$
in the proposed RepNN over the first 1,000 training iterations, as shown in Fig.~\ref{fig: gradient dynamics of RepNN}. The gradient norms remain well controlled, with no evidence of gradient explosion. Although the monitored parameters exhibit different gradient scales, each continues to receive measurable gradient signals during optimization. These observations suggest that, despite substantial tanh saturation at initialization, effective gradient propagation is maintained in the proposed RepNN for this specific case.}


\section{RepNN with additional activation functions}\label{Appendix: different activation function}

To further investigate the impact of activation functions on the RepNN, we compare the vanilla DNN Eq.~\eqref{DNNs} and the RepNN Eq.~\eqref{eq:general-deep-dnn-reparameterization} with various activation functions for approximating the high-frequency function Eq.~\eqref{Methodology: higher-frequency}.
For the vanilla DNNs, the trainable parameters are initialized using Xavier initialization Eq.~\eqref{DNN:Xavier_initialition}.  Both the vanilla DNN and the RepNN consist of 4 hidden layers with 100 neurons per layer. The remaining experimental settings are the same as those in \ref{sec:nn_architec}.


\begin{table}[h]	
\setlength{\abovecaptionskip}{0cm}
		\setlength{\belowcaptionskip}{0.2cm}
\caption{1D high-frequency function Eq.~\eqref{Methodology: higher-frequency}: The performance of the vanilla DNN Eq.~\eqref{DNNs} and the RepNN Eq.~\eqref{eq:general-deep-dnn-reparameterization} with different activation functions.}
\label{appendix: table different activation functions}
\centering
\begin{adjustbox}{max width=\textwidth}
\begin{tabular}{c|c|c|c|c}
\hline 
\multirow{2}{*}{\makecell{Activation function}}  & \multicolumn{2}{c|}{Vanilla DNN} & \multicolumn{2}{c}{RepNN}\\
\cline{2-5}
& $\mathcal{L}({\bm{\theta}})$ &  $\left\|\epsilon\right\|_2$  & $\mathcal{L}({\bm{\theta}})$ &  $\left\|\epsilon\right\|_2$  \\
\hline
tanh & $4.82 \times 10^{-1}$ & $1.00 \times 10^{0}$ 
     & $7.94 \times 10^{-7}$ & $1.42 \times 10^{-3}$ \\
\hline
Sigmoid  & $4.47 \times 10^{-1}$ & $9.62 \times 10^{-1}$ 
         & $1.75 \times 10^{-5}$ & $5.84 \times 10^{-3}$  \\                      
\hline
Softsign  & $1.71 \times 10^{-1}$ & $5.75 \times 10^{-1}$ 
          & $3.35 \times 10^{-6}$ & $2.76 \times 10^{-3}$ \\
\hline
\end{tabular}
\end{adjustbox}
\end{table}

We present in Table~\ref{appendix: table different activation functions} the results obtained by the vanilla DNN and the RepNN with different activation functions. As observed, for all considered activation functions, the RepNN significantly improves the prediction accuracy compared to the vanilla DNN. Specifically, the RepNN reduces the training loss by approximately 4--5 orders of magnitude and achieves relative \(L_2\) errors 2--3 orders smaller than those of the corresponding vanilla DNNs. Among the three activation functions, the RepNN with the \(\tanh\)  achieves the smallest training loss and relative \(L_2\) error. These results indicate that  the \(\tanh\) activation function enables the RepNN to capture high-frequency features more effectively.

\clearpage
\bibliography{mybibfile}

@article{WANG2021113938,
title = {On the eigenvector bias of Fourier feature networks: From regression to solving multi-scale PDEs with physics-informed neural networks},
journal = {Computer Methods in Applied Mechanics and Engineering},
volume = {384},
pages = {113938},
year = {2021},
issn = {0045-7825},
doi = {10.1016/j.cma.2021.113938},
author = {Sifan Wang and Hanwen Wang and Paris Perdikaris}
}

@article{khodakarami2026spectral,
title = {Spectral bias in physics-informed and operator learning: Analysis and mitigation guidelines},
journal = {Computer Methods in Applied Mechanics and Engineering},
volume = {461},
pages = {119156},
year = {2026},
issn = {0045-7825},
doi = {10.1016/j.cma.2026.119156},
author = {Siavash Khodakarami and Vivek Oommen and Nazanin Ahmadi Daryakenari and et al.}
}

@misc{kingma2,
      title={Adam: A Method for Stochastic Optimization}, 
      author={Diederik P. Kingma and Jimmy Ba},
      year={2017},
      eprint={1412.6980},
      archivePrefix={arXiv},
      primaryClass={cs.LG},
}

@article{RAISSI1,
title = {Physics-informed neural networks: A deep learning framework for solving forward and inverse problems involving nonlinear partial differential equations},
journal = {Journal of Computational Physics},
volume = {378},
pages = {686-707},
year = {2019},
issn = {0021-9991},
doi = {10.1016/j.jcp.2018.10.045},
author = {M. Raissi and P. Perdikaris and G.E. Karniadakis}
}

@article{lu2021learning,
  title={Learning nonlinear operators via DeepONet based on the universal approximation theorem of operators},
  author={Lu, Lu and Jin, Pengzhan and Pang, Guofei and et al.},
  journal={Nature machine intelligence},
  volume={3},
  number={3},
  pages={218--229},
  year={2021},
  doi = {10.1038/s42256-021-00302-5},
  publisher={Nature Publishing Group UK London}
}

@article{AMIN2026118645,
title = {I-FENN with DeepONets: Accelerating simulations in coupled multiphysics problems},
journal = {Computer Methods in Applied Mechanics and Engineering},
volume = {451},
pages = {118645},
year = {2026},
issn = {0045-7825},
doi = {10.1016/j.cma.2025.118645},
author = {Fouad M. Amin and Diab W. Abueidda and Panos Pantidis and Mostafa E. Mobasher}
}

@article{101093gjiggac399,
    author = {Song, Chao and Wang, Yanghua},
    title = {Simulating seismic multifrequency wavefields with the Fourier feature physics-informed neural network},
    journal = {Geophysical Journal International},
    volume = {232},
    number = {3},
    pages = {1503-1514},
    year = {2022},
    month = {10},
    issn = {0956-540X},
    doi = {10.1093/gji/ggac399},
}

@article{GENG2025105967,
title = {Seismic first-arrival traveltime simulation based on reciprocity-constrained PINN},
journal = {Journal of Applied Geophysics},
volume = {243},
pages = {105967},
year = {2025},
issn = {0926-9851},
doi = {10.1016/j.jappgeo.2025.105967},
author = {Hang Geng and Chao Song and Umair {bin Waheed} and Cai Liu}
}

@article{zanardi2023adaptive,
  title={Adaptive physics-informed neural operator for coarse-grained non-equilibrium flows},
  author={Zanardi, Ivan and Venturi, Simone and Panesi, Marco},
  journal={Scientific Reports},
  volume={13},
  number={1},
  pages={15497},
  year={2023},
  publisher={Nature Publishing Group UK London}
}

@article{ZHANG2024113647,
title = {CRK-PINN: A physics-informed neural network for solving combustion reaction kinetics ordinary differential equations},
journal = {Combustion and Flame},
volume = {269},
pages = {113647},
year = {2024},
issn = {0010-2180},
doi = {10.1016/j.combustflame.2024.113647},
author = {Shihong Zhang and Chi Zhang and Bosen Wang}
}

@InProceedings{pmlrv97rahaman19a,
  title = 	 {On the Spectral Bias of Neural Networks},
  author =       {Rahaman, Nasim and Baratin, Aristide and Arpit, Devansh and et al.},
  booktitle = 	 {Proceedings of the 36th International Conference on Machine Learning},
  pages = 	 {5301--5310},
  year = 	 {2019},
  editor = 	 {Chaudhuri, Kamalika and Salakhutdinov, Ruslan},
  volume = 	 {97},
  series = 	 {Proceedings of Machine Learning Research},
  month = 	 {09--15 Jun},
  publisher =    {PMLR}
}

@Article{CiCP281970,
author = {Ziqi Liu and Wei Cai and  Zhiqin John Xu},
title = {Multi-Scale Deep Neural Network (MscaleDNN) for Solving Poisson-Boltzmann Equation in Complex Domains},
journal = {Communications in Computational Physics},
year = {2020},
volume = {28},
number = {5},
pages = {1970--2001},
issn = {1991-7120},
doi = {10.4208/cicp.OA-2020-0179},
}

@article{XU2025113905,
title = {On understanding and overcoming spectral biases of deep neural network learning methods for solving PDEs},
journal = {Journal of Computational Physics},
volume = {530},
pages = {113905},
year = {2025},
issn = {0021-9991},
doi = {10.1016/j.jcp.2025.113905},
author = {Zhi-Qin John Xu and Lulu Zhang and Wei Cai}
}

@article{doi10113719M1310050,
author = {Cai, Wei and Li, Xiaoguang and Liu, Lizuo},
title = {A Phase Shift Deep Neural Network for High Frequency Approximation and Wave Problems},
journal = {SIAM Journal on Scientific Computing},
volume = {42},
number = {5},
pages = {A3285-A3312},
year = {2020},
doi = {10.1137/19M1310050}
}

@inproceedings{hu2024neutron,
  title={Neutron resonance cross sections evaluation based on the phase shift deep neural network},
  author={Hu, Zehua},
  booktitle={EPJ Web of Conferences},
  volume={302},
  pages={07002},
  year={2024},
  organization={EDP Sciences}
}

@inproceedings{10555536258343625935,
  title = 	 {Phase-shifted adversarial training},
  author =       {Kim, Yeachan and Kim, Seongyeon and Seo, Ihyeok and Shin, Bonggun},
  booktitle = 	 {Proceedings of the Thirty-Ninth Conference on Uncertainty in Artificial Intelligence},
  pages = 	 {1068--1077},
  year = 	 {2023},
  editor = 	 {Evans, Robin J. and Shpitser, Ilya},
  volume = 	 {216},
  series = 	 {Proceedings of Machine Learning Research},
  month = 	 {31 Jul--04 Aug},
  publisher =    {PMLR}
}

@article{HUANG2025117751,
title = {Frequency-adaptive multi-scale deep neural networks},
journal = {Computer Methods in Applied Mechanics and Engineering},
volume = {437},
pages = {117751},
year = {2025},
issn = {0045-7825},
doi = {10.1016/j.cma.2025.117751},
author = {Jizu Huang and Rukang You and Tao Zhou}
}

@article{WANG2024113112,
title = {A practical PINN framework for multi-scale problems with multi-magnitude loss terms},
journal = {Journal of Computational Physics},
volume = {510},
pages = {113112},
year = {2024},
issn = {0021-9991},
doi = {10.1016/j.jcp.2024.113112},
author = {Yong Wang and Yanzhong Yao and Jiawei Guo and Zhiming Gao}
}

@article{LI2023112242,
title = {Subspace decomposition based DNN algorithm for elliptic type multi-scale PDEs},
journal = {Journal of Computational Physics},
volume = {488},
pages = {112242},
year = {2023},
issn = {0021-9991},
doi = {10.1016/j.jcp.2023.112242},
author = {Xi-An Li and Zhi-Qin John Xu and Lei Zhang}
}

@article{tancik2020fourier,
  title={Fourier features let networks learn high frequency functions in low dimensional domains},
  author={Tancik, Matthew and Srinivasan, Pratul and Mildenhall, Ben and et al.},
  journal={Advances in Neural Information Processing Systems},
  volume={33},
  pages={7537--7547},
  year={2020}
}

@article{LI2023114963,
title = {A deep domain decomposition method based on Fourier features},
journal = {Journal of Computational and Applied Mathematics},
volume = {423},
pages = {114963},
year = {2023},
issn = {0377-0427},
doi = {10.1016/j.cam.2022.114963},
author = {Sen Li and Yingzhi Xia and Yu Liu and Qifeng Liao}
}

@misc{feng202ention,
      title={Overcoming Spectral Bias via Cross-Attention}, 
      author={Xiaodong Feng and Tao Tang and Xiaoliang Wan and Tao Zhou},
      year={2025},
      eprint={2512.18586},
      archivePrefix={arXiv},
      primaryClass={math.NA}
}

@misc{huasorneuralnetworks,
      title={Frequency-adaptive tensor neural networks for high-dimensional multi-scale problems}, 
      author={Jizu Huang and Rukang You and Tao Zhou},
      year={2025},
      eprint={2508.15198},
      archivePrefix={arXiv},
      primaryClass={cs.LG}
}

@article{zhelving,
title = {Data-integrated neural networks for solving partial differential equations},
journal = {Neurocomputing},
pages = {135041},
year = {2026},
issn = {0925-2312},
doi = {10.1016/j.neucom.2026.135041},
author = {Jiachun Zheng and Yunqing Huang and Nianyu Yi and Yunlei Yang}
}

@article{zhltz,
  title={The Adaptive Solution of High-Frequency Helmholtz Equations via Multi-Grade Deep Learning},
  author={Zhao, Peiyao and Wang, Rui and Wu, Tingting and Xu, Yuesheng},
  journal={Journal of Scientific Computing},
  volume={109},
  number={2},
  pages={50},
  year={2026},
  publisher={Springer}
}

@article{10106350286561,
    author = {Xiong, Xiong and Lu, Kang and Zhang, Zhuo and et al.},
    title = {High-frequency flow field super-resolution via physics-informed hierarchical adaptive Fourier feature networks},
    journal = {Physics of Fluids},
    volume = {37},
    number = {9},
    pages = {097111},
    year = {2025},
    month = {09},
    issn = {1070-6631},
    doi = {10.1063/5.0286561},
}

@article{wang202ing,
  title={Solving high-dimensional partial differential equations using tensor neural network and a posteriori error estimators},
  author={Wang, Yifan and Lin, Zhongshuo and Liao, Yangfei and et al.},
  journal={Journal of Scientific Computing},
  volume={101},
  number={3},
  pages={67},
  year={2024},
  publisher={Springer}
}

@article{doi1043,
author = {Wang, Sifan and Teng, Yujun and Perdikaris, Paris},
title = {Understanding and Mitigating Gradient Flow Pathologies in Physics-Informed Neural Networks},
journal = {SIAM Journal on Scientific Computing},
volume = {43},
number = {5},
pages = {A3055-A3081},
year = {2021},
doi = {10.1137/20M1318043}
}

@article{MA202641,
title = {Regularized numerical method for the space-fractional logarithmic Klein-Gordon equation},
journal = {Computers $\&$ Mathematics with Applications},
volume = {203},
pages = {41-55},
year = {2026},
issn = {0898-1221},
doi = {10.1016/j.camwa.2025.11.025},
author = {Qibo Ma and Xiaoyun Jiang and Junqing Jia}
}

@misc{ren2025rmed,
      title={General Fourier Feature Physics-Informed Extreme Learning Machine (GFF-PIELM) for High-Frequency PDEs}, 
      author={Fei Ren and Sifan Wang and Pei-Zhi Zhuang and et al.},
      year={2025},
      eprint={2510.12293},
      archivePrefix={arXiv},
      primaryClass={cs.LG},
}

@misc{driscoll2014chebfun,
  title={Chebfun guide},
  author={Driscoll, Tobin A and Hale, Nicholas and Trefethen, Lloyd N},
  year={2014},
  publisher={Pafnuty Publications, Oxford}
}

@article{liCausalityear,
author = {Liu, Lizuo and Nath, Kamaljyoti and Cai, Wei},
year = {2024},
month = {05},
pages = {1194-1228},
title = {A Causality-DeepONet for Causal Responses of Linear Dynamical Systems},
volume = {35},
journal = {Communications in Computational Physics},
doi = {10.4208/cicp.OA-2023-0078}
}

@inproceedings{glorot,
  title={Understanding the difficulty of training deep feedforward neural networks},
  author={Glorot, Xavier and Bengio, Yoshua},
  booktitle={Proceedings of the thirteenth international conference on artificial intelligence and statistics},
  pages={249--256},
  year={2010},
  organization={JMLR Workshop and Conference Proceedings}
}

@inproceedings{he2015delving,
  title={Delving deep into rectifiers: Surpassing human-level performance on imagenet classification},
  author={He, Kaiming and Zhang, Xiangyu and Ren, Shaoqing and Sun, Jian},
  booktitle={Proceedings of the IEEE international conference on computer vision},
  pages={1026--1034},
  year={2015}
}

@article{LLLBFGS,
author = {Byrd, Richard H. and Lu, Peihuang and Nocedal, Jorge and Zhu, Ciyou},
title = {A limited memory algorithm for bound constrained optimization},
journal = {SIAM Journal on Scientific Computing},
volume = {16},
number = {5},
pages = {1190-1208},
year = {1995},
doi = {10.1137/0916069}
}

@article{xu2025overview,
  title={Overview frequency principle/spectral bias in deep learning},
  author={Xu, Zhi-Qin John and Zhang, Yaoyu and Luo, Tao},
  journal={Communications on Applied Mathematics and Computation},
  volume={7},
  number={3},
  pages={827--864},
  year={2025},
  publisher={Springer}
}

@article{hong2022activation,
  title={On the activation function dependence of the spectral bias of neural networks},
  author={Hong, Qingguo and Siegel, Jonathan W and Tan, Qinyang and Xu, Jinchao},
  journal={arXiv:2208.04924},
  year={2022}
}

@article{he2020reluReLU,
  title={ReLU deep neural networks and linear finite elements},
  author={He, Juncai and Li, Lin and Xu, Jinchao and Zheng, Chunyue},
  journal={Journal of Computational Mathematics},
  volume={38},
  number={3},
  pages={502--527},
  year={2020},
  publisher={JSTOR}
}

@article{CHEN2025113939,
title = {A level set immersed finite element method for parabolic problems on surfaces with moving interfaces},
journal = {Journal of Computational Physics},
volume = {531},
pages = {113939},
year = {2025},
issn = {0021-9991},
doi = {10.1016/j.jcp.2025.113939},
author = {Jiaqi Chen and Xufeng Xiao and Xinlong Feng and Dongwoo Sheen}
}

@article{doi1011370718033,
author = {Babuska, I. and Szabo, B. A. and Katz, I. N.},
title = {The p-Version of the Finite Element Method},
journal = {SIAM Journal on Numerical Analysis},
volume = {18},
number = {3},
pages = {515-545},
year = {1981},
doi = {10.1137/0718033},
}

@article{TensorNN,
author = {Yifan Wang and Hehu Xie},
title = {Tensor neural network and its numerical integration},
journal={Journal of Computational Mathematics},
doi={10.4208/jcm.2307-m2022-0233},
year={2024},
pages={1714–1742},
volume={42},
month={Nov.}}

@article{cheng2026generalized,
  title={Generalized Transferable Neural Networks for Steady-State Partial Differential Equations},
  author={Cheng, Tao and Ju, Lili and Qiao, Zhonghua and Zhang, Xiaoping},
  journal={arXiv preprint arXiv:2604.03020},
  year={2026}
}

@article{Structured_First_Layer,
title={Structured First-Layer Initialization Pre-Training Techniques to Accelerate Training Process Based on $\varepsilon$-Rank},
author={Tao Tang and Jiang Yang and Yuxiang Zhao and Yuxiang Zhao},
journal={Communications in Computational Physics},
DOI={10.4208/cicp.OA-2025-0185},
year={2026},
pages={61–87},
month={Mar.},
number={1},
volume={40}}

@misc{yang2026adtworks,
      title={Adaptive-Distribution Randomized Neural Networks for PDEs: A Low-Dimensional Distribution-Learning Framework}, 
      author={You Yang and Fei Wang},
      year={2026},
      eprint={2604.23999},
      archivePrefix={arXiv},
      primaryClass={math.NA},
}

@article{sitzmann2020implicit,
  title={Implicit neural representations with periodic activation functions},
  author={Sitzmann, Vincent and Martel, Julien and Bergman, Alexander and et al.},
  journal={Advances in Neural Information Processing Systems},
  volume={33},
  pages={7462--7473},
  year={2020}
}

@article{Zhang_Shukla_Wang_s, 
title={Turbulence closure in Reynolds-averaged Navier–Stokes and flow inference around a cylinder using physics-informed neural networks and sparse experimental data}, 
volume={1034}, 
DOI={10.1017/jfm.2026.11471}, 
journal={Journal of Fluid Mechanics}, 
author={Zhang, Zhen and Shukla, Khemraj and Wang, Zhicheng and et al.}, 
year={2026}, 
pages={A16}}

@article{ZHENG2025110329,
title = {CF-DeepONet: Deep operator neural networks for solving compressible flows},
journal = {Aerospace Science and Technology},
volume = {163},
pages = {110329},
year = {2025},
issn = {1270-9638},
doi = {10.1016/j.ast.2025.110329},
author = {Jinglai Zheng and Hanying Hu and Jie Huang and Buyue Zhao and Haiming Huang}
}

@article{CHEN2022110996,
title = {Meta-MgNet: Meta multigrid networks for solving parameterized partial differential equations},
journal = {Journal of Computational Physics},
volume = {455},
pages = {110996},
year = {2022},
issn = {0021-9991},
doi = {10.1016/j.jcp.2022.110996},
author = {Yuyan Chen and Bin Dong and Jinchao Xu}
}

@article{zhang2024blending,
  title={Blending neural operators and relaxation methods in PDE numerical solvers},
  author={Zhang, Enrui and Kahana, Adar and Kopani{\v{c}}{\'a}kov{\'a}, Alena and et al.},
  journal={Nature Machine Intelligence},
  volume={6},
  number={11},
  pages={1303--1313},
  year={2024},
  publisher={Nature Publishing Group UK London}
}

@article{ALDIRANY2024116666,
title = {Multi-level neural networks for accurate solutions of boundary-value problems},
journal = {Computer Methods in Applied Mechanics and Engineering},
volume = {419},
pages = {116666},
year = {2024},
issn = {0045-7825},
doi = {10.1016/j.cma.2023.116666},
author = {Ziad Aldirany and Régis Cottereau and Marc Laforest and Serge Prudhomme}
}

@misc{wangves,
title={Random Weight Factorization Improves the Training of Continuous Neural Representations}, 
      author={Sifan Wang and Hanwen Wang and Jacob H. Seidman and Paris Perdikaris},
      year={2022},
      eprint={2210.01274},
      archivePrefix={arXiv},
      primaryClass={cs.LG},
}

@article{YOU2026114530,
title = {MscaleFNO: Multi-scale Fourier neural operator learning for oscillatory functions and wave scattering problems},
journal = {Journal of Computational Physics},
volume = {547},
pages = {114530},
year = {2026},
issn = {0021-9991},
doi = {10.1016/j.jcp.2025.114530},
author = {Zilin You and Zhenli Xu and Wei Cai}
}

@misc{wang2safsded,
    title={Simulating Three-dimensional Turbulence with Physics-informed Neural Networks}, 
      author={Sifan Wang and Shyam Sankaran and Xiantao Fan and et al.},
      year={2025},
      eprint={2507.08972},
      archivePrefix={arXiv},
      primaryClass={cs.LG},
}

@article{CHEN2025124577,
title = {Three-dimensional spatiotemporal wind field reconstruction based on LiDAR and multi-scale PINN},
journal = {Applied Energy},
volume = {377},
pages = {124577},
year = {2025},
issn = {0306-2619},
doi = {10.1016/j.apenergy.2024.124577},
author = {Yuanqing Chen and Ding Wang and Dachuan Feng and et al.}
}

@article{doi10113723M1558227,
author = {Li, Zhengyi and Wang, Yanli and Liu, Hongsheng and et al.},
title = {Solving the Boltzmann Equation with a Neural Sparse Representation},
journal = {SIAM Journal on Scientific Computing},
volume = {46},
number = {2},
pages = {C186-C215},
year = {2024},
doi = {10.1137/23M1558227}
}

@article{xie2025simultaneous,
title={Simultaneous suppression of seismic random and erratic noise using PINN with high-frequency preservation},
author={Xie, Peihong and Liu, Yang and Liu, Cai and et al.},
journal={Journal of Geophysics and Engineering},
volume={22},
number={6},
pages={1796--1808},
year={2025},
publisher={Oxford University Press}
}

@misc{chsicsinformed,
      title={Enforcing hidden physics in physics-informed neural networks}, 
      author={Nanxi Chen and Sifan Wang and Rujin Ma and et al.},
      year={2025},
      eprint={2511.14348},
      archivePrefix={arXiv},
      primaryClass={cs.LG}
}

@article{HUANG2025113676,
title = {Sparse-regularized high-frequency enhanced neural network for solving high-frequency problems},
journal = {Journal of Computational Physics},
volume = {523},
pages = {113676},
year = {2025},
issn = {0021-9991},
doi = {10.1016/j.jcp.2024.113676},
author = {Qilin Huang and Mingjin Fang and Dongsheng Cheng and et al.}
}

@article{KHADEMI2025109672,
title = {Physics-informed neural networks with trainable sinusoidal activation functions for approximating the solutions of the Navier-Stokes equations},
journal = {Computer Physics Communications},
volume = {314},
pages = {109672},
year = {2025},
issn = {0010-4655},
doi = {10.1016/j.cpc.2025.109672},
author = {Amirhossein Khademi and Steven Dufour}
}

@misc{tang2based,
title={Multiscale lubrication simulation based on fourier feature networks with trainable frequency}, 
author={Yihu Tang and Li Huang and Limin Wu and Xianghui Meng},
year={2024},
eprint={2405.12638},
archivePrefix={arXiv},
primaryClass={cs.LG},
}

@article{LIU2025106886,
title = {Diminishing spectral bias in physics-informed neural networks using spatially-adaptive Fourier feature encoding},
journal = {Neural Networks},
volume = {182},
pages = {106886},
year = {2025},
issn = {0893-6080},
doi = {10.1016/j.neunet.2024.106886},
author = {Yarong Liu and Hong Gu and Xiangjun Yu and Pan Qin}
}

@article{HOU2026108247,
title = {Fourier feature-enhanced multi-layer residual stacking network: A novel multiscale modeling approach for physics-informed neural networks},
journal = {Neural Networks},
volume = {195},
pages = {108247},
year = {2026},
issn = {0893-6080},
doi = {10.1016/j.neunet.2025.108247},
author = {Bo-Ya Hou and Yu-Long Bai and Xia-Ting Jing and Chun-lin Huang}
}

@article{yu2018deep,
  title={The deep Ritz method: a deep learning-based numerical algorithm for solving variational problems},
  author={Yu, Bing and others},
  journal={Communications in Mathematics and Statistics},
  volume={6},
  number={1},
  pages={1--12},
  year={2018},
  publisher={Springer}
}

@misc{li2021ftric,
      title={Fourier Neural Operator for Parametric Partial Differential Equations}, 
      author={Zongyi Li and Nikola Kovachki and Kamyar Azizzadenesheli and et al.},
      year={2021},
      eprint={2010.08895},
      archivePrefix={arXiv},
      primaryClass={cs.LG},
}

@article{LIN2025114364,
title = {Monte Carlo physics-informed neural networks for multiscale heat conduction via phonon Boltzmann transport equation},
journal = {Journal of Computational Physics},
volume = {542},
pages = {114364},
year = {2025},
issn = {0021-9991},
doi = {10.1016/j.jcp.2025.114364},
author = {Qingyi Lin and Chuang Zhang and Xuhui Meng and Zhaoli Guo}
}

@inproceedings{ICLR2025_e9886640,
 author = {Vyas, Nikhil and Morwani, Depen and Zhao, Rosie and et al.},
 booktitle = {International Conference on Learning Representations},
 editor = {Y. Yue and A. Garg and N. Peng and F. Sha and R. Yu},
 pages = {93423--93444},
 title = {SOAP: Improving and Stabilizing Shampoo using Adam for Language Modeling},
 volume = {2025},
 year = {2025}
}

@article{HU2024106369,
title = {Tackling the curse of dimensionality with physics-informed neural networks},
journal = {Neural Networks},
volume = {176},
pages = {106369},
year = {2024},
issn = {0893-6080},
doi = {10.1016/j.neunet.2024.106369},
author = {Zheyuan Hu and Khemraj Shukla and George Em Karniadakis and Kenji Kawaguchi}
}

@misc{luo2019thinciplegeneral,
      title={Theory of the Frequency Principle for General Deep Neural Networks}, 
      author={Tao Luo and Zheng Ma and Zhi-Qin John Xu and Yaoyu Zhang},
      year={2019},
      eprint={1906.09235},
      archivePrefix={arXiv},
      primaryClass={cs.LG},
}

@misc{newman2018stabletensorneuralnetworks,
      title={Stable Tensor Neural Networks for Rapid Deep Learning}, 
      author={Elizabeth Newman and Lior Horesh and Haim Avron and Misha Kilmer},
      year={2018},
      eprint={1811.06569},
      archivePrefix={arXiv},
      primaryClass={cs.LG},
}

@article{novikov2015tensorizing,
  title={Tensorizing neural networks},
  author={Novikov, Alexander and Podoprikhin, Dmitrii and Osokin, Anton and Vetrov, Dmitry P},
  journal={Advances in neural information processing systems},
  volume={28},
  year={2015}
}
\end{document}